\documentclass[letterpaper,10pt]{article}
\usepackage{geometry}
\usepackage[round]{natbib}

\let\oldsection\section
\RenewDocumentCommand{\section}{s o m}{%
  \IfBooleanTF{#1}
    {\oldsection*{\MakeUppercase{#3}}}% \section*
    {\IfValueTF{#2}
       {\oldsection[\MakeUppercase{#2}]{\MakeUppercase{#3}}}% \section[.]{..}
       {\oldsection{\MakeUppercase{#3}}}% \section{..}
    }%
}

\usepackage{color}
\usepackage{amssymb,bbm, bm}
\usepackage{empheq}

\usepackage{xcolor}
\usepackage{tcolorbox}
\usepackage{tikz}
\usetikzlibrary{shadows,arrows.meta,positioning,backgrounds,fit,chains,scopes}
\usepackage{caption}
\usepackage{subcaption}

\usepackage[parfill]{parskip}
\usepackage[colorlinks = true, citecolor=blue, linkcolor=red, urlcolor=cyan]{hyperref}

\usepackage{todonotes}
\usepackage[toc,page,header]{appendix}
\usepackage{minitoc}

\usepackage{bbold}
\usepackage[noend,ruled,algo2e,linesnumbered]{algorithm2e}
\usepackage{algorithm}

\SetCommentSty{mycommfont}
\SetKwInput{KwInput}{Input}
\SetKwInput{KwOutput}{Output}  
\SetAlFnt{\small}
\SetAlCapFnt{\small}
\SetAlCapNameFnt{\small}
\SetAlCapHSkip{0pt}
\IncMargin{-\parindent}

\newcommand{\one}{\mathbf{1}}

\renewcommand{\bar}{\overline}
\renewcommand{\tilde}{\widetilde}

\newcommand{\calL}{\mathcal{L}}

\newcommand{\calS}{\mathcal{S}}
\newcommand{\calR}{\mathcal{R}}

\newcommand{\calD}{\mathcal{D}}
\newcommand{\calH}{\mathcal{H}}
\newcommand{\calC}{\mathcal{C}}

\newcommand{\calM}{\mathcal{M}}

\newcommand{\calW}{\mathcal{W}}

\newcommand{\calA}{\mathcal{A}}
\newcommand{\calV}{\mathcal{V}}
\newcommand{\calG}{\mathcal{G}}

\newcommand{\calF}{\mathcal{F}}

\newcommand{\Identity}{\mathrm{Id}}

\newcommand{\PO}{\mathrm{PO}}

\usepackage{amsmath,amsfonts,bm,amsthm, txfonts}

\def\eqref#1{equation~\ref{#1}}
\def\1{\bm{1}}

\DeclareMathAlphabet{\mathsfit}{\encodingdefault}{\sfdefault}{m}{sl}
\SetMathAlphabet{\mathsfit}{bold}{\encodingdefault}{\sfdefault}{bx}{n}

\newcommand{\E}{\varmathbb{E}}

\newcommand{\R}{\varmathbb{R}}

\newcommand{\inner}{\mathrm{inner}}

\DeclareMathOperator*{\argmax}{arg\,max}
\DeclareMathOperator*{\argmin}{arg\,min}

\newtheorem{theorem}{Theorem}
\newtheorem{definition}{Definition}
\newtheorem{lemma}{Lemma}
\newtheorem{assumption}{Assumption}
\newtheorem{corollary}{Corollary}

\newcommand{\set}[1]{\left\{ #1 \right\}}
\newcommand{\norm}[1]{\left\lVert #1 \right \rVert}
\newcommand{\abs}[1]{\left| #1 \right|}

\newcommand{\frakm}{\mathcal{M}}

\usepackage{cleveref}
\begin{document}

\title{Performative Reinforcement Learning with Linear Markov Decision Process}

\author{  Debmalya Mandal\\Dept. of Computer Science\\University of Warwick, UK\\ \texttt{Debmalya.Mandal@warwick.ac.uk}
\and Goran Radanovi\'c\\Max-Planck Institute for \\Software Systems, Germany\\ \texttt{gradanovic@mpi-sws.org} }

\date{~}
\maketitle

%\iclrfinalcopy % Uncomment for camera-ready version, but NOT for submission.

\begin{abstract}%
We study the setting of \emph{performative reinforcement learning} where the deployed policy affects both the reward, and the transition of the underlying Markov decision process. Prior work~\citep{MTR23} has addressed this problem under the tabular setting and established last-iterate convergence of repeated retraining with iteration complexity explicitly depending on the number of states. In this work, we generalize the results to \emph{linear Markov decision processes} which is the primary theoretical model of large-scale MDPs. The main challenge with linear MDP is that the regularized objective is no longer strongly concave and we want a bound that scales with the dimension of the features, rather than states which can be infinite. Our first result shows that repeatedly optimizing a regularized objective converges to a \emph{performatively stable policy}. In the absence of strong concavity, our analysis leverages a new recurrence relation that uses a specific linear combination of optimal dual solutions for proving convergence. We then tackle the finite sample setting where the learner has access to a set of trajectories drawn from the current policy. We consider a reparametrized version of the primal problem, and construct an empirical Lagrangian which is to be optimized from the samples. We show that, under a \emph{bounded coverage} condition, repeatedly solving a saddle point of this empirical Lagrangian converges to a performatively stable solution, and also construct a primal-dual algorithm that solves the empirical Lagrangian efficiently. Finally, we show several applications of the general framework of performative RL including multi-agent systems.
\end{abstract}

\doparttoc
\faketableofcontents
\section{Introduction}
The success of reinforcement learning in challenging domains like Go~\citep{GO17}, Poker~\citep{Poker19}, language modeling~\citep{ouyang2022training} has led to its adoption in various human-facing systems. Indeed we are increasingly observing the use of RL in social media recommendations~\citep{RLRecSurvey22}, healthcare~\citep{YLNY21}, traffic modeling~\citep{wei2021recent}, and many other open-ended problems. However, the classical framework of reinforcement learning is not directly applicable in these reactive settings. For example, a news recommender system might affect people's preferences, and a traffic prediction system influences how drivers choose traffic routes~\citep{dulac2019challenges}. The common theme underlying these different settings is that there is  no fixed Markov decision process (MDP), but rather the underlying MDP changes in response to the deployed policy. 

Recently, \citet{MTR23} introduced the framework of \emph{performative reinforcement learning} to encompass such reactive setting in reinforcement learning. In performative RL, the deployed policy affects both the reward and the probability transition function of the underlying MDP. In contrary to the classical RL, there is no fixed optimal policy, rather the new goal is to find a \emph{performatively stable policy}, a policy which is optimal with respect to the updated MDP after deployment. The main result of \cite{MTR23} is that repeatedly optimizing a regularized value function converges to an approximate performative stable policy in a \emph{tabular} setting.

However, most of the practical applications of RL involve a large number of states, and requires \emph{function approximation} to deal with the complexity of the optimal policy or value function. Indeed, most of the  success of RL~\citep{mnih2015human, lillicrap2015continuous, GO17} can be attributed to handling large scale state spaces through deep neural networks. On the theoretical side, a significant effort has also been invested to design provable RL algorithms with various complexities of function approximation~\citep{JYWJ20, jin2021bellman, zhou2021nearly, yang2020reinforcement}. However, these works assume a \emph{static} setting as the underlying MDP doesn't change in response to the policy, and as discussed earlier, cannot  address the concerns of real-world applications. In this work, our goal is the design of provable RL algorithms with function approximation in the presence of \emph{performativity}. In particular, we ask the following questions.
\begin{quote}
    Can we design provably efficient RL algorithms that converge to \emph{performatively stable policy} under linear function approximation? Moreover, when there is only access to a finite number of samples, can we obtain stable policies with statistical and computational complexities depending only on the dimension of the features?
\end{quote}

The reader might ask whether existing works~\citep{MTR23, PZMH20} can be naturally extended to our setting. Although \citet{MTR23} showed that repeated optimization converges to a performatively stable policy, their approach doesn't translate to linear function approximation as the regularized objective is no longer a strongly concave function, and we need to devise new strategies for ensuring convergence. Furthermore, when we can access the underlying MDP only through trajectories collected from the deployed policy, we need to ensure that sample complexity only grows polynomially with the dimension of the features $D$, and is independent of the size of the state space, which can be infinite. In particular, our contributions are the following.
\begin{enumerate}
    \item We show that repeatedly optimizing a regularized objective converges to a \emph{performatively stable policy} in a linear MDP. In the absence of strong concavity of the objective, we establish a new recurrence relation depending on a time-varying linear combination of the optimal dual solutions of the regularized objective. 
    \item Furthermore, we show how to tune the strength of the regularization to obtain an approximate \emph{performatively} stable optimal policy based on the sensitivity of the environment. We then show that our method also converges to a \emph{performatively optimal policy} by establishing a bound on the distance between performatively optimal and stable policies.
    \item For the finite sample setting, we introduce a \emph{reparametrized} version of the primal problem. Based on the new reparametrized problem, we construct an empirical Lagrangian which is to be optimized from the samples, and show that, under a \emph{bounded coverage} condition, repeatedly solving a saddle point of this Lagrangian converges to a performatively stable policy, with a sample complexity polynomial in the dimension $D$ of the features. We also design an efficient algorithm for solving the empirical Lagrangian.
    \item Finally, we show several applications of the framework of performative RL involving stochastic Stackelberg games with one and multiple followers that further highlights the generality of the framework of performativity in RL.
\end{enumerate}

% We consider the framework of performative reinforcement learning where the policy affects both the reward and the probability transition function. In particular, \citet{MTR23} considered a setting where the occupancy measure of a deployed policy (say $d$) affects the reward ($r_d$) and the probability transition function ($P_d$). One of the main goals of performative reinforcement learning is to find a stable policy (or occupancy measure) which can be formulated as solving the following optimization problem.
% \begin{align}
% \label{eq:rl-dual}
%     \begin{split}
%         \max_d &\ \sum_{s,a} d(s,a) r_d(s,a) \\
%         \textrm{s.t.} &\ \sum_a d(s,a) = \rho(s) + \gamma \cdot \sum_{s',a } P_d(s \mid s',a) d(s',a) \ \forall s\\
%         &d(s,a) \ge 0 \ \forall s,a
%     \end{split}
% \end{align}
% In this paper, we consider the setting where the number of states is not finite and can even be infinite. We will assume that each state, action pair is represented by $D$ dimensional features, and the reward and the transition probability functions are linear functions of these features. We aim to generalize the results of \cite{MTR23}, and show that repeated retraining (with appropriate regularization) can converge to a stable policy even for linear MDPs. %We first aim to solve the optimization problem~\ref{eq:regularized-rl-dual} with time complexity polynomial in $D, A$ and $1/\varepsilon$ but independent of the size of the state-space. 

\subsection{Related Work}

We recognize two important lines of related work: {\em performative prediction} and {\em reinforcement learning}.

\textbf{Performative Prediction}. Performative prediction was introduced by \citet{PZM+20}, who proposed two important solution concepts, performative stability and optimality, and established convergence guarantees of repeated retraining approaches. A series of papers extended these results~\citep{MPZH20, MPZ21, izzo2021learn, lu2023bilevel, yan2024zero} and studied different variants of the canonical performative prediction setting. These variants include multi-agent settings~\citep{narang2023multiplayer, li2022multi, BHK20, piliouras2023multi} or {\em stateful} settings~\citep{BHK20,LW22,RRL+22,izzo2022learn}, where the distribution shifts are gradual. We refer the reader to \citet{hardt2023performative} for an extensive overview of the results related to performative prediction. In contrast to this line or work, we do not focus on prediction, but sequential decision making. 

\textbf{Performative Reinforcement Learning}. the results in this paper are most related to recent works on performative RL~\citep{MTR23,rank2024performative,pollatos2025corruption}. As we already mentioned in the introduction, we extend the framework of~\citet{MTR23} by considering linear MDPs. The framework of~\citet{rank2024performative} is orthogonal to ours as it considers tabular MDPs but also a different model of pefromativity, similar to the stateful performative prediction setting. Additionally, \citet{caiperformative} recently extended the performative prediction setting to dynamical systems. Their algorithm and analysis exploit linearity similar to ours. %They show that more sophisticated repeated retraining approaches can yield better convergence guarantees when the change in the underlying MDP is gradual.  In our model, we do not consider gradual shifts.

%Our work is also related to stochastic games, reinforcement learning (function approximation and non-stationarity). However, because of limited space, we discuss them in the Appendix. 
\textbf{Stochastic Games}. The performative RL framework is also closely related to stochastic games~\citep{shapley1953stochastic} and multi-agent RL~\citep{zhang2021multi}. More specifically, the framework relates to settings that consider commitment policies~\citep{von2010market, letchford2012computing,vorobeychik2012computing,dimitrakakis2017multi,zhong2021can}, in which a leader agent commits a policy to which a follower (or followers) respond by optimizing its utility function. %These works have studied computational and statistical complexity of finding optimal commitment policies for the leader. 
Performative RL provides a more general abstraction, removing the need to prescribe a specific utility function to the follower. %Instead,  the literature on performative RL considers sensitivity assumptions regarding on the environment shift relative to the changes in the deployed policy.  

\textbf{Reinforcement Learning Theory}. Our work is related to the literature on function approximation in reinforcement learning~\citep{JYWJ20, yang2020reinforcement, zhou2021nearly,  du2021bilinear}. We adopt a common modeling assumption about the structural properties of the environment, formalizing it as the linear Markov decision process (e.g., see \cite{JYWJ20}). 
For finite sample setting, our approach is based on offline RL~\citep{levine2020offline, jin2021pessimism} and we utilize the results of~\cite{GNOP23} to establish finite sample guarantees. However, it is important to note that offline RL does not consider a type of distribution shift that we study here. Namely, in offline RL, a deployed policy does not affect the environment.  

\textbf{Non-Stationarity in RL}. More broadly, our work falls under the scope of non-stationary reinforcement learning~\citep{CSZ20, WL21}. However, the main difference is that the environment shift in our setting is induced by the policy, whereas existing work in non-stationary RL typically assumes policy-independent but bounded amount of shifts~\citep{besbes2014stochastic}.

%As recognized by~\citet{rank2024performative}, similar sensitive assumptions on the follower's repose model to the changes in the leader's policy are used in prior work that studies learnability of commitment policies~\citep{radanovic2019learning}. 
\section{Model}
We consider Markov Decision Processes (MDPs) with a state space $S$, action set $A$, discount factor $\gamma$, and starting state distribution $\rho$. The reward and the probability transition functions of the MDP depend on the adopted policy. We consider infinite-horizon setting where the learner's goal is to  maximize the total sum of discounted rewards. 
%We will write $s_k$ to denote the state visited at time-step $k$ and $a_k$ to denote the action taken at time-step $k$.
%
When the learner adopts policy $\pi$, the underlying MDP has reward function $r_\pi$ and probability transition function $P_\pi$. We will write $M(\pi)$ to denote the corresponding MDP, i.e., $M(\pi) = (S, A, P_\pi, r_\pi, \rho)$. %Note that only the reward and the transition probability function change according to the policy $\pi$. 

When the learner adopts policy $\pi$ and the underlying MDP is $M(\pi') = (S,A,P_{\pi'}, r_{\pi'}, \rho)$, $V^\pi_{\pi'}(\rho)$ denotes the value function, i.e., the expected sum of discounted rewards given the starting state distribution $\rho$. In particular, we will refer to $V^\pi_\pi(\rho)$ as the \emph{performative value function} which can be interpreted as follows. The learner adopts policy $\pi$, in response the MDP changes to $M(\pi)$ and then $V^\pi_\pi(\rho)$ is the learner's value function in the new MDP $M(\pi)$. We next define the \emph{performatively optimal policy}, which  maximizes performative value function.

%the probability of a trajectory $\tau = (s_k, a_k)_{k=0}^{\infty}$ is given as $\Pro(\tau) = \rho(s_0) \prod_{k=1}^\infty P_{\pi'}(s_{k+1}|s_k, \pi(s_k))$. We will write $\tau \sim \Pro^{\pi}_{\pi'}$ to denote such a trajectory $\tau$. Given a policy $\pi$ and an underlying MDP $M(\pi')$ we write $V^{\pi}_{\pi'}(\rho)$ to define the value function w.r.t. the starting state distribution $\rho$. This is defined as follows.
%\begin{equation}
%    \label{defn:value-function}
%    V^{\pi}_{\pi'}(\rho) = \E_{\tau \sim \Pro^{\pi}_{\pi'}} \left[\sum_{k=0}^\infty \gamma^k r_{\pi'}(s_k,a_k) \rvert \rho \right]
%\end{equation}

\begin{definition}[Performatively Optimal Policy]\label{defn:perf-optimality}
A policy $\pi_P$ is performatively optimal if it maximizes performative value function, i.e., %\goran{typo? $\pi_{\theta'}$}
$
\pi_P \in \argmax_{\pi'} V^{\pi'}_{\pi'}(\rho)
$.
% $$
% \pi \in \argmax_{\pi'} V^{\pi'}_{\pi'}(\rho)
% $$
\end{definition}
Although, $\pi_{P}$ maximizes the performative value function, it need not be stable, i.e., it need not be optimal with respect to the changed environment $M(\pi_P)$. We next define the notion of performatively stable policy which captures this notion of stability.
\begin{definition}[Performatively Stable Policy]
A policy $\pi_S$ is \emph{performatively stable} if it satisfies the
% following condition
condition
$
\pi_S \in \argmax_{\pi'} V^{\pi'}_{\pi_S}(\rho)
$.
\end{definition}
The definition of performatively stable policy implies that if the underlying MDP is $M(\pi_S)$ then an optimal policy is $\pi_{S}$. It is not a priori clear if a performatively stable policy always exists as the value function is a non-convex function of policies for most standard parametric representations including softmax and direct parameterization. And, for non-convex functions, it is not always possible to guarantee the existence of a fixed point. However, \citet{MTR23} observed that a performatively stable occupancy measure exists. In particular, given a policy $\pi$, let us define its long-term discounted state-action occupancy measure in the MDP $M(\pi)$ 
% is defined as follows.
 as
$
d^\pi(s,a) = \E_{\tau \sim \varmathbb{Pr}^\pi_\pi} \left[\sum_{k=0}^\infty \gamma^k \mathrm{1}\set{s_k = s, a_k = a}\mid \rho\right]
$.
Given such an occupancy measure $d$, one can consider the following policy $\pi^d$.
\begin{align}\label{eq:measure-to-policy}
    \pi^d(a|s) = \left\{\begin{array}{cc}
        \frac{d(s,a)}{\sum_b d(s,b)} & \textrm{ if } \sum_b d(s,b) > 0 \\
        \frac{1}{A} & \textrm{ otherwise }
    \end{array} \right.
\end{align}
With this definition, we can pose the problem of finding a performatively stable occupancy measure. 
An occupancy measure $d_S$ is performatively stable if it is the optimal solution of the following problem.
\begin{align}\label{eq:rl-dual}
    \begin{split}
        d_S \in \ &\argmax_{d: d \ge 0} \sum_{s,a} d(s,a) r_d(s,a)\\
        \textrm{s.t.} & \sum_a d(s,a) = \rho(s) + \gamma \cdot \sum_{s',a} d(s',a) P_d(s',a,s) \ \forall s 
        %&d(s,a) \ge 0 \ \forall s,a
    \end{split}
\end{align}
With slight abuse of notation we are writing $r_d \coloneqq  r_{\pi_d}$ and $P_d \coloneqq P_{\pi^d}$ (as defined in \cref{eq:measure-to-policy}). Note that the objective is linear and the feasible region is a convex set. Therefore, fixed-point theorems can be used to show that a performatively stable occupancy measure $d_S$ always exists. 

Now suppose the learner has found a stable occupancy measure $d_S$ and the corresponding MDP is $M_S = M(d_S) = M(\pi^{d_S})$. Then after deploying the policy $\pi^{d_S}$ the resulting occupancy measure is $d_S$ and the learner doesn’t want to re-optimize. 

\textbf{Linear Markov Decision Process}. We now introduce the definition of the linear Markov decision process~\citep{JYWJ20}. We assume known features $\phi : \calS \times \calA \rightarrow \R^D$ where each state, action pair $(s,a)$ is represented by the feature $\phi(s,a)$. Although the classic paper of \cite{JYWJ20} consider infinite state-space, we will assume finite state-space and write $\Phi \in \R^{SA \times D}$ to denote the feature matrix. However, the state-space can be arbitrarily large compared to the dimension $D$ of the features. We assume finite state-space but we believe our argument can be generalized to infinite state-space as well. In particular, observe that, one can always approximate a continuous and bounded state space by a finite state-space up to any desired accuracy and consider the feature matrix for that finite state-space. 

In our setting, the MDP's are parameterized by the occupancy measure $d$. Given an occupancy measure $d$, let $r_d$ be the corresponding reward function and $P_d$ be the corresponding probability transition function. Then there exist unknown parameters $\theta_d \in \R^D$ and $D$-dimensional measure $\bm{\mu}_d = (\mu^1_d,\ldots, \mu^D_d)$ such that the following holds.
\begin{equation}
    \label{eq:defn-linear-MDP}
    r_d(s,a) = \left \langle \phi(s,a), \theta_d\right \rangle\quad P_d(s' \mid s,a) = \left \langle \phi(s,a), \mu_d(s')\right \rangle
\end{equation}
We will assume that the parameters are bounded, i.e., $\norm{\theta_d}_2 \le \sqrt{D}$ and $\bm{\mu}_d(\calS) \le \sqrt{D}$ for any occupancy measure $d$. In matrix notations, we can write reward $r_d = \Phi \theta_d$. Moreover, let $P_d \in \R^{S \times SA}$ be the probability transition matrix with entries $P_d(s' ; s,a) = P_d(s' \mid s,a)$. Then we have $P_d = \bm{\mu}_d \Phi^\top$. Substituting the expression of $r_d$ and $P_d$ in the optimization problem~\cref{eq:rl-dual} we get the following problem.
\begin{align}
\label{eq:rl-matrix-dual}
    \begin{split}
        \max_{d: d \ge 0} &\ d^\top \Phi \theta_d \\
        \textrm{s.t.} &\ Bd = \rho + \gamma \cdot \bm{\mu}_d \Phi^\top d\\
        %&d \ge 0 
    \end{split}
\end{align}
Here the matrix $B \in \R^{S \times SA}$ is defined as follows.
\begin{align*}
B(s;(s',a') ) = \left\{ \begin{array}{cc}
    1 &  \textrm{if } s'=s\\
    0 & \textrm{o.w.}
\end{array}\right.
\end{align*}

\section{Repeated Retraining}
In order to obtain a stable policy, we first design a repeated optimization scheme. Let $r_t$ (resp. $P_t$) be the reward (resp. transition probability) at iteration $t$. Under the assumption of linear MDP, we can equivalently assume that the relevant parameters at iteration $t$ are $\theta_t$ and $\bm{\mu}_t$. Then we solve the following regularized optimization problem to obtain the new occupancy measure $d_{t+1}$.
\begin{align}
\label{eq:regularized-rl-matrix-dual}
    \begin{split}
        \max_{d: d \ge 0} &\ d^\top \Phi \theta_t  - \frac{\lambda}{2} d^\top \Phi \Phi^\top d \\
        \textrm{s.t.} &\ Bd = \rho + \gamma \cdot \bm{\mu}_t \Phi^\top d
        %&d \ge 0 
    \end{split}
\end{align}
Given an occupancy measure $d$, we will write $\theta_d = \calF_\theta(d)$ to denote the $D$-dimensional parameter for the reward, and $\bm{\mu}_d = \calF_\mu(d)$ to denote the $D$-dimensional measure $\bm{\mu}_d$. Moreover, given an occupancy measure $d$ we deploy the policy $\pi_d$ as defined in \cref{eq:measure-to-policy}.
% \begin{equation}
%     \label{eq:measure-to-policy}
%     \pi_d(a \mid s) = \left\{\begin{array}{cc}
%         \frac{d(s,a)}{\sum_b d(s,b)}   & \textrm{ if } \sum_b d(s,b) > 0  \\
%         \frac{1}{A} & \textrm{ o.w. } 
%     \end{array} \right.
% \end{equation}

In response to the deployed policy $\pi_d$, the underlying environment changes and generates parameters $\theta_d$ and $\bm{\mu}_d$. We will assume that if two occupancy measures generate the same policy (according to \cref{eq:measure-to-policy}) then the corresponding parameters are also the same.
\begin{assumption}\label{asn:measure-to-parameters}
    For any two occupancy measures $d_1$ and $d_2$ such that $\pi_{d_1} = \pi_{d_2}$ (as defined in \cref{eq:measure-to-policy}), it follows that $\theta_{d_1} = \theta_{d_2}$ and $\bm{\mu}_{d_1} = \bm{\mu}_{d_2}$.
\end{assumption}
 The above assumption essentially says that the environment responds in terms of the induced policy $\pi_d$ for an occupancy measure $d$, and if $d_1 \neq d_2$ but $\pi_{d_1} = \pi_{d_2} = \pi$ then the environment changes identically for these two cases. 
\Cref{alg:repeated-optimization} details the repeated optimization method. In order to prove convergence of \Cref{alg:repeated-optimization} we will make the following assumptions.

\begin{algorithm}[!h]
Initialize occupancy measure $d_0$.\\
\For{$t = 1,\ldots$}{
Obtain parameters $\theta_{t}=\calF_\theta(d_{t-1})$ and $\bm{\mu}_{t} = \calF_\mu(d_{t-1})$.\\
Solve optimization problem~\cref{eq:regularized-rl-matrix-dual} to obtain occupancy measure $d_{t}$.\\
Deploy policy $\pi_{t}$ given as\\
\begin{flalign}
&\pi_{t}(a|s) = \left\{ \begin{array}{cc}
    \frac{d_t(s,a)}{\sum_b d_t(s,b)} & \textrm{ if } \sum_b d_t(s,b) > 0 \\
    \frac{1}{A} & \textrm{ otherwise }
\end{array}\right.&&
\end{flalign}
}
\caption{Repeated Optimization~\label{alg:repeated-optimization}}
\end{algorithm}

%We now prove the convergence of algorithm~\eqref{alg:repeated-optimization}. We will make the following assumptions.
\begin{assumption}\label{asn:lipschitzness}
    The  mappings $(\calF_\theta(\cdot), \calF_\mu(\cdot))$ are $(\varepsilon_\theta, \varepsilon_\mu)$-sensitive i.e. the following holds for any two occupancy measures $d$ and $d'$
 \begin{align*}
    &\norm{\calF_\theta(d) - \calF_\theta(d')}_2 \le \varepsilon_\theta \norm{d - d'}_2\ \textrm{and}\\ &\norm{\calF_\mu(d) - \calF_\mu(d')}_2 \le \varepsilon_\mu \norm{d - d'}_2
\end{align*}
\end{assumption}
\begin{assumption}\label{asn:features}
    The matrix $\Phi$ has rank $D$ and satisfies $\lambda_{\max}(\Phi^\top \Phi) \le \frakm$ for some constant $\frakm$. Moreover, for any two valid occupancy measures $d, d'$ (i.e.,  $d,d' \in \{d\in \R^{SA}: d\ge 0 \textrm{ and } Bd = \rho + \gamma \mu_d \Phi^\top d\}$), we have $(d-d')^\top \Phi \Phi^\top (d - d') \ge \kappa \norm{d - d'}_2^2$ for some $\kappa > 0$.
\end{assumption}

% \begin{assumption}\label{asn:features-psd}
%     Let $\Phi = [\Phi_1^\top;\Phi_2^\top;\ldots;\Phi_A^\top ]^\top$. For any $a,b \in [A]$, we have $\mu_2 \Identity_d \succcurlyeq \Phi_a^\top \Phi_b \succcurlyeq \mu_1 \Identity_d$. 
% \end{assumption}

% \begin{assumption}\label{asn:measures}
%     For any occupancy measure $d$, we have $\bm{\mu}_d^\top \bm{\mu} \succcurlyeq c_1 \Identity_d$.
% \end{assumption}

Assumption \ref{asn:lipschitzness} is standard in the literature on performative prediction, and states that the environment doesn't change too much if the occupancy measure of the deployed policy doesn't change much. The second assumption \ref{asn:features} has two parts. First, the feature matrix has full column rank and $\lambda_{\max}(\Phi^\top \Phi) \le \frakm$. Suppose the rank of the feature matrix is strictly less than $d$, then it is possible to reduce the dimension of the features through orthogonalization. We require the assumption of bounded eigenvalue only because we prove convergence of occupancy measures in $L_2$ norm which is much stronger than the convergence in $L_1$ norm for an object of arbitrarily large dimension. Since $\norm{\Phi}_1 \le 1$, without the bounded eigenvalue assumption, we can prove convergence of occupancy measures in $L_1$ norm. 

The second part of the assumption concerns the row-space of the feature matrix $\Phi$, and says that for two different occupancy measures $d,d'$  the $L_2$-norm of $\Phi^\top (d-d')$ is at least  $\sqrt{\kappa}$ times the $L_2$-norm of $d-d'$. %If $\Phi^\top d \approx \Phi^\top d'$ then the distribution of features under $d$ and $d'$ are the same, despite the two being different occupancy measures. 
Without this assumption, there can be many occupancy measures $d$ such that $\Phi^\top d = \Phi^\top d_S$ where $d_S$ is the stable occupancy measure, and \Cref{alg:repeated-optimization} can converge to one such measure, and not necessarily to $d_S$.

\begin{theorem}\label{thm:convergence-rpo}
    Suppose assumptions~\ref{asn:measure-to-parameters},~\ref{asn:lipschitzness}, \ref{asn:features} hold and $\alpha = \frac{\sqrt{\frakm}}{\sqrt{A}(1-\gamma)}$ and $\varepsilon_\mu < \frac{2 \sqrt{\kappa}}{25 \gamma \alpha^2}$. If \Cref{alg:repeated-optimization} is run with regularization parameter $\lambda > \frac{25\left(\varepsilon_\theta + \alpha \gamma \sqrt{D} \varepsilon_\mu \right)}{8\sqrt{\kappa}}$ then we are guaranteed that
    $$\textstyle 
    \norm{d_{t} - d_S}_2 \le \delta \quad \forall t \ge \ln \left( \frac{2}{\delta(1-\gamma)}\right)/\ln(1/r),
    $$
    where $r = \frac{5}{4} \sqrt{\frac{\varepsilon_\theta + \alpha \gamma \sqrt{D} \varepsilon_\mu   }{\lambda \sqrt{\kappa}} + \frac{4 \gamma \varepsilon_\mu \alpha^2}{\sqrt{\kappa}}} < 1$.
\end{theorem}

The full proof is provided in \Cref{sec:convergence-rpo-proof}. Here we discuss the main challenges and the differences with the approach taken in \cite{MTR23}. 
 We consider the dual formulation of the optimization problem defined in \cref{eq:regularized-rl-matrix-dual}. 
 For the tabular setting, prior approach showed that the sequence of dual optimal solutions converge to a stable solution, and used the convergence of the dual optimal solutions to establish convergence of the sequence of primal optimal solutions $\{d_t\}_{t \ge 1}$. 
 However, for our setting, the dual problem need not be strongly-convex and we cannot use the same proof strategy. We first use assumption \ref{asn:features} to establish the following recurrence relation.
    \begin{equation}\label{eq:temp-bound-difference}
    \norm{d_{t+1} - d_S}_2 \le \frac{1}{\lambda \sqrt{\kappa}}\left(\varepsilon_\theta \norm{d_t - d_S}_2 + \norm{u_{t+1} - u_S}_2 \right)
    \end{equation}
    Here, $u_t$ is a linear combination of dual optimal solution $(g_t, h_t)$, i.e., $u_t = M_t h_t + W g_t$ (similarly $u_S = M_s h_S + W g_S$). Then we focus on bounding the difference $\norm{u_{t+1} - u_S}_2$. 
    We use first-order optimality conditions of the dual problem at $(g_t, h_t)$, and assumption \ref{asn:lipschitzness} to establish the following bound.
    \begin{align*}
    \norm{u_{t+1} - u_S}_2 &\le \left(\varepsilon_\theta + \gamma \varepsilon_\mu \norm{h_{t+1}}_2 + 3\lambda \gamma \varepsilon_\mu \norm{M_t^\dagger}_2^2 \right) \norm{d_{t-1} - d_S}_2 + \gamma \varepsilon_\mu \norm{h_{t+1}}_2 \norm{d_t - d_S}_2
    \end{align*}
    We then use a bound on the dual optimal solution $h_{t+1}$(\Cref{lem:bound-dual-solution}), and assumption \ref{asn:features} (i.e., bound on the norm of ${M_t^\dagger}$) to establish a bound of the form $\norm{u_{t+1} - u_S}_2 \le \alpha_1 \norm{d_{t-1} - d_S}_2 + \alpha_2 \norm{d_t - d_S}_2$. Substituting this bound in \cref{eq:temp-bound-difference}, we can establish the following recurrence relation: $\norm{d_{t+1} - d_S}_2 \le \beta_1/\lambda \norm{d_t - d_S}_2 + \beta_2 /\lambda \norm{d_{t-1} - d_S}_2$ for some constants $\beta_1, \beta_2$. Finally, by choosing an appropriate value of  $\lambda$, we can ensure that the sequence $\{d_t\}_{t\ge 1}$ is a contraction.

We obtain a linear convergence rate similar to the result of \cite{MTR23}, i.e., the number of iterations required to obtain a $\delta$-approximate stable occupancy measure is $O(\log(1/\delta))$. Moreover, when instantiated to the tabular setting (i.e., $D = SA$) the required level of regularization is $\lambda = O\left( \varepsilon_\theta + \gamma \frac{{S}\sqrt{A}}{1-\gamma} \varepsilon_\mu\right)$ for $\kappa = O(1)$ and $\frakm = SA$. This is an improvement by a factor of $S^{3/2}/(1-\gamma)^3$ compared to prior work. In order to understand why the strength of regularization is important, we next present two results that show why the required value of $\lambda$ controls the approximation with respect to the unregularized objective.

\subsection{Approximating the Unregularized Objective}
\Cref{thm:convergence-rpo} proves that repeatedly optimizing the regularized objective of ~\cref{eq:regularized-rl-matrix-dual} converges to a stable solution (say $d^\lambda_S$). We can show that this stable solution also approximates the performatively stable and optimal solution with respect to the original unregularized objective. Because of limited space, here we provide informal statements of the claims, and full details are provided in \Cref{sec:apx-stability-proof} and \Cref{sec:apx-optimality-proof}. We will require the following definition of an  approximately stable policy.
\begin{definition}
    An occupancy measure $d_S$ is $\beta$-approximately stable if
    $$
    d_S^\top r_S \ge \max_{d \in \calC(P_S)} d^\top r_S - \beta.
    $$
    Here $r_S = r_{d_S}$ (resp. $P_S=P_{d_S}$) is the reward (probability transition) induced by the policy $\pi^{d_S}$, and $\calC(P_S)$ is the set of valid occupancy measures with respect to $P_S$.
\end{definition}
The next theorem states that the stable occupancy measure according to the regularized objective ($d^\lambda_S$) is approximately stable with respect to the unregularized objective.

\begin{theorem}\label{thm:apx-stability}
    Suppose the assumptions of \Cref{thm:convergence-rpo} hold. Then there exists a choice of regularization parameter ($\lambda$) such that repeatedly optimizing objective define in ~\cref{eq:regularized-rl-matrix-dual} converges to a stable solution $d^\lambda_S$ that is $\frac{25\frakm \left( \varepsilon_\theta + \alpha \gamma \sqrt{D} \varepsilon_\mu \right) }{16 \sqrt{\kappa}(1-\gamma)^2}$-approximately stable with respect to the unregularized objective.
\end{theorem}
Note that as the performativity vanishes, i.e., $\varepsilon_\theta, \varepsilon_\mu \rightarrow 0$, the solution $d^\lambda_S$ converges to the performatively stable solution with respect to the unregularized objective. We next turn to bounding the approximation with respect to the performatively optimal policy. Let $d_{\PO}$ be the performatively optimal occupancy measure. With a slight abuse of notation, we will write $V^{d_\PO}_{d_\PO}(\rho) = V^{\pi^{d_\PO}}_{\pi^{d_\PO}}(\rho)$ to denote the value of the policy $\pi^{d_\PO}$ in the MDP induced by $\pi^{d_\PO}$.  
 
\begin{theorem}[Informal Statement]\label{thm:apx-optimality}
    Suppose the assumptions of \Cref{thm:convergence-rpo} hold, and $\Delta = \frac{3 \gamma \varepsilon_\mu \frakm \sqrt{D} }{(1-\gamma)^2}  + \varepsilon_\theta \sqrt{\frakm}$, and $\lambda_0 = \frac{25}{8\sqrt{\kappa}}\left( \varepsilon_\theta + \alpha \gamma \sqrt{D}\varepsilon_\mu\right)$. Then there exists a choice of regularization parameter $(\lambda)$ such that repeatedly optimizing objective~(\ref{eq:regularized-rl-matrix-dual}) converges to a solution $d^\lambda_S$ with the guarantee that $$V^{d_{\PO}}_{d_{\PO}}(\rho)  - V^{d^\lambda_S}_{d^\lambda_S}(\rho) \le O\left(\max\set{1,\Delta} \Delta + \lambda_0 \right).$$
\end{theorem}
The full statement of the theorem requires introducing new definitions. Hence we provide the formal statement and other details in the appendix \Cref{sec:apx-optimality-proof}. The main step of the proof is to show the distance between the performatively optimal solution and performatively stable solution can be bounded by $O(\Delta/\kappa \lambda)$ (\cref{lem:distance-optimality-stability}). Note that the above theorem bounds the gap in performative value function under $d_{\PO}$ and $d^\lambda_S$. As $\varepsilon_\theta, \varepsilon_\mu \rightarrow 0$ both $\Delta, \lambda_0$ approach zero,  the suboptimality gap approaches zero.

\section{Finite Sample Setting} \label{sec:finite-sample}
In the previous section, we assumed the learner has full access to the new model $(P_t, r_t)$ at every iteration $t$. In this section, we relax this assumption and consider a setting where the learner deploys policy $\pi_t$ in MDP $M_t$ and only accesses samples through this deployed policy. In particular, we assume the following data generation process at iteration $t$.

\textbf{Data}: Given the occupancy measure $d_t$, let $\Tilde{d}_t$ be the normalized occupancy measure defined as $\Tilde{d}_t(s,a) = (1-\gamma) d_t(s,a)$ for any state, action pair $(s,a)$. For each $j \in [m_t]$, we first sample a starting state $s^0_j \sim \rho(\cdot)$, then sample $(s_j, a_j) \sim \Tilde{d}_t(\cdot)$. Finally, the next state $s'_j \sim P_t(\cdot \mid s_j, a_j)$ and reward $r_j \sim r_t(s_j, a_j)$. Therefore, for each $j \in [m_t]$, we have the tuple $(s^0_j, s_j, a_j, r_j,s'_j)$ and the collection of $m_t$ such tuples constitute the dataset $\calD_t$ at iteration $t$.

We will follow an approach similar to \citet{MTR23} by constructing the Lagrangian corresponding to the optimization problem~\Cref{eq:regularized-rl-matrix-dual}, and then solving for a saddle point of the Lagrangian. In order for this approach to work, we need to show that the empirical version of the Lagrangian is close to the true Lagrangian through an $\varepsilon$-net construction. However, the standard Lagrangian of \Cref{eq:regularized-rl-matrix-dual} has infinite dimensional variables, and the size of any $\varepsilon$-net is unbounded, and hence the previous argument cannot be applied. Therefore, we introduce a reparametrization of the primal variable $d$ by introducing a finite-dimensional variable. Let us rewrite the optimization problem~\Cref{eq:regularized-rl-matrix-dual} by introducing the variable $\nu = \Phi^\top d$. Note that $\nu \in \R^D$ whereas $d$ can be infinite-dimensional.
\begin{align}
\label{eq:regularized-rl-matrix-dual-new}
    \begin{split}
        \max_{\nu, d\ge 0} &\ \nu^\top \theta_t  - \frac{\lambda}{2} \nu^\top \nu \\
        \textrm{s.t.} &\ Bd = \rho + \gamma \cdot \bm{\mu}_t \nu \\
        &\nu = \Phi^\top d\\
       % &d \ge 0 
    \end{split}
\end{align}
The Lagrangian of the optimization problem~\ref{eq:regularized-rl-matrix-dual-new} is  
\begin{align*}
    \calL(d,\nu; g, \omega) &= \nu^\top \theta_t  - \frac{\lambda}{2} \nu^\top \nu + \left \langle g, \rho + \gamma \cdot \bm{\mu}_t \nu - Bd \right \rangle  + \left\langle \omega, \Phi^\top d - \nu \right \rangle
\end{align*}
We now introduce an empirical version of the Lagrangian which can be estimated from the data. Let us write $\Sigma_t = \E_{(s,a) \sim \pi_t}[\phi(s,a) \phi(s,a)^\top]$ as the expected covariance matrix of the features observed under policy $\pi_t$ in the model $M_t$. Then it is possible to derive the following equivalent expression of the Lagrangian $\calL(\cdot)$ (the proof is in \Cref{subsec:derivation-lagrangian}).
\begin{align}
    \calL(d,\nu; g, \omega) &= \nu^\top \left( \theta_t  + \gamma \cdot \bm{\mu}^\top_t g - \omega \right) - \frac{\lambda}{2} \nu^\top \nu   + \left \langle g, \rho \right \rangle + \left \langle d, \Phi \omega - B^\top g\right \rangle \label{defn:Lagrangian}\\
%\end{flalign}
    %&= \nu^\top \Sigma_t^{-1} \Sigma_t \left( \theta_t  + \gamma \cdot \bm{\mu}_t^\top g - \omega \right) - \frac{\lambda}{2} \nu^\top \nu + \left \langle g, \rho \right \rangle + \left \langle d, \Phi w - B^\top g\right \rangle\nonumber \\
    %&= \nu^\top \Sigma_t^{-1} \E_{(s,a) \sim \pi_t}\left[ \phi(s,a) \phi(s,a)^\top \left( \theta_t  + \gamma \cdot \bm{\mu}_t\nu - \omega \right) \right] - \frac{\lambda}{2} \nu^\top \nu + \E_{s^0 \sim \rho}\left[ g(s^0) \right] + \left \langle d, \Phi w - B^\top g\right \rangle\nonumber \\
    %&= \nu^\top \Sigma_t^{-1} \E_{(s,a) \sim \pi_t}\left[ \phi(s,a) r_t(s,a)  + \gamma \cdot g^\top P_t(\cdot \mid s,a) - \phi(s,a) \phi(s,a)^\top \omega \right] \nonumber \\&- \frac{\lambda}{2} \nu^\top \nu + \E_{s^0 \sim \rho}\left[ g(s^0) \right] + \left \langle d, \Phi w - B^\top g\right \rangle\nonumber \\
%    \begin{align}
    &= \nu^\top \Sigma_t^{-1} \mathop{\E}_{\stackrel{(s,a) \sim \pi_t}{ s' \sim P_t(\cdot \mid s,a)} }\left[ \phi(s,a) r_t(s,a)  + \gamma \cdot g(s')  - \phi(s,a) \phi(s,a)^\top \omega \right] \nonumber - \frac{\lambda}{2} \nu^\top \nu  + \mathop{\E}_{s^0 \sim \rho}\left[ g(s^0) \right] + \left \langle d, \Phi \omega - B^\top g\right \rangle\nonumber 
\end{align}
This motivates the following choice of the empirical Lagrangian.
\begin{align}
\label{eq:empirical-Lagrangian}\textstyle 
\begin{split}
    &\widehat{\calL}_t(d, \nu; g, \omega) = \nu^\top \Sigma_t^{-1} \cdot \frac{1}{m_t} \sum_{j=1}^{m_t} \phi(s_j,a_j) \left( r_t(s_j,a_j)  + \gamma \cdot g(s'_j) -  \phi(s_j,a_j)^\top \omega \right)   - \frac{\lambda}{2}\nu^\top \nu + \frac{1}{m_t} \sum_{j=1}^{m_t} g(s^0_j) + \left \langle d, \Phi \omega - B^\top g\right \rangle  
    \end{split}
\end{align}

\Cref{alg:repeated-optimization-finite} depicts our approach for the finite sample setting, which repeatedly solves a saddle point of the empirical Lagrangian defined in \Cref{eq:empirical-Lagrangian}. Since this is an infinite-dimensional optimization problem,  it is not obvious that the problem can be solved. We first assume that we can exactly solve the optimization problem, and provide convergence guarantees for \Cref{alg:repeated-optimization-finite}. In the next subsection, we will show how to efficiently solve the saddle point optimization problem. We will make the following assumption regarding the deployed policy.

%The choice of the Lagrangian is motivated by the recent work of \citet{GNOP23} who considered a similar (unregularized) objective in offline reinforcement learning and observed that it is possible to solve the Lagrangian by performing gradient ascent over $\nu$ and descent over $\omega$, both of which are $D$-dimensional vectors. The main observation of \cite{GNOP23} was that the parameters $d$ and $g$ could be symbolically represented with respect to the parameters $\nu, \omega$ and policy $\pi$, and therefore, it is possible to avoid performing gradient steps over $d$, and $g$ which can be infinite-dimensional. Algorithm~\eqref{alg:repeated-optimization-finite} depicts our repeated optimization method for the finite sample setting.

\begin{algorithm}[!h]
\caption{Repeated Optimization from Finite Samples~\label{alg:repeated-optimization-finite}}
Initialize initial policy $\pi_0$.\\
\For{$t = 1,\ldots$}{
Deploy policy $\pi_t$ and collect dataset $\calD_t$ of size $m_t$.\\
Solve the min-max  problem.
\begin{equation}\label{eq:saddle-point-optimization}
(\widehat{d}_t, \widehat{\nu}_t; \widehat{g}_t, \widehat{\omega}_t) \leftarrow \max_{d,\nu\in \R^D} \min_{g,\omega \in \R^D} \widehat{\calL}_t(d,\nu; g,\omega)
\end{equation}
Set policy $\pi_{t+1}$ as 
\begin{flalign}
& \pi_{t+1}(a|s) = \left\{ \begin{array}{cc}
    \frac{\widehat{d}_t(s,a)}{\sum_b \widehat{d}_t(s,b)} & \textrm{ if } \sum_b  \widehat{d}_t(s,b)  > 0 \\
    \frac{1}{A} & \textrm{ otherwise }
\end{array}\right.&&
\end{flalign}
}
\end{algorithm}
\begin{assumption}[Bounded Coverage Ratio]\label{asn:bounded-coverage}
    For any policy $\pi$, let $d_\pi$ be the occupancy measure of the policy $\pi$ in the MDP $M_\pi$, and let $\Sigma_\pi = \E_{(s,a) \sim d_\pi}[\phi(s,a) \phi(s,a)^\top]$. Moreover, let $d_\star$ be the occupancy measure of the optimal policy in $M_\pi$\footnote{This is the standard optimal policy in the MDP $M_\pi$ and not the performatively optimal policy.}. Then, there exists a constant $B > 0$ such that
    $$
    \E_{(s,a)\sim d_\star}[\phi(s,a)^\top] \ \Sigma_\pi^{-2} \ \E_{(s,a)\sim d_\star}[ \phi(s,a)] \le B.
    $$
\end{assumption}

Assumption \ref{asn:bounded-coverage} generalizes single-policy concentrability assumption  popular in the literature on offline reinforcement learning~\cite{jin2021pessimism}. In offline RL, it is assumed that there is an overlap between the data generating policy and the optimal policy. Without this assumption, even with an infinite amount of data, no information about the optimal policy can be obtained. Consequently,  the best policy that can be learned from the dataset will always be suboptimal. Note that, our algorithm proceeds in iterations, and at iteration $t$, the learner deploys a policy $\pi_t$, collects a dataset $D_t$ from the new MDP $M(\pi_t)$ and learns an optimal policy in the new environment. In order to approximate the optimal policy in the new environment there must be an overlap between the deployed policy $\pi_t$ and $\pi_t^\star$ the policy that is optimal in the new MDP $M(\pi_t)$. This is precisely what assumption \ref{asn:bounded-coverage} states.
%Finally, we state assumption 4 for all policies $\pi$ and induced MDP $M_\pi$ as we cannot control the policy $\pi_t$ encountered by the learning algorithm at iteration $t$. If there is no performativity, then the assumption is required only for the data generating policy (say $\pi_0$) and then it is similar to the coverage assumption in offline RL [1, 2].

\begin{theorem}\label{thm:finite-sample-convergence}
Suppose assumptions~\ref{asn:lipschitzness}, \ref{asn:features}, and \ref{asn:bounded-coverage} hold and $\alpha = \frac{\sqrt{\frakm}}{\sqrt{A}(1-\gamma)}$,   $c_1 = \min\set{1, {\gamma^2} \cdot \sigma^\star_{\min}(\bm{\mu}_d \bm{\mu}_d^\top) }$\footnote{$\sigma^\star_{\min}(A)$ is the smallest positive eignevalue of a symmetric matrix $A$. }, and $\varepsilon_\mu < \frac{\sqrt{\kappa}}{24 \gamma \alpha^2}$. If \Cref{alg:repeated-optimization-finite} is run with  $\lambda > \frac{6\left(\varepsilon_\theta + \alpha \gamma \sqrt{D} \varepsilon_\mu \right)}{\sqrt{\kappa B}}$ and the number of samples $m_t = O\left( \frac{D^5 B^2 \lambda^4 }{(1-\gamma)^2 c_1^4\delta^4 \kappa}\log \frac{DB\lambda t}{c_1 \delta \kappa p}\right)$, then with probability at least $1-p$, 
    $$\textstyle 
    \lVert \widehat{d}_{t} - d_S\rVert_2 \le \delta \quad \forall t \ge \ln \left( \frac{2}{\delta(1-\gamma)}\right)/\ln(1/r),
    $$
    with $r$ as defined in \Cref{thm:convergence-rpo}.
    %with $r = \frac{5}{4} \sqrt{\frac{\varepsilon_\theta + \alpha \gamma \sqrt{D} \varepsilon_\mu   }{\lambda \sqrt{\kappa}} + \frac{4 \gamma \varepsilon_\mu \alpha^2}{\sqrt{\kappa}}} < 1$.
\end{theorem}

Suppose $B,\kappa,c_1 = O(1)$. Then \Cref{thm:finite-sample-convergence} shows that the required number of samples $m_t = \tilde{O}\left( \frac{D^5}{(1-\gamma)^2 \delta^4} \right)$ which is independent of the number of states and depends polynomially on the dimension of the features.  We can apply our algorithm to the tabular setting by instantiating $D=\abs{\calS}\abs{\calA}$. This gives a sample complexity of 
 $\tilde{O}\left( \frac{S^5A^5}{(1-\gamma)^2\delta^4}\right)$
 for obtaining a $\delta$-approximate stable solution. On the other hand, \citet{MTR23} has a sample complexity of 
 $\tilde{O}\left( \frac{S^6 A^2}{(1-\gamma)^4\delta^4}\right)$. Both approaches solve the empirical Lagrangian for the finite sample setting.  Below we provide a brief sketch of the proof.

We first show that the saddle point of the Lagrangian has bounded norm. In particular, if $(d^\star, \nu^\star, g^\star, \omega^\star) \leftarrow \max_{d,\nu} \min_{g,\omega} \calL_t(d,\nu; g, \omega)$ then $\norm{\Sigma_t^{-1} \nu^\star}_2 \le \sqrt{B}$, $\norm{g^\star}_2 + \norm{\omega^\star}_2 \le \frac{\lambda + 2D}{c_1}$, and $\norm{d^\star}_2 \le \frac{1}{1-\gamma}$. This allows us to consider an equivalent max-min optimization problem where the norms of the variables are bounded. 

The next step is to show that the empirical Lagrangian is close to the true Lagrangian so that a saddle point of the latter is also an approximate saddle point of the true Lagrangian. Since the variables $\nu$ and $\omega$ are $D$-dimensional and bounded norms, we can construct an $\varepsilon$-net and use union bound to show that the approximate Lagrangian $\widehat{\calL}_t(d,\nu,g,\omega)$ is close to the true Lagrangian $\calL_t(d,\nu,g,\omega)$ for any $\nu$ and $\omega$. However, the same argument does not apply to the variable $g$ as it is infinite-dimensional.  In fact, we first prove 
\begin{align*}
\abs{\widehat{\calL}_t(d,\nu, \tilde{g}, \omega) - \calL_t(d,\nu, \tilde{g},\omega)} \le \varepsilon
\end{align*}
holds for any $d, \nu,\omega, \textrm{ and } 
\tilde{g} \in \argmin_g \widehat{\calL}_t(d,\nu,g,\omega)$. This allows us to use $\varepsilon$-net construction only for the set of possible minimizers of $\widehat{\calL}_t(d,\nu, \cdot, \omega)$. 
Using this result, \Cref{lem:saddle-point-approximation} shows that a saddle point $(\widehat{d}, \widehat{\nu}, \widehat{g}, \widehat{\omega})$ of the empirical Lagrangian is an approximate saddle point of the true Lagrangian in the sense that $\calL(d^\star_t, \nu^\star_t; g^\star_t, \omega^\star_t) - \calL(\widehat{d}_t, \widehat{\nu}_t; g^\star_t, \omega^\star_t) \le 2 \varepsilon$. This allows us to show that the occupancy measures $d^\star_t$ and $\widehat{d}_t$ are close i.e. $\lVert d^\star_t - \widehat{d}_t\rVert_2 \le O\left(\sqrt{\frac{\varepsilon}{\lambda \kappa}} +  \frac{\varepsilon}{\sqrt{\kappa}}\right)$. The rest of the proof uses results of \Cref{thm:convergence-rpo} which shows that the sequence $d^\star_t$ converges to $d_S$ with the right value of  $\lambda$, and hence the sequence $\widehat{d}_t$ also converges to $d_S$.

\subsection{Solving the Empirical Lagrangian}

%\todo[inline]{It might be possible to simplify the algorithm and obtain faster rate of convergence.}

\Cref{alg:repeated-optimization-finite} finds an approximate stable policy in the finite samples setting, but it assumes that the empirical Lagrangian problem (as defined in \Cref{eq:empirical-Lagrangian}) is solvable. In this subsection, we construct an efficient algorithm to obtain a saddle point of the empirical Lagrangian. Our algorithm is motivated by the recent work of \cite{GNOP23} who designed a primal-dual algorithm for solving offline RL problem with linear MDPs. In our setting, the Lagrangian is strongly concave in $\nu$, but linear in other variables, and we update their method accordingly.

\Cref{alg:offline-primal-dual} describes a method to obtain an approximate saddle point of the empirical Lagrangian \Cref{eq:empirical-Lagrangian}. A standard strategy to solve such an optimization problem would be to perform alternate gradient descent in variables $g, \omega$ and ascent in variables $d, \nu$. However,  note that the variables $d$ and $g$ are infinite-dimensional, and their gradients cannot be explicitly computed. Hence we take gradient descent steps with respect to $\omega$ and ascent steps with respect to $\nu$, and represent $d$ and $g$ using $\nu$ and $\omega$. Furthermore, the objective is strongly concave in $\nu$ and we can replace the gradient ascent step in $\nu$ by a single update step. The algorithm runs for $T$ iterations and at each iteration, it takes  $K$ gradient steps for $\omega$, and one update step for $\nu$. 
The gradient with respect to $\omega$ is given as follows.
$$
\nabla_\omega \widehat{\calL}_t(d,\nu; g, \omega) = \Phi^\top d - \frac{1}{m_t} \sum_{j=1}^{m_t} \phi(s_j,a_j) \phi(s_j,a_j)^\top \Sigma_t^{-1} \nu 
$$
The gradient with respect to $\nu$ is given as follows.
\begin{align*}
\nabla_\nu \widehat{\calL}_t(d,\nu; g, \omega) = \Sigma_t^{-1}\cdot \frac{1}{m_t} \sum_{j=1}^{m_t} \phi(s_j,a_j) \left( r_t(s_j,a_j)  + \gamma \cdot g(s'_j)   - \phi(s_j,a_j)^\top \omega \right) - \lambda \cdot \nu
\end{align*}

Setting $\nabla_\nu \widehat{\calL}_t(d,\nu; g, \omega) = 0$, we get the update step for the variable $\nu$ (line 5, \Cref{alg:offline-primal-dual}).
Note that, computing the gradient requires the knowledge of $d$ and $g$. We maintain $d$ and $g$ as a function of the parameters $\nu$ and $\omega$. In particular, at iteration $t$, define policy $\pi_t$ as 
$$\pi_t(a \mid s) = \sigma\left( \eta_\pi \cdot \sum_{j< t} \phi(s,a)^\top \omega_j\right)$$
where $\sigma(\cdot)$ is the sigmoid function, and  $\eta_\pi$ is a constant. Then we choose $g_t$ as $g^{\pi_t, \omega_t}$ with 
$$g^{\pi_t, \omega_t}(s,a) = \sum_{a} \pi_t(a\mid s) \phi(s,a)^\top \omega_t.$$ 
And, we choose $d_t$ as $d^{\pi_t,\nu_t}$ with $$d^{\pi_t,\nu_t}(s,a) = \pi_t(a \mid s) \cdot \left(\rho(s) + \gamma \cdot \bm{\mu}(s)^\top \nu_t \right).$$ 
Since the measure $\bm{\mu}$ is unknown we cannot use this representation of the occupancy measure, and instead build an estimator of $d^{\pi_t, \nu_t}$. Given a tuple $(s^0, s, a, r, s')$ let
$$
\widehat{d}^{\pi_t,\nu_t}(\tilde{s},\tilde{a}) = \pi_t(\tilde{a} \mid \tilde{s}) \cdot \left( \one\set{\tilde{s} = s^0} + \nu_t^\top \Sigma_t^{-1} \phi(s,a) \one\set{\tilde{s} = s'} \right).
$$
It can be easily checked that $\E\left[ \widehat{d}^{\pi_t,\nu_t}({s},{a})\right] = {d}^{\pi_t,\nu_t}({s},{a})$. The next theorem proves that the policy returned by \Cref{alg:offline-primal-dual} is approximately optimal with respect to the regularized objective.

\begin{algorithm}[!t]
\caption{Offline Regularized Primal-Dual\label{alg:offline-primal-dual}}
\KwIn{(a) Dataset $\calD$, and and (b) Number of iterations $T_{\inner}$ and $K$.}
Set $\calW = \{\omega : \norm{\omega}_2 \le \frac{2D}{1-\gamma}\}$,  $\calV = \{\nu: \norm{\nu}_2 \le D \sqrt{B} \}$, $\eta_\omega = \frac{D\sqrt{B}}{\sqrt{K (B + (1-\gamma)^{-2})}}$,  $\eta_\pi = \sqrt{\frac{\log A}{T}} \frac{1-\gamma}{D}$.\\ 
Initialize $\omega_0 \in \calW, \nu_0 \in \calV$ and $\pi_0$.\\
\For{$\ell=1,\ldots,T_{\inner}$}{
\tcc{Take a gradient step for $\omega_\ell$}
Initialize $\omega_{\ell,0} = \omega_{\ell-1}$.\\
\For{$k=0,\ldots,K-1$}{
Obtain sample $(s^0_{\ell,k}, s_{\ell,k}, a_{\ell,k}, r_{\ell,k}, s'_{\ell,k})$ uniformly at random from $\calD$.\\
$d_{\ell,k}(s,a) = \pi_\ell(a|s)\cdot(  \nu_{\ell-1}^\top \Sigma^{-1} \phi(s_{\ell,k},a_{\ell,k}) \one \set{s = s'_{\ell,k}} + \one\{s=s^0_{\ell,k}\})$.\\
$\widetilde{g}_{\omega_{\ell,k}} = \Phi^\top d_{\ell-1,k} - \phi(s_{\ell,k},a_{\ell,k})\phi(s_{\ell,k},a_{\ell,k})^\top \Sigma^{-1} \nu_{\ell-1}$.\\
$\omega_{\ell,k+1} \leftarrow \textrm{Proj}_{\mathcal{W}}(\omega_{\ell,k} - \eta_\omega \cdot \widetilde{g}_{w_{\ell,k}})$.
}
Set $\omega_\ell = \frac{1}{K} \sum_{k=1}^K \omega_{\ell,k}$.\\
\tcc{Update variable $\nu_\ell$ and $\pi_\ell$}
\vspace{-2em}
\begin{align*}
\nu_\ell \leftarrow \frac{1}{\lambda} \Sigma^{-1}\cdot \frac{1}{\abs{\calD}} \sum_{j = 1}^{\abs{\calD}} \phi(s_j,a_j) &\left( r_j   + \gamma \cdot g^{\pi_{\ell-1}, \omega_{\ell-1} }(s'_j)   - \phi(s_j,a_j)^\top \omega_{\ell - 1} \right).
\end{align*}
%\STATE Get a batch of samples $\{s_{h,j}, a_{h,j}, s'_{h,j}, r_{h,j}\}_{j=1}^K \sim \mu^h_\tref$.\\
%\tcc{Take a gradient step for $\nu_t$}
%Obtain sample $(s^0_{t}, s_{t}, a_{t}, r_{t}, s'_{t})$.\\
%Set $g_t({s}) = \sum_{a} \pi_t(a \mid s) \cdot \phi(s,a)^\top \omega_t$.\\
%Set $\widetilde{g}_{\nu_t} =  \Sigma^{-1}\cdot \frac{1}{\abs{\calD}} \sum_{j = 1}^{\abs{\calD}} \phi(s_j,a_j) \left( r_j   + \gamma \cdot g_t(s'_j) - \phi(s_j,a_j)^\top \omega_t \right) - \lambda \cdot \nu_t$.\\
%$\nu_{t+1} \leftarrow \textrm{Proj}_{\mathcal{V} }(\nu_t + \frac{1}{\lambda t} \cdot \widetilde{g}_{\nu_t})$.\\
%\tcc{Update policy}
$\pi_{\ell}(a \mid s) = \sigma\left(\eta_\pi \cdot \sum_{j=1}^\ell \phi(s,a)^\top  \omega_{\ell - 1} \right)$.
}
Return $\pi_j $ where $j \sim \textrm{uniform}([T])$.
\end{algorithm}

\begin{theorem}\label{thm:convergence-offline-primal-dual}
    Suppose \Cref{alg:offline-primal-dual} is run with $K \ge \frac{144 D^2 B}{\varepsilon^2}\left(B + (1-\gamma)^{-2}\right)$ and $T \ge \frac{576D^2}{\varepsilon^2} \cdot \frac{\log A}{(1-\gamma)^2}$. Let $\tilde{\nu}$ be the average feature  of a policy selected uniformly at random from $\set{\pi_t}_{t\in [T]}$. Then for a policy $\pi^\star$ with  $\nu^\star = \Phi^\top d^{\pi^\star}$ we have, 
    $$\textstyle 
 \left \langle \tilde{\nu}, \theta_t \right \rangle- \frac{\lambda}{2} \norm{\tilde{\nu}}_2^2 \ge \nu^{\star^\top} \theta_t - \frac{\lambda}{2} \norm{\nu^\star}_2^2 - \varepsilon.
$$
\end{theorem}

The proof follows an argument similar to the proof of theorem 5.3 from \cite{GNOP23}. We upper bound the sub-optimality gap in the regularized objective by a dynamic regret, and then \Cref{lem:regret-upper-bound} in the appendix, first expresses the regret in terms of regret in $\omega,\nu$, and $\pi$ and provides an upper bound on each term separately. As an immediate corollary of \Cref{thm:convergence-offline-primal-dual} we can show that the policy returned by \Cref{alg:offline-primal-dual} is an approximate saddle point of the empirical Lagrangian, which is required by \Cref{alg:repeated-optimization-finite}.

\begin{corollary}\label{cor:apx-empirical-saddle-point}
    Under the same setting as in \Cref{thm:convergence-offline-primal-dual}, let the number of samples $m_t \ge  O\left( \frac{D^5 B \lambda^4 }{(1-\gamma)^2 c_1^4\varepsilon^2}\log \frac{DB\lambda }{c_1 \varepsilon p}\right)$ where $c_1 = \min \set{1, \gamma^2 \cdot \sigma^\star_{\min}(\bm{\mu}_d \bm{\mu}_d^\top)}$. Then there exists $\tilde{g}, \tilde{\omega}$ so that with probability at least $1-p$,
    $$\max_{d,\nu} \widehat{\calL}_t({d}, {\nu}; \tilde{g}, \tilde{\omega} ) - 2\varepsilon \le \widehat{\calL}_t(\tilde{d}, \tilde{\nu}; \tilde{g}, \tilde{\omega} ) \le \min_{g,\omega} \widehat{\calL}_t(\tilde{d}, \tilde{\nu}; {g}, {\omega} ) + 2 \varepsilon.$$
\end{corollary}

The next result shows that we can use algorithm 3 as an approximate oracle to solve the saddle point optimization problem (\Cref{eq:saddle-point-optimization}) in \Cref{alg:repeated-optimization-finite} and obtain convergence to performatively stable solution by choosing an appropriate value of the regularization parameter $\lambda$.

\begin{corollary}\label{cor:apx-oracle-convergence}
    Under the same setting as in \Cref{thm:finite-sample-convergence}, suppose $\lambda > \frac{18(\varepsilon_\theta + \alpha \gamma \sqrt{D} \varepsilon_\mu  }{\sqrt{\kappa B}}$. If \Cref{alg:repeated-optimization-finite} uses \Cref{alg:offline-primal-dual}  every iteration to approximately solve the saddle point optimization~\Cref{eq:saddle-point-optimization}, then with probability at least $1-p$,
    $$\textstyle 
    \lVert \tilde{d}_{t} - d_S\rVert_2 \le \delta \quad \forall t \ge \ln \left( \frac{2}{\delta(1-\gamma)}\right)/\ln(1/r),
    $$
    with $r$ as defined in \Cref{thm:convergence-rpo}.
\end{corollary}

\section{Applications}\label{sec:applications}
We provide several applications of  performative RL involving stochastic games.

\textbf{Stochastic Stackelberg Game}: Consider two RL agents ($1$ and $2$) interacting in an MDP with a shared state-space $S$. The set of actions available to agent $1$ (resp. $2$) is $A_1$ (resp. $A_2$). The transitions and rewards are determined by the actions of both the players, i.e., there exist reward functions $r_1, r_2: S\times A_1 \times A_2 \rightarrow \R$, and transition $P: S \times A_1 \times A_2 \rightarrow \Delta(S)$.%Additionally, we assume that the MDP is linear i.e. there exist features $\phi: S \times A_1 \times A_2 \rightarrow \R^D$ and measures $\psi: S \rightarrow \R^D$ so that
%$$
%r_1(s,a_1, a_2) = \theta_1^\top \phi(s,a_1,a_2),\ r_2(s,a_1, a_2) = \theta_2^\top \phi(s,a_1,a_2),\ P(s' \mid s, a_1, a_2) = \psi(s')^\top \phi(s,a_1,a_2)
%$$

We adopt the framework of Stackelberg game~\citep{von2010market} where agent $1$ is the \emph{leader} and agent $2$ is the \emph{follower}. The first agent deploys a stationary policy (say $\pi_1$) to maximize her reward, and in response, the second agent adopts a policy $\pi_2$ that is obtained through Boltzmann softmax operator with temperature parameter $\beta$. To be precise, given policy $\pi_1$ of agent $1$, let the modified MDP be $(S,A_2,\bar{r}_2, \bar{P}_2,\gamma) $ where $\bar{r}_2(s,a_2) = \sum_{a_1} \pi_1(a \mid s) r_2(s,a_1,a_2)$ and $\bar{P}_2(s' \mid s, a_2) = \sum_{a_1} \pi_1(a \mid s) P(s' \mid s, a_1, a_2)$. Now let $Q_2^\star: S \times A_2 \rightarrow \R$ be the optimal state,action Q-function in this new MDP. Then the policy adopted by the second agent is $\pi_2(a \mid s) = \frac{\exp(\beta Q_2^\star(s,a))}{\sum_b \exp(\beta Q_2^\star(s,b))}$. 

Therefore, for a policy $\pi_1$ by the agent $1$, the second agent updates her policy, and from the perspective of the first agent the underlying MDP changes to $M(\pi_1)$. The next lemma bounds the sensitivity of the reward, and probability transition functions.

\begin{lemma}\label{lem:two-agent-mdp}
    Suppose $\norm{\pi_1(\cdot \mid s) - \tilde{\pi}_1(\cdot \mid s)}_1 \le \delta$, and $R = \max_{i}\max_{s,a} r_i(s,a)$. Then, $\forall s,a,s'$
    \begin{align*} 
    &\abs{{r}_{\pi_1}(s,a) - {r}_{\tilde{\pi}_1}(s,a)} \le \delta \cdot \frac{2\sqrt{2} \beta A_1 A_2^{3/2}  R^2}{(1-\gamma)^2} \ \textrm{ and }\\
    &\abs{P_{\pi_1}(s' | s,a) - {P}_{\tilde{\pi}_1}(s' | s,a)} \le \delta \cdot \frac{2\sqrt{2}A_1 A_2^{3/2} \beta   R}{(1-\gamma)^2}.
    \end{align*}
\end{lemma}
If agent $1$ assumes that the MDP is linear, then \Cref{lem:two-agent-mdp} implies that the sensitivity parameters are $\varepsilon_\theta = \frac{2\sqrt{2} \beta A_1 A_2^{3/2}  r^2_{\max}}{(1-\gamma)^2} $ and $\varepsilon_\mu = \frac{2\sqrt{2} \beta A_1 A_2^{3/2}  r_{\max}}{(1-\gamma)^2}$. Therefore, according to \Cref{thm:convergence-rpo}, if agent $1$ repeatedly optimizes a regularized objective with appropriate $\lambda$, they would converge to a performatively stable policy.

\textbf{Multiple Followers}: We now generalize the previous two-player stochastic Stackelberg game to a setting where there are $m$ followers. Now the reward of the $j$-th agent is $r_j: S \times \prod_{i=1}^{m+1} A_i \rightarrow \R$ and the probability transition function is given as $P: S \times \prod_{i=1}^{m+1} A_i \rightarrow \Delta(S)$. As before, agent $1$ is the \emph{leader} and adopts a stationary policy (say $\pi_1$). This induces a stochastic game $\overline{\calM} = (S, \{A_i\}_{i=2}^{m+1}, \overline{P}, \{\bar{r}_i\}_{i=2}^{m+1}, \gamma)$ among the $m$ follower agents where $\bar{r}_i(s, \bm{a}) = \sum_{a_1} \pi_1(a_1 \mid s) r_i(s, a_1, \bm{a})$ and $\overline{P}(s' \mid s, \bm{a}) = \sum_{a_1} \pi_1(a_1 \mid s) P(s' \mid s, a_1, \bm{a})$. 

We assume that the $m$ follower agents play a policy that is an approximation of the optimal coarse correlated equilibrium (CCE) obtained through the Boltzmann softmax operator with temperature parameter $\beta$. In particular, let $\pi^\star_f : S \rightarrow \Delta\left( \prod_{i=2}^{m+1} A_i\right)$ be the CCE in the game $\overline{\calM}$ that maximizes the welfare, i.e., the sum of expected rewards of the $m$ followers. Additionally, let $Q^\star_f(s, \bm{a}) = \sum_{i=2}^{m+1} \E_{\pi^\star_f} \left[ \sum_{t=0}^\infty \gamma^t \bar{r}_i(s_t, \bm{a}_t) \mid s, \bm{a}\right]$ be the  state-action welfare function under the policy $\pi^\star_f$. Then the $m$ followers adopt a policy $\pi_f$ that is given as $\pi_f(\bm{a} \mid s) = \frac{\exp(\beta Q_f^\star(s,\bm{a}))}{\sum_{\bm{b}} \exp(\beta Q_f^\star(s,\bm{b}))}$. Then the following bound holds.

\begin{lemma}\label{lem:multi-agent-mdp}
    Let $\norm{\pi_1(\cdot | s) - \tilde{\pi}_1(\cdot | s)}_1 \le \delta$, $A = \max_i A_i$, and $R = \max_{i, s,\bm{a}} r_i(s,\bm{a})$. Then, $\forall s,a $
    \begin{align*}
    &\abs{{r}_{\pi_1}(s,a) - {r}_{\tilde{\pi}_1}(s,a)} \le \delta \cdot \frac{3 \sqrt{2} \beta m A^{3m/2+1}  R^2}{(1-\gamma)^2} \ \textrm{ and }\\
    &\abs{P_{\pi_1}(s' | s,a) - {P}_{\tilde{\pi}_1}(s' | s,a)} \le \delta \cdot \frac{3 \sqrt{2} \beta m A^{3m/2+1}  R}{(1-\gamma)^2}
    \end{align*}
\end{lemma}
Now, the sensitivity parameters grow exponentially with the number of followers. This is unavoidable for a general stochastic game, but  can be significantly reduced for  succinct games (e.g.,~\citep{kearns2013graphical}). 

%In particular, we adopt the framework of mean-field MARL~\citep{yang2018mean, chen2021pessimism}. The reward and the transition of the agent not only depends on the current state, and action, but also on the empirical distribution of the states of the other agents. In particular, let $\calM(S)$ be the set of all empirical distributions over the state space $S$. Then there exists reward function $r : S \times \calM(S) \times A \rightarrow \R$ and probability transition function $P: S\times \calM(S) \times A \rightarrow \Delta(S)$ so that for any agent $i$ takes action $a$ at state $s$, then the reward is $r\left(s, \frac{1}{m} \sum_{j \neq i} \mathbb{1}_{s_j}, a\right)$ and transition is $P\left(s' \mid s, \frac{1}{m} \sum_{j \neq i} \mathbb{1}_{s_j}, a\right)$.

%We again adopt the framework of Stackelberg game where agent $1$ is leader and adopts a policy $\pi_1: S \times \calM(S) \rightarrow \Delta(A)$. In response, the group of $m$ followers adopt a policy $\pi_2$. Note that, following the literature in mean-field MARL~\cite{chen2021pessimism} we assume that the group of $m$ followers adopt a policy that depends on the current state, and aggregate statistic of the other agents, and not on the identity of an agent.

\section{Conclusion}
\label{sec:conclusion}
In this work, we provide computationally and statistically efficient algorithms for performative RL with large-scale MDPs. Our work is centered around the linear MDP model, and it would be interesting to generalize our approach to nonlinear function approximation (e.g., ~\citep{wang2020reinforcement}). However, there are two challenges with  general function approximation. First, a performatively stable policy might not exist since the value function is non-convex in both policy and occupancy measures. Therefore, we might have to resort to a locally stable solution~\citep{LW24}. Second, obtaining \emph{last-iterate} convergence in non-convex performative prediction is difficult, and to the best of our knowledge, the existing result provides \emph{best-iterate} convergence. In this work, we have provided \emph{last-iterate} convergence for linear MDPs, and obtaining similar convergence for general MDPs is significantly harder.

It could also be exciting to consider the effects of multiple learners in performative RL by imposing a specific structure on the underlying game so that independent repeated optimization still converges~\citep{piliouras2023multi}. Finally, we briefly study the approximation of performatively optimal policy, and exploring the design of such policies is another avenue of future research.
\subsection*{Acknowledgements}
The work of Goran Radanovic was funded by the Deutsche Forschungsgemeinschaft (DFG, German Research Foundation) – project number 467367360.

\bibliography{icrl_refs}

\clearpage
\onecolumn
\appendix
\addcontentsline{toc}{section}{Appendix}
\part{Appendix}

\parttoc
\section{Proof of  Theorem \texorpdfstring{\ref{thm:convergence-rpo} }{}  }
\label{sec:convergence-rpo-proof}
\begin{proof}
The Lagrangian of the optimization problem \cref{eq:regularized-rl-matrix-dual} is given as follows.
\begin{align*}
    \calL_t(d,h,g ) = d^\top \Phi \theta_t - \frac{\lambda}{2} d^\top \Phi \Phi^\top d + h^\top \left( Bd  - \rho - \gamma \cdot \bm{\mu}_t \Phi^\top d\right) + d^\top g 
\end{align*}
Let $\mathcal{F}_t$ be the function defined as $\mathcal{F}_t(h,g) = \max_d \calL_t(d,h,g)$. Then the dual of the optimization problem defined in \cref{eq:regularized-rl-matrix-dual} is given as follows.
$$
\min_{h,g \ge 0} \mathcal{F}_t(h,g)
$$
%A saddle point of the Lagrangian corresponds to an optimal solution i.e.
%\begin{align*}
%    (d_{t}, h_{t}) \in \argmax_{d } \argmin_{h, g \ge 0} \calL_{t}(d,h,g)
%\end{align*}
For a fixed $h$ and $g$, the gradient of $\calL_{t}(d,h)$ with respect to $d$ is given as follows.
\begin{align*}
    \nabla_d \calL_t(d,h) = \Phi \theta_{t} - \lambda \Phi \Phi^\top d + B^\top h - \gamma \Phi  \bm{\mu}_{t}^\top h + g
\end{align*}
At an optimal solution $d^\star$, $\nabla_d \calL_{t}(d^\star, h, g) = 0$, and we have
\begin{align*}
g &= -\Phi \theta_{t} + \lambda \Phi \Phi^\top d^\star - B^\top h + \gamma \Phi  \bm{\mu}_{t}^\top h\\
\Rightarrow g^\top d^\star &= -(d^\star)^\top \Phi \theta_{t} + \lambda (d^\star)^\top \Phi \Phi^\top d^\star - (d^\star)^\top B^\top h + \gamma (d^\star)^\top \Phi  \bm{\mu}_{t}^\top h
\end{align*}
Substituting the above expression of $g^\top d^\star$ we get the following form of the Lagrangian.
\begin{align}\label{eq:temp-Lagrangian}
\mathcal{F}_t(h,g) = \calL_t(d^\star,h,g)  = \frac{\lambda}{2}(d^\star)^\top \Phi \Phi^\top d^\star - h^\top \rho
\end{align}

Moreover, from the equation $\nabla_d \calL_t(d^\star,h,g) = 0$ we also have the following expression of $\Phi^\top d^\star$.
\begin{align}
    \Phi \Phi^\top d^\star &= \frac{1}{\lambda} \left(\Phi \theta_t + B^\top h - \gamma \Phi \bm{\mu}_t^\top h + g \right)\nonumber \\
    \Rightarrow \Phi^\top d^\star &= \frac{1}{\lambda}\left( \theta_t + (\Phi^\dagger B^\top - \gamma \bm{\mu}_t^\top ) h + \Phi^\dagger g\right) \label{eq:dual-to-primal-solution}
\end{align}
Here we use the fact that the matrix $\Phi $ has full row rank and its left pseudoinverse is $\Phi^\dagger = (\Phi^\top \Phi)^{-1}\Phi^\top$. Finally substituting the above expression of $\Phi^\top d^\star$ in \cref{eq:temp-Lagrangian} we get the following expression for the dual problem.
\begin{align}
    \label{eq:dual-regularized}
    \min_{h, g\ge 0} \mathcal{F}_t(h,g) = \frac{1}{2\lambda}\norm{(\Phi^\dagger B^\top - \gamma \bm{\mu}_t^\top ) h + \Phi^\dagger g +  \theta_t}_2^2  - h^\top \rho 
\end{align}
Let $(h_{t+1}, g_{t+1})$ be the optimal dual solutions corresponding to the occupancy measure $d_{t+1}$ i.e. $(h_{t+1}, g_{t+1}) \in \argmin_{h, g \ge 0} \calF_{t} (h,g)$. Additionally, let $(h_S, g_S)$ be the optimal dual solutions corresponding to the stable occupancy measure $d_S$ i.e. $(h_S, g_S) \in \argmin_{h,g\ge 0} \calF_S(h,g)$, where the objective $\calF_S(h,g)$ is given as follows.
\begin{align*}
    \calF_S(h,g) = \frac{1}{2\lambda}\norm{(\Phi^\dagger B^\top - \gamma \bm{\mu}_S^\top ) h + \Phi^\dagger g +  \theta_S}_2^2  - h^\top \rho 
\end{align*}
Let $u_t = (\Phi^\dagger B^\top - \gamma \bm{\mu}_t^\top ) h_t + \Phi^\dagger g_t$ and $u_S = (\Phi^\dagger B^\top - \gamma \bm{\mu}_S^\top ) h_S + \Phi^\dagger g_S$. Then using \cref{eq:dual-to-primal-solution} we get the following bound.
\begin{align}
    \sqrt{\kappa}\norm{d_{t+1} - d_S}_2 \le \norm{\Phi^\top d_{t+1} - \Phi^\top d_S}_2 &\le \frac{1}{\lambda}\left( \norm{\theta_{t+1} - \theta_S}_2 + \norm{u_{t+1} - u_S}_2 \right)\nonumber\\
    &\le \frac{1}{\lambda}\left( \varepsilon_\theta \norm{d_{t} - d_S}_2 + \norm{u_{t+1} - u_S}_2 \right) \label{eq:intermediate-bound-diff}
\end{align}
The first inequality uses assumption \ref{asn:features} and $d_{t+1} \neq d_S$. If $d_{t+1} = d_S$ we are already done. Now we need to bound the term $\norm{u_{t+1} - u_S}_2$. Note that the dual optimization problem in variable $h$ is unconstrained. Therefore, given $g_{t+1}$ we must have $\nabla_h \calF_{t}(h_{t+1}, g_{t+1}) = 0$. We will denote $M_{t} = \Phi^\dagger B^\top - \gamma \bm{\mu}_{t}^\top$. Moreover, the proof of \cref{lem:bound-M-t-pseudoinv} shows that the matrix $M_{t}$ has full column rank, and hence its right pseudoinverse exists and is given as $M_{t}^\dagger = M_{t}^\top (M_{t} M_{t}^\top)^{-1}$. 

\begin{align*}
&\frac{1}{\lambda} M_{t}^\top M_{t} h_{t+1} + \frac{1}{\lambda} M_{t}^\top (\Phi^\dagger g_{t+1}  + \theta_t) - \rho = 0\\
\Rightarrow& \quad M_{t} h_{t+1} = - \Phi^\dagger g_{t+1} - \theta_t + {\lambda} \left(M_{t}^\dagger\right)^\top \rho = -u_{t+1} + M_{t+1} h_{t+1} - \theta_t + {\lambda} \left(M_{t}^\dagger\right)^\top \rho
\end{align*}
Rearranging we get the following expression for $u_{t+1}$.
$$
u_{t+1} = (M_{t+1} - M_t) h_{t+1} - \theta_t + {\lambda} \left(M_{t}^\dagger\right)^\top \rho
$$
Similarly, we can establish the following relation.
$$
u_S =  -  \theta_S + {\lambda} \left(M_{S}^\dagger\right)^\top \rho
$$
Using the expressions of $u_{t+1}$ and $u_S$ we get the following upper bound. 
\begin{align*}
    \norm{u_{t+1} - u_S}_2 &\le \norm{\theta_{t} - \theta_S}_2 + \lambda \norm{\left((M_{t}^\dagger)^\top - (M_{S}^\dagger)^\top \right) \rho }_2 + \norm{M_{t+1} - M_t}_2 \norm{h_{t+1}}_2\\
    &\le  \varepsilon_\theta \norm{d_{t-1} - d_S}_2 + \lambda \norm{(M_{t}^\dagger)^\top - (M_{S}^\dagger)^\top}_2 + \gamma \norm{\bm{\mu}_t - \bm{\mu}_{t+1}}_2 \norm{h_{t+1}}_2 \quad \textrm{[By assumption \ref{asn:lipschitzness}]}\\
    &\le  \varepsilon_\theta \norm{d_{t-1} - d_S}_2 + \lambda \norm{(M_{t}^\dagger)^\top - (M_{S}^\dagger)^\top}_2  + \gamma \norm{\bm{\mu}_t - \bm{\mu}_{S}}_2 \norm{h_{t+1}}_2 + \gamma \norm{\bm{\mu}_{t+1} - \bm{\mu}_{S}}_2 \norm{h_{t+1}}_2 \\
    & \le (\varepsilon_\theta + \gamma \varepsilon_\mu \norm{h_{t+1}}_2) \norm{d_{t-1} - d_S}_2 + \gamma \varepsilon_\mu \norm{h_{t+1}}_2 \norm{d_t - d_S}_2  \\
    &+3\lambda \norm{M_{t}^\top -M_{S}^\top}_2 \norm{(M_{t}^\dagger)^\top}_2 \norm{(M_{S}^\dagger)^\top}_2 \quad \textrm{[By matrix perturbation bound\footnotemark  and assumption \ref{asn:lipschitzness}]} \\
    &\le (\varepsilon_\theta + \gamma \varepsilon_\mu \norm{h_{t+1}}_2) \norm{d_{t-1} - d_S}_2 + \gamma \varepsilon_\mu \norm{h_{t+1}}_2 \norm{d_t - d_S}_2  \\&+ 3\lambda \gamma \norm{\bm{\mu}_{t} - \bm{\mu}_S}_2 \norm{(M_{t} M_{t}^\top)^{-1} M_{t}}_2 \norm{(M_S M_S^\top)^{-1} M_S}_2\\
    &\le (\varepsilon_\theta + \gamma \varepsilon_\mu \norm{h_{t+1}}_2) \norm{d_{t-1} - d_S}_2 + \gamma \varepsilon_\mu \norm{h_{t+1}}_2 \norm{d_t - d_S}_2  \\&+ 3\lambda \gamma \varepsilon_\mu \norm{d_{t-1} - d_S}_2 \norm{(M_{t} M_{t}^\top)^{-1} M_{t}}_2 \norm{(M_S M_S^\top)^{-1} M_S}_2
    %&\le \varepsilon_\theta \norm{d_t - d_S}_2 + \frac{3 \lambda \gamma \mu_2^4 }{\mu_1^2 A^4} \varepsilon_\mu  \norm{d_t - d_S}_2 \norm{M_{t+1}}_2 \norm{M_S}_2 \ \textrm{[By assumption~\eqref{asn:lipschitzness} and lemma~\eqref{lem:bound-smallest-eigenvalue}]}\\
    %&\le \left(\varepsilon_\theta + \varepsilon_\mu \frac{3 \lambda \gamma \mu_2^4 }{\mu_1^2 A^4} \left(\frac{\sqrt{\mu_2}  }{\mu_1 } + \gamma \sqrt{D} \right)^2 \right) \norm{d_t - d_S}_2 \quad \textrm{[By lemma~\eqref{lem:bound-M-t}]}
\end{align*}
\footnotetext{For example, see Theorem 4.1 of \cite{Wedin73}.}

Using \Cref{lem:bound-M-t-pseudoinv} we get the following bound, $\norm{(M_{t} M_{t}^\top)^{-1} M_{t}}_2 \le \frac{\sqrt{\frakm}}{\sqrt{A}(1-\gamma)}$. Similarly, we can bound $\norm{(M_{S} M_{S}^\top)^{-1} M_{S}}_2 \le \frac{\sqrt{\frakm}}{\sqrt{A}(1-\gamma)}$. Finally, observe that $h_{t+1}$ is an optimal dual solution, and we can use \Cref{lem:bound-dual-solution} to bound norm of the solution $h_{t+1}$. Let $\alpha = \frac{\sqrt{\frakm}}{\sqrt{A}(1-\gamma)}$. Then we have,
$$
\norm{u_{t+1} - u_S}_2 \le \left( \varepsilon_\theta + 3 \lambda \gamma \varepsilon_\mu \alpha^2 + \alpha (\lambda \alpha + \sqrt{D}) \gamma \varepsilon_\mu \right) \norm{d_{t-1} - d_S}_2 + \gamma \varepsilon_\mu \alpha (\lambda \alpha + \sqrt{D}) \norm{d_t - d_S}_2
$$
Substituting the above bound in \cref{eq:intermediate-bound-diff} we get the following recurrence relation.
\begin{align}
    \norm{d_{t+1} - d_S}_2 &\le \underbrace{\frac{1}{\lambda \sqrt{\kappa}} \left( \varepsilon_\theta + 3 \lambda \gamma \varepsilon_\mu \alpha^2 + \alpha (\lambda \alpha + \sqrt{D}) \gamma \varepsilon_\mu \right)}_{:=\beta} \norm{d_{t-1} - d_S}_2\nonumber \\
    &+ \underbrace{\frac{1}{\lambda \sqrt{\kappa}} \left(\varepsilon_\theta +  \gamma \varepsilon_\mu \alpha (\lambda \alpha + \sqrt{D})\right)}_{:=\beta_1} \norm{d_t - d_S}_2\label{eq:defn-beta}
\end{align}
Notice that $\beta \ge \beta_1$, which gives us the following recurrence relation $\norm{d_{t+1} - d_S}_2 \le \beta \left(\norm{d_t - d_S}_2 + \norm{d_{t-1} - d_S}_2 \right)$. We claim that $\norm{d_{t+1} - d_S}_2 \le \frac{2}{1-\gamma} r^t$ for $r = \frac{\beta}{2} \left(1 + \frac{1}{2}\sqrt{1 + \frac{4}{\beta}} \right)$. The proof of this claim is through induction. Indeed, $\norm{d_0 - d_S}_2 \le \frac{1}{1-\gamma}$ as $\sum_{s,a} d(s,a) = \frac{1}{1-\gamma}$ for any occupancy measure $d$. Furthermore, assuming the claim holds for any index less than or equal to $t$, we get the following bound.
$$
\norm{d_{t+1} - d_S}_2 \le \frac{2\beta}{1-\gamma} \left(r^{t-1} + r^{t-2} \right) = \frac{2 r^{t-2}}{1-\gamma} \left(\beta r + \beta \right)
$$
It can be checked that $r = \frac{\beta}{2} \left(1 + \frac{1}{2}\sqrt{1 + \frac{4}{\beta}} \right)$ is one of the roots of $x^2 - \beta x - \beta = 0$ and $\beta r + \beta = r^2$. This proves that $\norm{d_{t+1} - d_S}_2 \le \frac{2}{1-\gamma}r^t$. If $r < 1$, then as long as $t \ge \ln \left( \frac{2}{\delta(1-\gamma)}\right)/\ln(1/r)$ we are guaranteed that $\norm{d_t - d_S}_2 \le \delta$.

Now we determine the condition that guarantees $r < 1$. Moreover, we can express the parameter $\beta$ as follows. Note that $r < \frac{\beta}{2}\left(1 + \frac{1}{2} + \frac{1}{\sqrt{\beta}}\right) < \frac{3}{4}\beta + \frac{\sqrt{\beta}  }{2}$. If we ensure that $\beta < \frac{16}{25}$ then we get $r < \left( \frac{3}{4} + \frac{1}{2}\right) \sqrt{\beta} < \frac{5}{4}\sqrt{\beta} < 1$. From the definition of $\beta$ in \cref{eq:defn-beta}, we can express $\beta$ as follows.
$$
\beta = \frac{\varepsilon_\theta + \alpha \gamma \sqrt{D} \varepsilon_\mu   }{\lambda \sqrt{\kappa}} + \frac{4 \gamma \varepsilon_\mu \alpha^2}{\sqrt{\kappa}}
$$
If $\lambda > \frac{25 (\varepsilon_\theta + \alpha \gamma \sqrt{D} \varepsilon_\mu )}{8\sqrt{\kappa}}$ then the first term in $\beta$ is less than $8/25$. In order for $\beta < 16/25$ we need $\varepsilon_\mu < \frac{2 \sqrt{\kappa}}{25 \gamma \alpha^2}$. 
\end{proof}

\begin{lemma}\label{lem:bound-M-t-pseudoinv}
    Suppose assumption~\ref{asn:features} holds, then $\norm{M_t^\dagger}_2 \le \frac{\sqrt{\frakm}}{\sqrt{A}(1-\gamma)}$.
\end{lemma}
\begin{proof}
    Since $M_t = \Phi^\dagger B^\top - \gamma \bm{\mu}_t^\top$ and $\Phi^\top \Phi$ is invertible (by assumption~\ref{asn:features}), we obtain the following expression for $M_t$.
    \begin{align*}
        M_t &= \Phi^\dagger B^\top - \gamma \bm{\mu}_t^\top\\
        &= (\Phi^\top \Phi)^{-1} \left[ \Phi^\top B^\top - \gamma \Phi^\top \Phi \bm{\mu}_t^\top \right]\\
        &= (\Phi^\top \Phi)^{-1} \Phi^\top [B^\top - \gamma P^\top]\\
        &= \Phi^\dagger [B^\top - \gamma P_t^\top]
    \end{align*}
    where $P_t = \bm{\mu}_t \Phi^\top$ is the probability transition matrix. The singular values of $\Phi^\dagger$ are obtained by inverting all the non-zero singular values of $\Phi$ and leaving zero singular values as they are. By assumption \ref{asn:features} $\Phi$ has rank $d$ and hence only non-zero singular values. Moreover, since $\lambda_{\max} (\Phi^\top \Phi) \le \frakm$, we have $\sigma_{\max}(\Phi) \le \sqrt{\frakm}$ and $\sigma_{\min}(\Phi^\dagger) \ge \frac{1}{\sqrt{\frakm}}$. By using an argument very similar to lemma 5 of \cite{MTR23} we can show that $\sigma_{\min}(B^\top - \gamma P_t^\top) \ge \sqrt{A}(1-\gamma)$. Therefore,
    $$
    \sigma_{\min}(M_t) \ge \sigma_{\min}(\Phi^\dagger) \sigma_{\min}(B^\top - \gamma P_t^\top) \ge \frac{\sqrt{A}(1-\gamma)}{\sqrt{\frakm}}
    $$
    Since the singular values of $M_t^\dagger$ are formed by inverting the non-zero singular values of $M_t$ and leaving the zero singular values as they are, we obtain $\norm{M_t^\dagger}_2 \le \frac{\sqrt{\frakm}}{\sqrt{A}(1-\gamma)}$.
\end{proof}

\begin{lemma}\label{lem:bound-dual-solution}
    Let $(h^\star, g^\star)$ be an optimal solution of the optimization problem~\eqref{eq:dual-regularized} of the minimum norm. Then $\norm{h^\star}_2 \le \alpha \left( \lambda \alpha + \sqrt{D}\right)$, where
    $
    \alpha = \frac{\sqrt{\frakm}}{\sqrt{A}(1-\gamma)}.
    $
\end{lemma}
\begin{proof}
    At an optimal solution $(h^\star, g^\star)$ we must have $\nabla_h \calF_t(h^\star, g^\star) = 0$ and $\left \langle \nabla_g \calF_t(h^\star, g^\star), g - g^\star \right \rangle \ge 0$ for any $g \ge 0$. Let $M_t = \Phi^\dagger B^\top - \gamma \bm{\mu}_t^\top$. Then we have,
    \begin{align}
        \nabla_h \calF_t(h^\star, g^\star) &= \frac{1}{\lambda} M_t^\top M_t h^\star + \frac{1}{\lambda} M_t^\top \left( \Phi^\dagger g^\star + \theta_t\right) - \rho = 0\nonumber \\
    \Rightarrow &M_t h^\star + \theta_t = \left( M_t^\dagger\right)^\top \lambda \rho - \Phi^\dagger g^\star. \label{eq:zero-derivative}
    \end{align}
    On the other hand, 
    $$
    \nabla_g \calF_t(h^\star, g^\star) = \frac{1}{\lambda} (\Phi^\dagger)^\top \Phi^\dagger g^\star + \frac{1}{\lambda} (\Phi^\dagger)^\top \left( M_t h^\star + \theta_t\right) =  (\Phi^\dagger)^\top \left( M_t^\dagger\right)^\top \rho
    $$
    Let the $j$-th coordinate of the vector $(\Phi^\dagger)^\top \left( M_t^\dagger\right)^\top \rho$ is non-zero. This implies that the $j$-th coordinate of $g^\star$ is zero, as $\left \langle \nabla_g \calF_t(h^\star, g^\star), g - g^\star \right \rangle \ge 0$ for any $g \ge 0$ and one can take $g = g^\star \pm \frac{g^\star_j}{2}\cdot \bm{e}_j$ to conclude that $g^\star_j = 0$. This is also equivalent to the condition ${g^{\star}}^\top (\Phi^\dagger)^\top \left( M_t^\dagger\right)^\top \rho = 0$.

    We now show that without loss of generality, one can choose $g^\star = 0$. Given a solution $(h^\star, g^\star)$ let us choose $h = h^\star + \Delta$ where $\Delta =  M_t^\dagger \Phi^\dagger g^\star$. Then the objective at the solution $(h,0)$ is the following.
    \begin{align*}
        \calF_t(h, 0) &= \frac{1}{2\lambda}\norm{M_t h +   \theta_t}_2^2  - h^\top \rho\\
        &= \frac{1}{2\lambda}\norm{M_t (h^\star + \Delta) +   \theta_t}_2^2  - (h^\star + \Delta)\rho\\
        &= \frac{1}{2\lambda}\norm{M_t h^\star + M_t \Delta +   \theta_t}_2^2  - (h^\star)^\top \rho - \Delta^\top \rho\\
        &= \frac{1}{2\lambda}\norm{M_t h^\star + \Phi^\dagger g^\star +   \theta_t}_2^2  - (h^\star)^\top \rho - {g^{\star}}^\top (\Phi^\dagger)^\top \left( M_t^\dagger\right)^\top \rho\\
        &= \frac{1}{2\lambda}\norm{M_t h^\star + \Phi^\dagger g^\star +   \theta_t}_2^2  - (h^\star)^\top \rho = \calF_t(h^\star, g^\star)
    \end{align*}
    Substituting $g^\star = 0$ in \cref{eq:zero-derivative} we get the following equation: $M_t h^\star + \theta_t = \lambda \left( M_t^\dagger \right)^\top  \rho$. The solution of this equation is $h^\star = M_t^\dagger \left(\lambda \left( M_t^\dagger \right)^\top  \rho - \theta_t \right)$ and we can bound its norm as follows.
    \begin{align*}
        \norm{h^\star}_2 &\le \norm{M_t^\dagger}_2 \left(\lambda \norm{\left( M_t^\dagger \right)^\top}_2  \norm{\rho}_2 + \norm{ \theta_t}_2 \right)\\
        &\le \frac{\sqrt{\frakm}}{\sqrt{A}(1-\gamma)} \left( \lambda \frac{\sqrt{\frakm}}{\sqrt{A}(1-\gamma)} + \sqrt{D}  \right)\\
        %&\le \left( \frac{\sqrt{\mu_2}}{\mu_1} + \gamma \sqrt{D}\right) \frac{\mu_2^2}{\mu_1 A^2} \left(\lambda \left( \frac{\sqrt{\mu_2}}{\mu_1} + \gamma \sqrt{D}\right) \frac{\mu_2^2}{\mu_1 A^2} + \sqrt{D} \right) 
    \end{align*}
    The last inequality uses \Cref{lem:bound-M-t-pseudoinv}.
    \end{proof}
\section{Proof of Theorem \texorpdfstring{\ref{thm:apx-stability}}{}  }\label{sec:apx-stability-proof}

\begin{proof}
    Let $\calC(d^\lambda_S)$ be the set of occupancy measures that are feasible with respect to the measure $\bm{\mu}_\lambda = \calF_\mu(d^\lambda_S)$ i.e. $\calC(d^\lambda_S)= \set{d :\ Bd = \rho + \gamma \cdot \bm{\mu}_\lambda \Phi^\top d,\ d \ge 0}$. As $d^\lambda_S$ maximizes the objective~\eqref{eq:regularized-rl-matrix-dual}, we have the following bound.
    \begin{equation}
        d^{\lambda^\top}_S\Phi \theta_S - \frac{\lambda}{2} d^{\lambda^\top}_S \Phi \Phi^\top  d^{\lambda}_S \ge \max_{d \in \calC(d^\lambda_S)} d^\top \Phi \theta_S - \frac{\lambda}{2} d^\top \Phi \Phi^\top d
    \end{equation}
    After rearranging and using assumption \ref{asn:features} we get the following lower bound.
    \begin{align*}
        d^{\lambda^\top}_S\Phi \theta_S &\ge \max_{d \in \calC(d^\lambda_S)} d^\top \Phi \theta_S - \frac{\lambda}{2} d^\top \Phi \Phi^\top d + \frac{\lambda}{2} d^{\lambda^\top}_S \Phi \Phi^\top  d^{\lambda}_S\\
        &\ge \max_{d \in \calC(d^\lambda_S)} d^\top \Phi \theta_S -  \frac{\lambda}{2} d^\top \Phi \Phi^\top d\\
        &\ge \max_{d \in \calC(d^\lambda_S)} d^\top \Phi \theta_S - \frac{\lambda}{2} \frakm \sum_a d(\cdot, a)^\top d(\cdot, a)\\
        &\ge \max_{d \in \calC(d^\lambda_S)} d^\top \Phi \theta_S - \frac{\lambda \frakm}{2 (1-\gamma)^2}
    \end{align*}
    The last line uses $\norm{d}_2^2 = \sum_{s,a} d^2(s,a) = (1-\gamma)^{-2} \sum_{s,a} ((1-\gamma) d(s,a))^2 \le (1-\gamma)^{-2} \sum_{s,a} (1-\gamma) d(s,a) = (1-\gamma)^{-2}$. Now we substitute $\lambda = \frac{25\left(\varepsilon_\theta + \alpha \gamma \sqrt{D} \varepsilon_\mu \right)}{8\sqrt{\kappa}}$ and obtain the following inequality.
    $$
    d^{\lambda^\top}_S\Phi \theta_S \ge \max_{d \in \calC(d^\lambda_S)} d^\top \Phi \theta_S - \frac{25\frakm \left( \varepsilon_\theta + \alpha \gamma \sqrt{D} \varepsilon_\mu \right) }{16 \sqrt{\kappa}(1-\gamma)^2}
    $$
\end{proof}
\section{Proof of Theorem \texorpdfstring{\ref{thm:apx-optimality}}{} }\label{sec:apx-optimality-proof}
In order to formally state, and prove \Cref{thm:apx-optimality} we will need to introduce some definitions. 
Let $d^\lambda_\PO$ be the performatively optimal occupancy measure, and $\pi^\lambda_\PO$ be the performatively optimal policy which is defined as follows.
$$
\pi^\lambda_\PO(a \mid s) =  \left\{\begin{array}{cc}
    \frac{d^\lambda_\PO(s,a)}{\sum_b d^\lambda_\PO(s,b)} & \textrm{ if } \sum_b d^\lambda_\PO(s,b) > 0 \\
     \frac{1}{A} & \textrm{ o.w. } 
\end{array} \right.
$$
Let us also define $d^\lambda_{\PO \rightarrow \PO}$ (resp. $d^\lambda_{\PO \rightarrow S}$) to be the occupancy measure of the policy $\pi^\lambda_\PO$ in the MDP $M^\lambda_\PO$ (resp. $M^\lambda_S$). Note that, $d^\lambda_{\PO \rightarrow \PO}$ need not be equal to $d^\lambda_\PO$. 
\begin{theorem}[Formal statement of \Cref{thm:apx-optimality}]\label{thm:apx-optimality-full}
    Suppose the assumptions of \Cref{thm:convergence-rpo} hold, and $\Delta = \frac{3 \gamma \varepsilon_\mu \frakm \sqrt{D} }{(1-\gamma)^2}  + \varepsilon_\theta \sqrt{\frakm }$, and $\lambda_0 = \frac{25}{8\sqrt{\kappa}}\left( \varepsilon_\theta + \alpha \gamma \sqrt{D}\varepsilon_\mu\right)$ (required lower bound from \Cref{thm:convergence-rpo}). Then there exists a choice of regularization parameter $(\lambda)$ such that repeatedly optimizing objective~(\ref{eq:regularized-rl-matrix-dual}) converges to a solution $d^\lambda_S$ with the following guarantee.
    $$
    d_{\PO\rightarrow\PO}^\top \Phi \theta_\PO - d_S^{\lambda^\top} \Phi \theta_S^\lambda \le  4\sqrt{\frac{(1+\Delta) \Delta }{\kappa} \cdot \frac{\frakm}{(1-\gamma)^2}}+ \lambda_0 \cdot \frac{\frakm}{(1-\gamma)^2}
    $$
\end{theorem}

The suboptimality gap established in \Cref{thm:apx-optimality} asymptotically scales as $O(\max\set{1,\Delta} \Delta + \lambda_0)$. As $\varepsilon_\theta, \varepsilon_\mu \rightarrow 0$, both $\lambda_0$ and $\Delta$ converge to zero, and $d^\lambda_S$ approaches a performatively optimal solution with respect to the unregularized objective.

The proof of \Cref{thm:apx-optimality} first upper bounds the suboptimality gap in terms of $\norm{d^\lambda_{\PO \rightarrow S} - d^\lambda_S}_2$ the distance between the occupancy measures resulting from deploying (regularized) optimal policy $(\pi^\lambda_\PO)$ and stable policy $(\pi^\lambda_S)$ in the stable environment. Then \Cref{lem:distance-optimality-stability} provides an upper bound on the distance between these two measures that scales as $O\left({\frac{\Delta}{\lambda}} \right)$ under certain conditions. Substituting this upper bound and then choosing an appropriate value of the regularizer $\lambda$ gives the desired bound.
\begin{proof}
Given a policy $\pi$, we define the matrix $\Pi \in \R^{SA \times S}$ as follows.
$$
\Pi = \begin{bmatrix}
    \pi(\cdot \mid s_1) & 0 & \ldots\\
    0 & \pi(\cdot \mid s_2) & \ldots\\
    \vdots & \vdots & \ddots
\end{bmatrix}
$$
Given a measure $\bm{\mu}$ (i.e. probability transition $P = \bm{\mu} \Phi^\top$ the occupancy measure $d$ of the deployed policy is given as
\begin{equation}\label{eq:policy-to-measure}
d = \Pi \rho + \gamma \Pi \bm{\mu} \Phi^\top \Pi \rho + \gamma^2 \left( \Pi \bm{\mu} \Phi^\top \right)^2 \Pi \rho + \ldots = \left( \Identity - \gamma \Pi \bm{\mu} \Phi^\top \right)^{-1} \Pi \rho
\end{equation}
Furthermore, following objective~\cref{eq:regularized-rl-matrix-dual}, given an MDP $M = (\theta, \bm{\mu})$ we will write the regularized reward of an occupancy measure $d$ as
$$
RR(d; M) = d^\top \Phi \theta - \frac{\lambda}{2} d^\top \Phi \Phi^\top d.
$$
Then, using the definition of performative optimality and stability, we get the following sequence of inequalities.
\begin{equation}\label{eq:seq-optimality-stability}
RR(d^\lambda_{\PO \rightarrow \PO}; M^\lambda_\PO) \ge RR(d^\lambda_S; M^\lambda_S) \ge RR(d^\lambda_{\PO \rightarrow S}; M^\lambda_S)
\end{equation}

\begin{align}
    &d_{\PO\rightarrow\PO}^\top \Phi \theta_\PO - d_S^{\lambda^\top} \Phi \theta_S^\lambda\nonumber\\
    &= \left(d_{\PO\rightarrow\PO}^\top \Phi \theta_\PO - \frac{\lambda}{2}  d_{\PO\rightarrow\PO}^\top \Phi \Phi^\top d_{\PO\rightarrow\PO} \right) +\frac{\lambda}{2}d_{\PO\rightarrow\PO}^\top \Phi \Phi^\top d_{\PO\rightarrow\PO}\nonumber\\
    &-  \left(d_{S}^{\lambda^\top} \Phi \theta_S^\lambda - \frac{\lambda}{2}  d_{S}^{\lambda^\top} \Phi \Phi^\top d_{S}^{\lambda^\top} \right) -\frac{\lambda}{2}d_{S}^{\lambda^\top} \Phi \Phi^\top d_{S}^{\lambda^\top}\nonumber\\
    &\le \left(d_{\PO\rightarrow\PO}^{\lambda^\top} \Phi \theta^\lambda_\PO - \frac{\lambda}{2}  d_{\PO\rightarrow\PO}^{\lambda^\top} \Phi \Phi^\top d_{\PO\rightarrow\PO}^\lambda \right) +\frac{\lambda}{2}d_{\PO\rightarrow\PO}^{\top} \Phi \Phi^\top d_{\PO\rightarrow\PO}\nonumber\\
    &-  \left(d_{S}^{\lambda^\top} \Phi \theta_S^\lambda - \frac{\lambda}{2}  d_{S}^{\lambda^\top} \Phi \Phi^\top d_{S}^{\lambda^\top} \right) \nonumber\\
    &= \underbrace{RR(d^\lambda_{\PO \rightarrow \PO}; M^\lambda_\PO) - RR(d^\lambda_S; M^\lambda_S)}_{:= T_1} + \frac{\lambda}{2}d_{\PO\rightarrow\PO}^{\top} \Phi \Phi^\top d_{\PO\rightarrow\PO} \label{eq:gap-performative-optimality}
\end{align}
Using \cref{eq:seq-optimality-stability} we can upper bound the term $T_1$ as follows.
\begin{align}
    T_1 &\le RR(d^\lambda_{\PO \rightarrow \PO}; M^\lambda_\PO) - RR(d^\lambda_{S}; M^\lambda_S)\nonumber \\
    &\le \left(d_{\PO\rightarrow\PO}^{\lambda^\top} \Phi \theta^\lambda_\PO - \frac{\lambda}{2}  d_{\PO\rightarrow\PO}^{\lambda^\top} \Phi \Phi^\top d_{\PO\rightarrow\PO}^\lambda \right) 
    - \left(d_{S}^{\lambda^\top} \Phi \theta^\lambda_S - \frac{\lambda}{2}  d_{S}^{\lambda^\top} \Phi \Phi^\top d_{S}^\lambda \right)\nonumber \\
    &= \left(d_{\PO\rightarrow\PO}^{\lambda^\top} \Phi \theta^\lambda_\PO -  d_{ S}^{\lambda^\top} \Phi \theta^\lambda_{\PO}\right) + \left(d_{ S}^{\lambda^\top} \Phi \theta^\lambda_{\PO} - d_{S}^{\lambda^\top} \Phi \theta^\lambda_{S} \right)\nonumber \\
    &- \frac{\lambda}{2}  d_{\PO\rightarrow\PO}^{\lambda^\top} \Phi \Phi^\top d_{\PO\rightarrow\PO}^\lambda + \frac{\lambda}{2}  d_{S}^{\lambda^\top} \Phi \Phi^\top d_{S}^\lambda\nonumber  \\
    &\le \norm{d_{\PO\rightarrow\PO}^\lambda - d_{S}^\lambda}_2\norm{\Phi \theta^\lambda_\PO}_2 + \norm{d_{S}^{\lambda^\top} \Phi}_2 \norm{\theta^\lambda_\PO - \theta^\lambda_S}_2 + \frac{\lambda}{2}  d_{S}^{\lambda^\top} \Phi \Phi^\top d_{S}^\lambda\nonumber \\
    &\le \left(\norm{d_{\PO\rightarrow\PO}^\lambda - d_{\PO \rightarrow S}^\lambda}_2 + \norm{d_{\PO\rightarrow S}^\lambda - d_{S}^\lambda}_2\right)\norm{\Phi \theta^\lambda_\PO}_2 + \norm{d_{S}^{\lambda^\top} \Phi}_2 \norm{\theta^\lambda_\PO - \theta^\lambda_S}_2 + \frac{\lambda}{2}  d_{S}^{\lambda^\top} \Phi \Phi^\top d_{S}^\lambda \label{eq:temp-bound-T1}
\end{align}

Using \cref{eq:policy-to-measure} we get,
\begin{align*}
    &\norm{d_{\PO\rightarrow\PO}^\lambda - d_{\PO\rightarrow S}^\lambda}_2 = \norm{\left\{\left(\Identity - \gamma \Pi_{\PO} \bm{\mu}_{\PO} \Phi^\top \right)^{-1} - \left(\Identity - \gamma \Pi_{\PO} \bm{\mu}_{S} \Phi^\top \right)^{-1}\right\} \Pi_{\PO}\rho}_2\\
    &\le \norm{\left\{\left(\Identity - \gamma \Pi_{\PO} \bm{\mu}_{\PO} \Phi^\top \right)^{-1} - \left(\Identity - \gamma \Pi_{\PO} \bm{\mu}_{S} \Phi^\top \right)^{-1}\right\}}_2 \norm{ \Pi_{\PO}\rho}_2\\
    &\le 3\gamma \norm{\Pi_{\PO}\left(\bm{\mu}_{\PO} - \bm{\mu}_S \right) \Phi^\top}_2 \norm{\left(\Identity - \gamma \Pi_{\PO} \bm{\mu}_{\PO} \Phi^\top \right)^{-1}}_2 \norm{\left(\Identity - \gamma \Pi_{\PO} \bm{\mu}_{\PO} \Phi^\top \right)^{-1}}_2  \norm{ \Pi_{\PO}\rho}_2
\end{align*}
The last inequality uses perturbation bound for the inverse of a matrix, in particular $\norm{A^{-1}}_2 - \norm{B^{-1}}_2 \le 3 \norm{A - B}_2 \norm{A^{-1}}_2 \norm{B^{-1}}_2$. We will now use the following set of observations.
\begin{enumerate}
    \item $\norm{\Pi_{\PO}}_2 \le 1$ as $\norm{\Pi_{\PO}v}_2^2 = \sum_{s,a} \pi^2(a\mid s) v^2_{s,a} \le \norm{v}_2^2 $.
    \item Since $\bm{\mu}_S \Phi^\top$ is a probability transition function we have $\norm{\Pi_{\PO}\bm{\mu}_S \Phi^\top}_2 \le \norm{\Pi_{\PO}}_2 \norm{\bm{\mu}_S \Phi^\top}_2 \le 1 $. This also implies that $\Identity - \gamma \Pi_{\PO}\bm{\mu}_S \Phi^\top \succcurlyeq (1-\gamma) \cdot \Identity$.
    \item By assumption~\ref{asn:features}, we have $\norm{\Phi^\top}_2 \le \sqrt{\frakm}$.
\end{enumerate}
\begin{equation*}
    \norm{d_{\PO\rightarrow\PO}^\lambda - d_{\PO\rightarrow S}^\lambda}_2 \le \frac{3 \gamma \sqrt{\frakm }}{(1-\gamma)^2} \norm{\bm{\mu}_{\PO} - \bm{\mu}_S}_2  \le \frac{3 \gamma \varepsilon_\mu \sqrt{\frakm }}{(1-\gamma)^2} \norm{d^\lambda_{\PO\rightarrow S} - d^\lambda_S}_2
\end{equation*}
The last inequality uses assumption~\eqref{asn:lipschitzness} and the fact that the measure $d_{\PO \rightarrow S}$ induces policy $\pi_{\PO}$ i.e. $\pi_{d_{\PO \rightarrow S}}$, as defined in \eqref{eq:measure-to-policy} equals $\pi_{\PO}$.
Substituting the above bound in \cref{eq:temp-bound-T1} we get the following upper bound on $T_1$.
\begin{align*}
    T_1 &\le \left(1 + \frac{3 \gamma \varepsilon_\mu \sqrt{\frakm }}{(1-\gamma)^2} \right) \norm{d^\lambda_{\PO\rightarrow S} - d^\lambda_S}_2 \norm{\Phi}_2 \norm{\theta^\lambda_{\PO}}_2 + \norm{d^\lambda_{S}}_2 \norm{\Phi}_2 \varepsilon_\theta \norm{d^\lambda_{\PO\rightarrow S} - d^\lambda_S}_2 + \frac{\lambda}{2} \norm{d^\lambda_{ S}}_{\Phi \Phi^\top}\\
    &\le \left(1 + \underbrace{\frac{3 \gamma \varepsilon_\mu \frakm \sqrt{D} }{(1-\gamma)^2}  + \varepsilon_\theta \sqrt{\frakm }}_{:=\Delta} \right) \norm{d^\lambda_{\PO\rightarrow S} - d^\lambda_S}_2 + \frac{\lambda}{2} \norm{d^\lambda_{S}}_{\Phi \Phi^\top}
\end{align*}
Substituting the above bound in \cref{eq:gap-performative-optimality} and using the observation that for any occupancy measure $d$ we have $d^\top \Phi \Phi^\top d \le \frac{\frakm }{(1-\gamma)^2}$ we get the following bound.
\begin{align}
    d^\top_{\PO \rightarrow \PO} \Phi \theta_{\PO} - d^{\lambda^\top}_S \Phi \theta^\lambda_S &\le \left( 1 + \Delta\right) \norm{d^\lambda_{\PO \rightarrow S} - d^\lambda_S}_2 + \frac{\lambda\frakm }{(1-\gamma)^2}\nonumber \\
    &\le \left(1 + \Delta\right)\frac{4\Delta}{\kappa \lambda}  + \frac{\lambda\frakm }{(1-\gamma)^2}\ \textrm{[By \Cref{lem:distance-optimality-stability}]}\ \label{eq:temp-bound-optimality-difference}
\end{align}
Note that the above expression can be written as $\frac{S_1}{{\lambda}} + \lambda \cdot S_2$ where 
$$
S_1 = \left(1 + \Delta\right)\frac{4\Delta}{\kappa} \ \textrm{ and } \ S_2 = \frac{\frakm }{(1-\gamma)^2}
$$
Let $\lambda_0 = \frac{25}{8\sqrt{\kappa}}\left(\varepsilon_\theta + \alpha \gamma \sqrt{D} \varepsilon_\mu \right)$. There are two cases to consider. If $\left( \frac{T_1}{T_2}\right)^{1/2} \ge \lambda_0$, we can substitute $\lambda = \left( \frac{T_1}{T_2}\right)^{1/2}$ in \cref{eq:temp-bound-optimality-difference} to obtain 
$$
d^\top_{\PO \rightarrow \PO} \Phi \theta_{\PO} - d^{\lambda^\top}_S \Phi \theta^\lambda_S \le 2 \sqrt{S_1 S_2} .
$$
On the other hand, if $\left( \frac{S_1}{S_2}\right)^{1/2} < \lambda_0$, we can substitute $\lambda = \lambda_0$ in \cref{eq:temp-bound-optimality-difference} to obtain the following bound.
$$
d^\top_{\PO \rightarrow \PO} \Phi \theta_{\PO} - d^{\lambda^\top}_S \Phi \theta^\lambda_S \le \frac{S_1}{\sqrt{\lambda_0}} + \lambda_0 \cdot S_2 \le \sqrt{S_1 S_2} + \lambda_0 \cdot S_2
$$
Combining the two bounds above, we are always guaranteed an upper bound of $2 \sqrt{S_1 S_2} + \lambda_0 \cdot S_2$ on the suboptimality gap. 
\end{proof}

\subsection{Distance Between Performatively Optimal and Stable Solution}
\begin{lemma}\label{lem:distance-optimality-stability}
    Let $\Delta = \left(\frac{3 \gamma \varepsilon_\mu  \frakm \sqrt{D} }{(1-\gamma)^2}  + \varepsilon_\theta \sqrt{\frakm } \right)$ and $c \cdot \Delta \ge \lambda \ge \frac{25}{8\sqrt{\kappa}}\left( \varepsilon_\theta + \alpha \gamma \sqrt{D}\varepsilon_\mu \right)$ for a constant $c \ge \frac{4}{\sqrt{\kappa \frakm}}$. Then $$\norm{d^\lambda_{\PO \rightarrow S} - d^\lambda_S}_2 \le \frac{4\Delta}{\kappa \lambda}.$$
\end{lemma}
\begin{proof}
    We first provide a lower bound on the difference $RR(d^\lambda_S; M^\lambda_S) - RR(d^\lambda_{\PO \rightarrow S}; M^\lambda_S)$. Since the occupancy measure $d^\lambda_S$ maximizes $RR(\cdot; M^\lambda_S)$ we have $\left \langle \nabla_d RR(d^\lambda_S;M^\lambda_S), d^\lambda_S - d^\lambda_{\PO \rightarrow S} \right \rangle \ge 0$. This implies the following inequality.
    \begin{align}
        &\left(d^\lambda_S - d^\lambda_{\PO \rightarrow S} \right)^\top \left( \Phi \theta^\lambda_S - \lambda \Phi \Phi^\top d^\lambda_S\right)\ge 0\nonumber \\
        \Rightarrow\ & d^{\lambda^\top}_S \Phi \theta^\lambda_S - \lambda d^{\lambda^\top}_S \Phi \Phi^\top d^\lambda_S \ge d^{\lambda^\top}_{\PO \rightarrow S} \Phi \theta^\lambda_S - \lambda d^{\lambda^\top}_{\PO \rightarrow S} \Phi \Phi^\top d^\lambda_S\nonumber \\
        \Rightarrow\ &RR(d^\lambda_S; M^\lambda_S) - \frac{\lambda}{2} d^{\lambda^\top}_S \Phi \Phi^\top d^\lambda_S \ge RR(d^\lambda_{\PO \rightarrow S}; M^\lambda_S) - \lambda d^{\lambda^\top}_{\PO \rightarrow S} \Phi \Phi^\top d^\lambda_S + \frac{\lambda}{2} d^{\lambda^\top}_{\PO \rightarrow S} \Phi \Phi^\top d^\lambda_{\PO \rightarrow S}\nonumber\\
        \Rightarrow\ &RR(d^\lambda_S; M^\lambda_S) - RR(d^\lambda_{\PO \rightarrow S}; M^\lambda_S) \ge \frac{\lambda}{2} \left(d^\lambda_S - d^\lambda_{\PO \rightarrow S} \right)^\top \Phi \Phi^\top \left(d^\lambda_S - d^\lambda_{\PO \rightarrow S} \right) \ge \frac{\kappa \lambda}{2} \norm{d^\lambda_S - d^\lambda_{\PO \rightarrow S}}_2^2 \label{eq:lower-bound-diff}
    \end{align}
The last inequality uses assumption~\eqref{asn:features}. We now provide an upper bound on the term $RR(d^\lambda_S; M^\lambda_S) - RR(d^\lambda_{\PO \rightarrow S}; M^\lambda_S)$. Note that $RR(d^\lambda_S; M^\lambda_S) - RR(d^\lambda_{\PO \rightarrow S}; M^\lambda_S) \le RR(d^\lambda_{\PO \rightarrow \PO}; M^\lambda_{\PO}) - RR(d^\lambda_{\PO \rightarrow S}; M^\lambda_S)$, and we upper bound the latter.
\begin{align}
    &RR(d^\lambda_{\PO \rightarrow \PO}; M^\lambda_{\PO}) - RR(d^\lambda_{\PO \rightarrow S}; M^\lambda_S)\nonumber\\
     &\le \left(d_{\PO\rightarrow\PO}^{\lambda^\top} \Phi \theta^\lambda_\PO - \frac{\lambda}{2}  d_{\PO\rightarrow\PO}^{\lambda^\top} \Phi \Phi^\top d_{\PO\rightarrow\PO}^\lambda \right) 
    - \left(d_{\PO\rightarrow S}^{\lambda^\top} \Phi \theta^\lambda_S - \frac{\lambda}{2}  d_{\PO\rightarrow S}^{\lambda^\top} \Phi \Phi^\top d_{\PO\rightarrow S}^\lambda \right)\nonumber \nonumber\\
    &= \left(d_{\PO\rightarrow\PO}^{\lambda^\top} \Phi \theta^\lambda_\PO -  d_{\PO\rightarrow S}^{\lambda^\top} \Phi \theta^\lambda_{\PO}\right) + \left(d_{\PO\rightarrow S}^{\lambda^\top} \Phi \theta^\lambda_{\PO} - d_{\PO\rightarrow S}^{\lambda^\top} \Phi \theta^\lambda_{S} \right)\nonumber \nonumber\\
    &- \frac{\lambda}{2}  d_{\PO\rightarrow\PO}^{\lambda^\top} \Phi \Phi^\top d_{\PO\rightarrow\PO}^\lambda + \frac{\lambda}{2}  d_{\PO\rightarrow S}^{\lambda^\top} \Phi \Phi^\top d_{\PO\rightarrow S}^\lambda\nonumber  \nonumber\\
    &\le \norm{d_{\PO\rightarrow\PO}^\lambda - d_{\PO\rightarrow S}^\lambda}_2\norm{\Phi \theta^\lambda_\PO}_2 + \norm{d_{\PO\rightarrow S}^{\lambda^\top} \Phi}_2 \norm{\theta^\lambda_\PO - \theta^\lambda_S}_2 + \frac{\lambda}{2}  d_{\PO\rightarrow S}^{\lambda^\top} \Phi \Phi^\top d_{\PO\rightarrow S}^\lambda \label{eq:upper-bound-diff}
\end{align}
Using \cref{eq:policy-to-measure} we get,
\begin{align*}
    &\norm{d_{\PO\rightarrow\PO}^\lambda - d_{\PO\rightarrow S}^\lambda}_2 = \norm{\left\{\left(\Identity - \gamma \Pi_{\PO} \bm{\mu}_{\PO} \Phi^\top \right)^{-1} - \left(\Identity - \gamma \Pi_{\PO} \bm{\mu}_{S} \Phi^\top \right)^{-1}\right\} \Pi_{\PO}\rho}_2\\
    &\le \norm{\left\{\left(\Identity - \gamma \Pi_{\PO} \bm{\mu}_{\PO} \Phi^\top \right)^{-1} - \left(\Identity - \gamma \Pi_{\PO} \bm{\mu}_{S} \Phi^\top \right)^{-1}\right\}}_2 \norm{ \Pi_{\PO}\rho}_2\\
    &\le 3\gamma \norm{\Pi_{\PO}\left(\bm{\mu}_{\PO} - \bm{\mu}_S \right) \Phi^\top}_2 \norm{\left(\Identity - \gamma \Pi_{\PO} \bm{\mu}_{\PO} \Phi^\top \right)^{-1}}_2 \norm{\left(\Identity - \gamma \Pi_{\PO} \bm{\mu}_{\PO} \Phi^\top \right)^{-1}}_2  \norm{ \Pi_{\PO}\rho}_2
\end{align*}
The last inequality uses perturbation bound for the inverse of a matrix, in particular $\norm{A^{-1}}_2 - \norm{B^{-1}}_2 \le 3 \norm{A - B}_2 \norm{A^{-1}}_2 \norm{B^{-1}}_2$. We will now use the following set of observations.
\begin{enumerate}
    \item $\norm{\Pi_{\PO}}_2 \le 1$ as $\norm{\Pi_{\PO}v}_2^2 = \sum_{s,a} \pi^2(a\mid s) v^2_{s,a} \le \norm{v}_2^2 $.
    \item Since $\bm{\mu}_S \Phi^\top$ is a probability transition function we have $\norm{\Pi_{\PO}\bm{\mu}_S \Phi^\top}_2 \le \norm{\Pi_{\PO}}_2 \norm{\bm{\mu}_S \Phi^\top}_2 \le 1 $. This also implies that $\Identity - \gamma \Pi_{\PO}\bm{\mu}_S \Phi^\top \succcurlyeq (1-\gamma) \cdot \Identity$.
    \item By assumption \ref{asn:features} we have, $\norm{\Phi^\top}_2 \le \sqrt{\frakm}$.
\end{enumerate}
\begin{equation*}
    \norm{d_{\PO\rightarrow\PO}^\lambda - d_{\PO\rightarrow S}^\lambda}_2 \le \frac{3 \gamma \sqrt{\frakm}}{(1-\gamma)^2} \norm{\bm{\mu}_{\PO} - \bm{\mu}_S}_2  \le \frac{3 \gamma \varepsilon_\mu \sqrt{\frakm}}{(1-\gamma)^2} \norm{d^\lambda_{\PO\rightarrow S} - d^\lambda_S}_2
\end{equation*}
The last inequality uses assumption~\ref{asn:lipschitzness} and the fact that the measure $d_{\PO \rightarrow S}$ induces policy $\pi_{\PO}$ i.e. $\pi_{d_{\PO \rightarrow S}}$, as defined in \eqref{eq:measure-to-policy} equals $\pi_{\PO}$.
Substituting the above bound in \cref{eq:upper-bound-diff} we get the following upper bound.
\begin{align*}
     &RR(d^\lambda_{\PO \rightarrow \PO}; M^\lambda_{\PO}) - RR(d^\lambda_{\PO \rightarrow S}; M^\lambda_S) \\ &\le \frac{3 \gamma \varepsilon_\mu \sqrt{\frakm}}{(1-\gamma)^2} \norm{d^\lambda_{\PO\rightarrow S} - d^\lambda_S}_2 \norm{\Phi}_2 \norm{\theta^\lambda_{\PO}}_2 + \norm{d^\lambda_{\PO \rightarrow S}}_2 \norm{\Phi}_2 \varepsilon_\theta \norm{d^\lambda_{\PO} - d^\lambda_S}_2 + \frac{\lambda}{2} \norm{d^\lambda_{\PO \rightarrow S}}_{\Phi \Phi^\top}\\
    &\le \left(\frac{3 \gamma \varepsilon_\mu \frakm \sqrt{D} }{(1-\gamma)^2}  + \varepsilon_\theta \sqrt{\frakm} \right) \norm{d^\lambda_{\PO\rightarrow S} - d^\lambda_S}_2 + \frac{\lambda}{2} \norm{d^\lambda_{\PO \rightarrow S}}_{\Phi \Phi^\top}\\
    &\le \underbrace{\left(\frac{3 \gamma \varepsilon_\mu \frakm \sqrt{D} }{(1-\gamma)^2}  + \varepsilon_\theta \sqrt{\frakm} \right)}_{:= \Delta } \norm{d^\lambda_{\PO\rightarrow S} - d^\lambda_S}_2 + \frac{\lambda \frakm}{2(1-\gamma)^2}
\end{align*}
Using the lower bound established in \cref{eq:lower-bound-diff} we obtain the following inequality.
\begin{align*}
    \frac{\kappa \lambda}{2} \norm{d^\lambda_S - d^\lambda_{\PO \rightarrow S}}_2^2 \le \Delta \norm{d^\lambda_{\PO\rightarrow S} - d^\lambda_S}_2 + \frac{\lambda \frakm}{2(1-\gamma)^2}
\end{align*}
Now there are two cases to consider. First, $\Delta \norm{d^\lambda_{\PO\rightarrow S} - d^\lambda_S}_2 \le \frac{\lambda \frakm }{2(1-\gamma)^2}$. Then the upper bound on $\norm{d^\lambda_{\PO\rightarrow S} - d^\lambda_S}_2$ is $\frac{\lambda \frakm }{2 \Delta (1-\gamma)^2}$. Second, $\Delta \norm{d^\lambda_{\PO\rightarrow S} - d^\lambda_S}_2 > \frac{\lambda \frakm }{2(1-\gamma)^2}$. Then we have $\frac{\kappa \lambda}{2} \norm{d^\lambda_{\PO\rightarrow S} - d^\lambda_S}_2^2 \le 2 \Delta \norm{d^\lambda_{\PO\rightarrow S} - d^\lambda_S}_2 $ and the upper bound on $\norm{d^\lambda_{\PO\rightarrow S} - d^\lambda_S}_2 $ is $\frac{4\Delta}{\kappa \lambda}$.

Let $\lambda_0 = \frac{25}{8\sqrt{\kappa}}\left(\varepsilon_\theta + \gamma \sqrt{D} \varepsilon_\mu \frac{\sqrt{\frakm}}{\sqrt{A}(1-\gamma)}\right)$ be the required lower bound on $\lambda$. From the definition of $\Delta$ we have, $\frac{\Delta}{\sqrt{\frakm}} \ge \varepsilon_\theta + \gamma \sqrt{D} \varepsilon_\mu \frac{\sqrt{\frakm}}{\sqrt{A}(1-\gamma)} = \lambda_0 \cdot \frac{8 \sqrt{\kappa}}{25}$. Therefore, the inequality $c \cdot \Delta \ge \lambda \ge \lambda_0$ is feasible as long as $c \ge \frac{25}{8\sqrt{\kappa \frakm}}$. Now note that, if $\lambda < \frac{2\sqrt{2} \Delta (1-\gamma)}{\sqrt{\kappa \frakm}}$ then $\frac{4\Delta}{\kappa \lambda} > \frac{\lambda \frakm}{2\Delta(1-\gamma)^2}$ and the upper bound on $\norm{d^\lambda_{\PO\rightarrow S} - d^\lambda_S}_2 $ is $\frac{4\Delta}{\kappa \lambda}$. Therefore, we need the constant $c$ to satisfy the following inequality.
$$
c \ge \max \left\{\frac{25}{8\sqrt{\kappa \frakm}},  \frac{2\sqrt{2}  (1-\gamma)}{\sqrt{\kappa \frakm}}\right\}
$$
Therefore, it is sufficient to take $c \ge 4/\sqrt{\kappa \frakm}$.
\end{proof} 
\section{Proof of Theorem \texorpdfstring{\ref{thm:finite-sample-convergence}}{} }
\begin{proof}
We first construct the dual problem of the optimization problem~\cref{eq:regularized-rl-matrix-dual-new}. Let $\calF_t$ be the function defined as $\calF_t(g,\omega) = \max_{d \ge 0, \nu} \calL_t(d,\nu; g, \omega)$. Then the dual optimization problem defined in \cref{eq:regularized-rl-matrix-dual-new} is given as
$$
\min_{g,\omega} \calF_t(g,\omega)
$$
Fix a choice of $g$ and $\omega$. An optimal solution $\nu$ satisfies
$$
\nabla_\nu \calF_t(d,\nu; g, \omega) = \theta_t - \lambda \nu + \gamma \cdot \bm{\mu}_t^\top g - \omega = 0 \ \textrm{ and } \ \nu = \frac{1}{\lambda} \left( \theta_t + \gamma \cdot \bm{\mu}_t^\top g - \omega \right)
$$
On the other hand, the derivative with respect to $d$ is given as follows.
$$
\nabla_d \calF_t(d,\nu; g, \omega) = -B^\top g + \Phi \omega
$$
If any entry of $-B^\top g + \Phi \omega$ is positive, we can choose the corresponding entry of $d$ to be arbitrarily large, and the value $\calF_t(g,\omega)$ would be unbounded. Therefore, we must have $-B^\top g + \Phi \omega \le 0$. Now substituting the choice of $\nu$ derived above we get the following dual optimization problem.
\begin{align*}
    \min_{g, \omega}\ & \frac{1}{2\lambda} \norm{\theta_t + \gamma \cdot \bm{\mu}_t^\top g - \omega}_2^2 + \left \langle g, \rho \right \rangle\\
    \textrm{s.t.}\ &  -B^\top g + \Phi \omega \le 0
\end{align*}
We now apply \Cref{lem:bound-optima-quadpro} to bound the norm of the optimal solution to the above optimization problem. Note that the objective can be written in the form $x^\top A_t x + b_t^\top v$ where 
$$
A_t = \frac{1}{2\lambda} \cdot \begin{bmatrix}
    \Identity_D & - \gamma \cdot \bm{\mu}_t^\top \\
    -\gamma \cdot \bm{\mu}_t & \gamma^2 \cdot \bm{\mu}_t \bm{\mu}_t^\top 
\end{bmatrix}
$$
and
$$
b_t = \frac{1}{\lambda} \cdot [-\theta_t;\ \lambda \rho + \bm{\mu}_t \theta_t].
$$
Suppose the eigenvalues of $\bm{\mu}_t \bm{\mu}_t^\top$ are $\sigma_1, \ldots, \sigma_D$. Then we claim that the eigenvalues of $A_t$ are $\frac{1}{2\lambda},\ldots, \frac{1}{2\lambda}, \frac{\gamma^2}{2\lambda} \sigma_1, \ldots, \frac{\gamma^2}{2\lambda} \sigma_D$. Indeed, let $v_i$ be the $i$-th eigenvector of $\bm{\mu}_t \bm{\mu}_t^\top$ and let $u_i =[0_D, v_i]$. Then $u_i^\top A_t u_i = \frac{\gamma^2}{2\lambda} u_i^\top \bm{\mu}_t \bm{\mu}_t^\top u_i = \frac{\gamma^2}{2\lambda} \sigma_i \norm{u_i}_2^2$. Therefore, the smallest positive eigenvalue of the matrix $A_t$, denoted as $\sigma^\star_{\min}(A_t)$ is bounded below as
$$
\sigma_t = \sigma^\star_{\min}(A_t) \ge  \frac{\min\set{1, {\gamma^2} \cdot \sigma^\star_{\min}(\bm{\mu}_t \bm{\mu}_t^\top) }}{2\lambda} \ge \frac{\min \set{1, {\gamma^2} \underline{\sigma} }}{2\lambda}
$$
where $\underline{\sigma} = \min_d \sigma^\star_{\min} \bm{\mu}_d \bm{\mu}_d^\top$. Furthermore, $\norm{b}_2 \le \frac{1}{\lambda}\left( \norm{\theta_t}_2 + \lambda \norm{\rho}_2 + \norm{\bm{\mu}_t \theta_t}_2\right) \le \frac{1}{\lambda}\left( \sqrt{D} + \lambda + D\right)$. Therefore, we can apply \Cref{lem:bound-optima-quadpro} to obtain the following bound on the optimal dual solution.
\begin{equation*}
    \norm{\omega^\star}_2 + \norm{g^\star}_2 \le \frac{\lambda + 2D}{\min \set{1, {\gamma^2} \underline{\sigma} } } 
\end{equation*}
We will write $c_1$ as $\min \set{1, \gamma^2 \underline{\sigma}}$. 

Let $(d^\star_t, \nu^\star_t; g^\star_t, \omega^\star_t) \in \argmax_{d,\nu} \min_{g,\omega} \calL_t(d,\nu; g,\omega)$. In the previous paragraph, we showed that it is sufficient to consider $\norm{\omega^\star_t}_2, \norm{g^\star_t}_2 \le \frac{\lambda + 2D}{\gamma^2 \cdot \underline{\sigma} }$. Since $d^\star_t$ is an occupancy measure and $\sum_{s,a} d^\star_t(s,a) = \frac{1}{1-\gamma}$ we have
$$
\sqrt{\sum_{s,a}\left( d^\star_t(s,a)\right)^2} = \frac{1}{1-\gamma}\sqrt{\sum_{s,a}\left( (1-\gamma) \cdot d^\star_t(s,a)\right)^2} \le \frac{1}{1-\gamma}\sqrt{\sum_{s,a} (1-\gamma) \cdot d^\star_t(s,a)} = \frac{1}{1-\gamma}.
$$
On the other hand, note that $\nu^\star_t = \Phi^\top d^\star_t$ and by assumption ~\ref{asn:bounded-coverage} we obtain the following bound.
$$
\norm{ \Sigma_t^{-1} \nu^\star_t}_2^2 = \norm{ \Sigma_t^{-1} \Phi^\top d^\star_t}_2^2 = d^{\star^\top}_t \Phi \Sigma_t^{-2} \Phi^\top d^\star_t = \E_{(s,a) \sim d^\star_t}[\phi(s,a)^\top] \Sigma_t^{-2}  \E_{(s,a) \sim d^\star_t}[\phi(s,a)]  \le B
$$
Therefore, $\norm{\Sigma_t^{-1} \nu^\star_t}_2 \le \sqrt{B}$ and for solving the empirical Lagrangian we can restrict the parameters so that $\norm{\Sigma_t^{-1} \nu^\star_t}_2 \le \sqrt{B}$. Therefore, we can apply \Cref{lem:apx-saddle-point} with $m_t = O\left(\frac{D^5 B \lambda^4}{(1-\gamma)^2 c_1^4 \varepsilon^2}\log \frac{DB\lambda t}{c_1 \varepsilon p} \right)$, and obtain that the following bound holds with probability at least $1-\frac{p}{t^2} \cdot \frac{6}{2\pi^2}$.
\begin{equation} \label{eq:apx-saddle-time-t}
\calL_t(d^\star_t, \nu^\star_t; g^\star_t, \omega^\star_t) - \calL_t(\widehat{d}_t, \widehat{\nu}_t; g^\star_t, \omega^\star_t) \le 2 \varepsilon
\end{equation}
Therefore, by a union bound the bound in eq. holds for any $t$ with probability at least $1-\sum_{t} \frac{p}{t^2} \cdot \frac{6}{2\pi^2} = 1- p/2$. Now observe that the objective $\calL_t(d,\nu;g,\omega)$ is strongly concave in $\nu$ as $\nabla^2_\nu \calL_t(d,\nu;g,\omega) = -\lambda \cdot \Identity_D$. Since given $g^\star_t, \omega^\star_t$, $(d^\star_t, \nu^\star_t)$ is an optimal solution of $\calL_t(\cdot, g^\star_t, \omega^\star_t)$ we have, 
\begin{equation}\label{eq:diff-nu}
\norm{\nu^\star_t - \widehat{\nu}_t}_2  \le \sqrt{\frac{\calL_t(d^\star_t, \nu^\star_t; g^\star_t, \omega^\star_t) - \calL_t(\widehat{d}_t, \widehat{\nu}_t; g^\star_t, \omega^\star_t)}{2\lambda} } \le \sqrt{\frac{\varepsilon}{\lambda}}
\end{equation}
Now note that $\nu^\star_t = \Phi^\top d^\star_t$, but $\widehat{\nu}_t \neq \Phi^\top \widehat{d}_t$. However, observe that given $\widehat{d}, \widehat{\nu}, \widehat{g}$, $\widehat{\omega}$ is an optimal solution to the optimization problem $\min_\omega \widehat{\calL}_t(\widehat{d}, \widehat{\nu}, \widehat{g}, \omega)$. Therefore, $\nabla_{\omega} \widehat{\calL}_t(\widehat{d}, \widehat{\nu}, \widehat{g}, \widehat{\omega}) = 0$ and we obtain the following equality.
$$
\frac{1}{m_t} \sum_{j=1}^{m_t} \phi(s_j,a_j) \phi(s_j,a_j)^\top \Sigma_t^{-1} \widehat{\nu}_t = \Phi^\top \widehat{d}_t
$$
This gives us the following bound with probability at least $1-\frac{p}{t^2} \cdot \frac{6}{2\pi^2}$
\begin{align*}
    \norm{\widehat{\nu}_t - \Phi^\top \widehat{d}_t}_2 &= \norm{\left(\frac{1}{m_t} \sum_{j=1}^{m_t} \phi(s_j,a_j) \phi(s_j,a_j)^\top - \Identity_d \right) \Sigma_t^{-1} \widehat{\nu}}_2\\
    &\le \norm{\left(\frac{1}{m_t} \sum_{j=1}^{m_t} \phi(s_j,a_j) \phi(s_j,a_j)^\top - \Identity_d \right)} \norm{ \Sigma_t^{-1} \widehat{\nu}}_2\\
    &\le 8\sqrt{B}\max \left\{ \sqrt{\frac{D + \log(t/p)}{m_t}}, \frac{D + \log(t/p)}{m_t}\right\}
\end{align*}
The last inequality uses standard concentration inequality for sample covariance matrix (e.g. see theorem 1.6.2 of \cite{tropp2015}). Now substituting $m_t \ge O\left(\frac{D^5 B \lambda^4}{(1-\gamma)^2 c_1^4 \varepsilon^2}\log \frac{DB\lambda t}{c_1 \varepsilon p} \right)$ we obtain $\norm{\widehat{\nu}_t - \Phi^\top \widehat{d}_t}_2 \le  8\sqrt{B}\varepsilon$ with probability at least $1-\frac{p}{t^2} \cdot \frac{6}{2\pi^2}$. Therefore, by a union bound, we obtain that with probability at least $1-p/2$, for any $t$ we have $\norm{\widehat{\nu}_t - \Phi^\top \widehat{d}_t}_2 \le 8 \sqrt{B} \varepsilon$. This result, along with \cref{eq:diff-nu} gives us the following inequality.
\begin{align*}
    \sqrt{\kappa} \norm{d^\star_t - \widehat{d}_t} \le \norm{\Phi^\top d^\star_t - \Phi^\top \widehat{d}_t} \le \norm{\nu^\star_t - \widehat{\nu}_t}_2 + \norm{\widehat{\nu}_t - \Phi^\top \widehat{d}_t}_2 \le \sqrt{\frac{\varepsilon}{\lambda}    } + 8\sqrt{B} \varepsilon
\end{align*}
After rearranging we obtain,
\begin{equation}
    \label{eq:diff-apx-exact-d}
    \norm{d^\star_t - \widehat{d}_t} \le \sqrt{\frac{\varepsilon}{\lambda \kappa}} + 8 \sqrt{\frac{B}{\kappa}} \varepsilon.
\end{equation}
The proof of theorem~\eqref{thm:convergence-rpo} establishes the following recurrence relation.
\begin{equation}
    \norm{d_{t+1}^\star - d_S}_2 \le \beta \left( \norm{d_{t}^\star - d_S}_2 + \norm{d_{t-1}^\star - d_S}_2\right)
\end{equation}
for a constant $\beta$. Using the two inequalities above, we can establish a recurrence relation on the norm of the difference between $\widehat{d}_t$ and $d_S$.
\begin{align*}
    \norm{\widehat{d}_{t+1} - d_S}_2 &\le \norm{d^\star_{t+1} - \widehat{d}_{t+1}} + \norm{d_{t+1}^\star - d_S}_2 \\
    &\le \beta \left( \norm{d_{t}^\star - d_S}_2 + \norm{d_{t-1}^\star - d_S}_2\right) + \sqrt{\frac{\varepsilon}{\lambda \nu}} + 8 \sqrt{\frac{B}{\nu}} \varepsilon\\
    &\le \beta \left( \norm{\widehat{d}_{t} - d_S}_2 + \norm{\widehat{d}_{t-1} - d_S}_2\right) + 3\sqrt{\frac{\varepsilon}{\lambda \kappa}} + 24 \sqrt{\frac{B}{\kappa}} \varepsilon
\end{align*}
Now if $\lambda > \frac{1}{\sqrt{\kappa B} }$ and $\varepsilon < \frac{\beta^2 \delta^2}{48} \sqrt{\frac{\kappa}{B}}$ then it can be easily checked that the recurrence relation is the following.
$$
\norm{\widehat{d}_{t+1} - d_S}_2  \le \beta \left( \norm{\widehat{d}_{t} - d_S}_2 + \norm{\widehat{d}_{t-1} - d_S}_2\right) + \beta \delta 
$$

Now suppose $\beta < 1/3$. Then there are two cases to consider. First, if $\max\set{\norm{\widehat{d}_{t} - d_S}_2 , \norm{\widehat{d}_{t-1} - d_S}_2} < \delta$ then $\norm{\widehat{d}_{t+1} - d_S}_2 < \delta$ and for all subsequent $t' > t+1$ we also have $\norm{\widehat{d}_{t'} - d_S}_2 < \delta$. Otherwise, we have $\norm{\widehat{d}_{t+1} - d_S}_2  \le 2\beta \left( \norm{\widehat{d}_{t} - d_S}_2 + \norm{\widehat{d}_{t-1} - d_S}_2\right)$. Now following an argument very similar to the proof of Theorem~\eqref{thm:convergence-rpo} we can establish that $\norm{\widehat{d}_{t+1} - d_S}_2 \le \frac{2}{1-\gamma}r^t$ for $r = \beta \left(1 + \frac{1}{2}\sqrt{1 + \frac{2}{\beta}}\right)$. As $\beta < 1/3$ we have $r < \beta\left(1 + \frac{1}{2} + \frac{1}{\sqrt{2\beta}}\right) = \frac{3\beta}{2} + \sqrt{\frac{\beta}{2}} < 1$. Therefore, as long as $t \ge \ln \left( \frac{2}{\delta(1-\gamma)}\right) / \ln(1/r)$, we are guaranteed that $\norm{\widehat{d}_t - d_S}_2 \le \delta$.

Now we determine the sufficient conditions for ensuring $\beta < \frac{1}{3}$ and $\varepsilon < \frac{\beta^2 \delta^2}{48} \sqrt{\frac{\kappa}{B}}$. The proof of theorem~\eqref{thm:convergence-rpo} establishes
$$
\beta = \frac{\varepsilon_\theta + \alpha \gamma \sqrt{D} \varepsilon_\mu   }{\lambda \sqrt{\kappa}} + \frac{4 \gamma \varepsilon_\mu \alpha^2}{\sqrt{\kappa}}
$$
Therefore, if $\lambda > \frac{6(\varepsilon_\theta + \alpha \gamma \sqrt{D} \varepsilon_\mu)   }{\sqrt{\kappa B}}$ and $\varepsilon_\mu < \frac{\sqrt{\kappa}}{24 \gamma \alpha^2}$ then $\beta < 1/3$. Finally for the upper bound on $\varepsilon$ we need the number of samples $m_t = O\left( \frac{D^5 B \lambda^4 }{(1-\gamma)^2 c_1^4\varepsilon^2}\log \frac{DB\lambda t}{c_1 \varepsilon p}\right) = O\left( \frac{D^5 B^2 \lambda^4 }{(1-\gamma)^2 c_1^4\delta^4 \kappa}\log \frac{DB\lambda t}{c_1 \delta \kappa p}\right)$.
\end{proof}

\begin{lemma}\label{lem:apx-saddle-point}
    Let us define the saddle points $(d^\star, \nu^\star, g^\star, \omega^\star)$ and $(\widehat{d}, \widehat{\nu}, \widehat{g}, \widehat{\omega})$ as
    $$(d^\star, \nu^\star; g^\star, \omega^\star) \in \argmax_{d, \nu} \min_{g, \omega}\calL_t(d,\nu;g,\omega),\ \textrm{and} \ (\widehat{d}, \widehat{\nu}; \widehat{g}, \widehat{\omega}) \in \argmax_{d, \nu} \min_{g, \omega}\widehat{\calL}_t(d,\nu;g,\omega).$$ Suppose $\max \set{\norm{\Sigma_t^{-1} \nu^\star}_2, \norm{\Sigma_t^{-1} \widehat{\nu}}_2} \le \sqrt{B}$, and $\max\set{\norm{g^\star}_2, \norm{\widehat{g}}_2, \norm{\omega^\star}_2, \norm{\widehat{\omega}}_2} \le \frac{\lambda + 2D}{\min \set{1,\gamma^2 \cdot \underline{\sigma} }}$, and the number of samples $m_t \ge O\left( \frac{D^5 B \lambda^4 }{(1-\gamma)^2 c_1^4\varepsilon^2}\log \frac{DB\lambda }{c_1 \varepsilon \delta_0}\right)$ where $c_1 = \min\set{1, {\gamma^2} \cdot \sigma^\star_{\min}(\bm{\mu}_d \bm{\mu}_d^\top) }$. %\footnote{To be precise, we need $\frac{D }{m_t} \log \frac{64 D^2 B m_t}{c_1 \varepsilon \delta_0} \le \varepsilon^2$}. 
    Then with probability at least $1-\delta_0$ we have,
    $$
    \calL_t(d^\star, \nu^\star; g^\star, \omega^\star) - \calL_t(\widehat{d}, \widehat{\nu}; g^\star, \omega^\star) \le 2 \varepsilon.
    $$
\end{lemma}

\begin{proof}
    First, note that the expected value of the empirical Lagrangian equals $\calL(d,\nu; g, \omega)$.
    \begin{align*}
    &\E_{(s,a) \sim d_{\pi_t}}\left[ \widehat{\calL}_t(d,\nu; g, \omega) \right] \\
    &= \nu^\top \Sigma_t^{-1} \cdot \frac{1}{m_t} \sum_{j=1}^{m_t} \E\left[ \phi(s_j,a_j) \left( r_t(s_j,a_j)  + \gamma \cdot g(s'_j) - \phi(s_j,a_j)^\top \omega \right) \right] \nonumber \\
    &- \frac{\lambda}{2}\nu^\top \nu + \frac{1}{m_t} \sum_{j=1}^{m_t} \E\left[g(s^0_j)\right] + \left \langle d, \Phi w - B^\top g\right \rangle \\
    &= \nu^\top \Sigma_t^{-1} \cdot \frac{1}{m_t} \sum_{j=1}^{m_t} \E\left[ \phi(s_j,a_j) \left(\phi(s_j, a_j)^\top \theta_t + \sum_{s'} P(s' \mid s_j, a_j)g(s') - \phi(s_j, a_j)^\top \omega \right)\right]\\
     &- \frac{\lambda}{2}\nu^\top \nu + \frac{1}{m_t} \sum_{j=1}^{m_t} \sum_s \rho(s) g(s) + \left \langle d, \Phi w - B^\top g\right \rangle \\
     &= \nu^\top \Sigma_t^{-1} \cdot \frac{1}{m_t} \sum_{j=1}^{m_t} \E\left[ \phi(s_j,a_j) \left(\phi(s_j, a_j)^\top \theta_t + \sum_{s'} \phi(s_j, a_j)^\top \bm{\mu}_t(s') g(s') - \phi(s_j, a_j)^\top \omega \right)\right]\\
     &- \frac{\lambda}{2}\nu^\top \nu + \rho^\top g + \left \langle d, \Phi w - B^\top g\right \rangle \\
     &= \nu^\top \Sigma_t^{-1} \cdot \frac{1}{m_t} \sum_{j=1}^{m_t} \E\left[ \phi(s_j, a_j) \phi(s_j, a_j)^\top \right] \left( \theta_t + \gamma \cdot \bm{\mu}_t g - \omega \right) \\
     &-\frac{\lambda}{2}\nu^\top \nu + \rho^\top g + \left \langle d, \Phi w - B^\top g\right \rangle \\
     &= \nu^\top \left( \theta_t + \gamma \cdot \bm{\mu}_t g - \omega \right) -\frac{\lambda}{2}\nu^\top \nu + \rho^\top g + \left \langle d, \Phi w - B^\top g\right \rangle\\
     &= \calL_t(d,\nu; g, \omega)
    \end{align*}

    Since $\norm{g}_2 \le \frac{\lambda + 2D}{\min \set{1, \gamma^2 \underline{\sigma}}}$, we can apply the Chernoff-Hoeffding inequality and obtain the following bound.
    \begin{align}
        \Pr\left(\abs{\frac{1}{m_t}\sum_{j=1}^{m_t} g(s_j^0) - \rho^\top g} \ge \frac{\lambda + 2D}{\min \set{1, \gamma^2 \underline{\sigma}}}\sqrt{\frac{\log\left( {4}/{\delta_1}\right)}{m_t} }\right) \le \frac{\delta_1}{2} \label{eq:prob-bound-term-1}
    \end{align}
    Moreover, for any $j$ we have,
    \begin{align*}
    &\nu^\top \Sigma_t^{-1} \phi(s_j, a_j) \left( r_t(s_j, a_j) + \gamma \cdot g(s'_j) - \phi(s_j, a_j)^\top \omega \right) \\
    &\le \norm{\nu^\top \Sigma_t^{-1}}_2 \norm{\phi(s_j,a_j)}_2 \abs{r_t(s_j, a_j) + \gamma \cdot g(s'_j) - \phi(s_j,a_j)^\top \omega}\\
    &\le \sqrt{BD}\left(\abs{r_t(s_j, a_j)} + \gamma \cdot \abs{g(s'_j)} + \norm{\phi(s_j,a_j)}_2 \norm{\omega}_2\right)\\
    &\le \sqrt{BD} \left(1 + \gamma \cdot \frac{\lambda + 2D}{\min \set{1, \gamma^2 \underline{\sigma}}} + \sqrt{D}  \cdot \frac{\lambda + 2D}{\min \set{1, \gamma^2 \underline{\sigma}}} \right) = \underbrace{ \sqrt{BD} \left(1 + \frac{(\gamma + \sqrt{D})(\lambda + 2D)}{\min \set{1, \gamma^2 \underline{\sigma}}} \right)}_{:= H}
    \end{align*}
    Then we can apply Chernoff-Hoeffding inequality and obtain the following bound.
    \begin{align}
        &\Pr\left(\abs{\frac{1}{m_t}\sum_{j=1}^{m_t} \nu^\top \Sigma_t^{-1} \phi(s_j, a_j) \left( r_t(s_j, a_j) + \gamma \cdot g(s'_j) - \phi(s_j, a_j)^\top \omega \right) - \nu^\top \left( \theta_t + \gamma \cdot \bm{\mu}_t^\top g - \omega \right)} \right.\nonumber\\
        &\quad \left.\ge H\sqrt{\frac{\log\left( {4}/{\delta_1}\right)}{m_t} }\right) \le \frac{\delta_1}{2} \label{eq:prob-bound-term-2}
    \end{align}
    Using the bounds derived in equations~\eqref{eq:prob-bound-term-1} and \eqref{eq:prob-bound-term-2} we get the following bound.
    \begin{align}\label{eq:preliminary-bound}
        \Pr\left(\abs{\widehat{\calL}_t(d,\nu; g, \omega) - \calL_t(d,\nu;g,\omega)} \ge \left( H + \frac{\lambda + 2D}{\min \set{1, \gamma^2 \underline{\sigma}}}\right) \sqrt{\frac{\log\left( {4}/{\delta_1}\right)}{m_t} }\right) \le \delta_1
    \end{align}
    Note that the difference term $\abs{\widehat{\calL}_t(d,\nu; g, \omega) - \calL_t(d,\nu;g,\omega)}$ is independent of $d$ and the bound derived in \cref{eq:preliminary-bound} holds for any $d$. We now extend the bound for any $\nu$ and $\omega$. 
    
    We can assume that $\omega \in \Omega = \set{\omega: \norm{\omega}_2 \le \frac{\lambda + 2D}{\min \set{1, \gamma^2 \underline{\sigma}}}}$. By lemma 5.2 of \cite{Vershy10} there is an $\varepsilon$-net $\Omega_\varepsilon$ of size at most $\left(1 + \frac{2(\lambda + 2D)}{\min \set{1, \gamma^2 \underline{\sigma}} \cdot \varepsilon} \right)^D$, of the set $\Omega$ so that for any $\omega$ there exists $\omega'$ so that $\norm{\omega - \omega'}_2 \le \varepsilon$. 

    On the other hand, $\nu \in \calV = \set{\nu: \norm{\Sigma_t^{-1} \nu}_2 \le \sqrt{B}}$. In order to construct an $\varepsilon$-net of the set $\calV$, consider the set $\set{y : \norm{y} \le \sqrt{B}}$. There is an $\varepsilon$-net  $\calC_\varepsilon$ of this set of cardinality $\left(1 + \frac{2\sqrt{B}}{\varepsilon} \right)^D$. Now consider the following set $\calV_\varepsilon = \set{\Sigma_t y: y \in \calC_\varepsilon}$. Given any $\nu \in \calV$ we know there exists $\tilde{y} \in \calC_\varepsilon$ so that $\norm{\Sigma_t^{-1} \nu - \tilde{y}}_2 \le \varepsilon$. Equivalently, there exists $\tilde{\nu}$ so that $\norm{\Sigma_t^{-1} \left(\nu - \tilde{\nu}\right)}_2 \le \varepsilon$.

    Moreover, for a pair of $(\omega, \nu) \in \Omega_\varepsilon \times \calV_\varepsilon$, we will fix a choice of $g$ which is given as the maximizer of the following optimization problem.
    \begin{align}\label{eq:defn-g}
    g = g(\omega, \nu) \in \argmin_{g': \norm{g'}_2 \le G} \widehat{\calL}(d, \nu; g', \omega)
    \end{align}
    Note that, from the definition of the empirical Lagrangian~\eqref{eq:empirical-Lagrangian}, we can choose $g(\omega, \nu)$ to be $G$ times unit vector with support in the set $\set{s_j^0: j \in [m_t]}$ or $ \cup \set{s_j': j  \in [m_t]}$. In particular let $\calG = \set{G \cdot \one_s : s \in \set{s^0_j, s'_j},\ j \in [m_t]}$. Now,
    by union bound over $\abs{\Omega_\varepsilon} \times \abs{\calV_\varepsilon}$ tuples, we can extend \cref{eq:preliminary-bound} i.e. the following event holds with probability at least $1-\delta_0$ for any $\omega \in \Omega_\varepsilon$, $\nu \in \calV_\varepsilon$, and $g \in \calG$ defined above.
    \begin{align}\label{eq:prelim-concentraion-inequality}
\abs{\widehat{\calL}_t(d,\nu; g, \omega) - \calL_t(d,\nu;g,\omega)} \le \underbrace{\left( H + \frac{\lambda + 2D}{\min \set{1, \gamma^2 \underline{\sigma}}}\right) \sqrt{\frac{D}{m_t}}\sqrt{ \log\left(\frac{32DB(\lambda + 2D)m_t}{\min \set{1, \gamma^2 \underline{\sigma}} \cdot \varepsilon \cdot \delta_0}\right)}}_{:= f(\varepsilon)} 
    \end{align}

Now given any $\nu, \omega$ let us pick $\tilde{\omega} \in \Omega_\varepsilon$ and $\tilde{\nu} \in \calV_\varepsilon$ so that $\norm{\omega - \tilde{\omega}}_2 \le \varepsilon$ and $\norm{\Sigma_t^{-1}\left(\nu - \tilde{\nu} \right) }_2 \le \varepsilon$. Moreover, we define $g$  as $g = g(\omega, \nu)$  (as defined in \cref{eq:defn-g}). Then we have,
\begin{align*}
    \calL_t(d,\nu;g,\omega) - \widehat{\calL}_t(d,\nu; g, \omega) &\le \abs{\calL_t(d,\nu;g,\omega) - \calL_t(d,\tilde{\nu};{g},\tilde{\omega})} + \abs{{\calL}_t(d,\tilde{\nu};g,\tilde{\omega}) - \widehat{\calL}_t(d,\tilde{\nu};{g},\tilde{\omega})} \\
    &+ \abs{\widehat{\calL}_t(d,{\nu};g,{\omega}) - \widehat{\calL}_t(d,\tilde{\nu};{g},\tilde{\omega})}
\end{align*}
Since $g \in \calG$, the second term is bounded by $f(\varepsilon)$, by \cref{eq:prelim-concentraion-inequality}. Moreover, we can apply \Cref{lem:bound-diff-nu-omega} to bound the first and the third term as follows.

$$ \abs{\calL_t(d,\nu;g,\omega) - {\calL}_t(d,\tilde{\nu};g,\tilde{\omega})} \le \varepsilon  D \left(\sqrt{D} + \left(1 + \gamma \sqrt{D} \frac{\lambda + 2D}{\min \set{1, \gamma^2 \underline{\sigma}}} \right) \cdot  \frac{\lambda + 2D}{\min \set{1, \gamma^2 \underline{\sigma}}} + \lambda D \sqrt{B} \right) + \varepsilon  \left(D\sqrt{B} + \frac{1}{1-\gamma} \right)
    $$
    and
    $$
    \abs{\widehat{\calL}_t(d,\nu;g,\omega) - \widehat{\calL}_t(d,\tilde{\nu};g,\tilde{\omega})} \le \varepsilon \left( \sqrt{D} + \left(1 + \gamma \sqrt{D} \frac{\lambda + 2D}{\min \set{1, \gamma^2 \underline{\sigma}}} \right) \cdot  \frac{\lambda + 2D}{\min \set{1, \gamma^2 \underline{\sigma}}} + \lambda \cdot D^2 \sqrt{B}\right) + \varepsilon  \left(\sqrt{B} + \frac{1}{1-\gamma}\right)
    $$
    Suppose $m_t$ is chosen so that 
    \begin{equation}\label{defn:eps_1}
    \frac{D}{m_t} \log \frac{64 D^2 B m_t}{\min \set{1, \gamma^2 \underline{\sigma}} \varepsilon \delta_0} \le \varepsilon^2_1
    \end{equation} 
    then it can be easily seen that $f(\varepsilon_1) \le \left( H + \frac{\lambda + 2D}{\min \set{1, \gamma^2 \underline{\sigma}}}\right)  \varepsilon_1 \sqrt{2 \log \lambda}$. Then the following bound holds for any $\nu \in \calV$, $\omega \in \Omega$ and $g = g(\omega, \nu)$.
    \begin{align*}
            \abs{\calL_t(d,\nu;g,\omega) - \widehat{\calL}_t(d,\nu; g, \omega) } &\le 2\varepsilon_1  \left( H + \frac{\lambda + 2D}{\min \set{1, \gamma^2 \underline{\sigma}}}\right) \left(1 + \gamma \sqrt{D} \frac{\lambda + 2D}{\min \set{1, \gamma^2 \underline{\sigma}}} + \sqrt{2 \log \lambda} \right) \\
            &+ 2\varepsilon_1 \left(\lambda D^2 \sqrt{B} + D^{3/2}\sqrt{B} + \frac{1}{1-\gamma} \right)\\
            &\le 8\varepsilon_1 \sqrt{B} D^{3/2}\left( \frac{\lambda + 2D}{\min \set{1, \gamma^2 \underline{\sigma}}}\right)^2 + \frac{4\varepsilon\cdot \lambda D^2 \sqrt{B}}{1-\gamma}\\
            &\le \frac{16 \varepsilon_1 D^2 \sqrt{B}}{1-\gamma} \left( \frac{\lambda + 2D}{\min \set{1, \gamma^2 \underline{\sigma}}}\right)^2
    \end{align*}
    Now we can substitute $\varepsilon = \frac{16 \varepsilon_1 D^2 \sqrt{B}}{1-\gamma} \left( \frac{\lambda + 2D}{\min \set{1, \gamma^2 \underline{\sigma}}}\right)^2$ in \cref{defn:eps_1}, and
     apply \cref{lem:saddle-point-approximation} to obtain the required bound.
\end{proof}

\begin{lemma}\label{lem:bound-diff-nu-omega}
    Suppose, we are given $\nu, \tilde{\nu}$ such that $\norm{\Sigma^{-1}(\nu -\tilde{\nu}) }_2 \le \varepsilon_1$ and $\max\set{\norm{\Sigma^{-1}\nu}_2, \norm{\Sigma^{-1}\tilde{\nu}}_2}\le V$, and $\omega, \tilde{\omega}$ such that $\norm{\omega - \tilde{\omega}}_2 \le \varepsilon_2$, and $\max \set{\norm{\omega}_2, \norm{\tilde{\omega}}_2} \le W$. Then for any $g$ with $\norm{g}_2 \le G$ we have,
    $$
    \abs{\calL(d,\nu;g,\omega) - {\calL}(d,\tilde{\nu};g,\tilde{\omega})} \le \varepsilon_1 D \left(\sqrt{D} + \gamma \sqrt{D} G + W + \lambda D V \right) + \varepsilon_2 \left(DV + \frac{1}{1-\gamma} \right)
    $$
    and
    $$
    \abs{\widehat{\calL}(d,\nu;g,\omega) - \widehat{\calL}(d,\tilde{\nu};g,\tilde{\omega})} \le \varepsilon_1\left( \sqrt{D} + \gamma G + W + \lambda \cdot D^2 V\right) + \varepsilon_2 \left( V + \frac{1}{1-\gamma}\right)
    $$
\end{lemma}
\begin{proof}
    \begin{align*}
        \abs{\calL(d,\nu;g,\omega) - {\calL}(d,\tilde{\nu};g,\tilde{\omega})} \le \abs{\calL(d,\nu;g,\omega) - {\calL}(d,\tilde{\nu};g,{\omega})} + \abs{\calL(d,\tilde{\nu};g,\omega) - {\calL}(d,\tilde{\nu};g,\tilde{\omega})}
    \end{align*}
    Using the definition of $\calL(\cdot)$~\cref{defn:Lagrangian} we obtain the following bound.
    \begin{align*}
        &\abs{\calL(d,\nu;g,\omega) - {\calL}(d,\tilde{\nu};g,\tilde{\omega})} \le \abs{(\nu -\tilde{\nu})^\top\left( \theta_t + \gamma \cdot \bm{\mu}_t^\top g - \omega\right) - \frac{\lambda}{2}\left( \norm{\nu}_2^2 - \norm{\tilde{\nu}}_2^2\right)}\\
        &+ \abs{(\tilde{\omega}-\omega)^\top \tilde{\nu} + d^\top \Phi (\omega -\tilde{\omega})}\\
        &\le \norm{\nu - \tilde{\nu}}_2 \norm{\theta_t + \gamma \cdot \bm{\mu}_t^\top g - \omega - \frac{\lambda}{2}\left( \nu + \tilde{\nu}\right) }_2 + \norm{\omega - \tilde{\omega}}_2 \norm{- \tilde{\nu} + \Phi^\top d}_2\\
        &\le D \varepsilon_1 \left(\sqrt{D} + \gamma \sqrt{D} G + W + \lambda D V \right) + \varepsilon_2 \left(DV + \frac{1}{1-\gamma} \right)
    \end{align*}
The last inequality uses (a) $\norm{\nu - \tilde{\nu}}_2 = \norm{\Sigma \Sigma^{-1} (\nu - \tilde{\nu})}_2 \le \varepsilon_1 \norm{\Sigma}_2 \le D \varepsilon_1$,  (b) for a linear MDP, $\norm{\theta_t}_2 \le \sqrt{D}$ and $\norm{\bm{\mu}_t}_2 \le \sqrt{D}$. Second, $\norm{\Phi^\top d}_2 = \norm{\sum_{s,a} d(s,a) \phi(s,a)}_2 \le \sum_{s,a} d(s,a) \norm{\phi(s,a)}_2 \le \frac{1}{1-\gamma}$.

Now let us write $\phi_j = \phi(s_j, a_j)$, and $r_j = r(s_j, a_j)$. Then from the definition of empirical Lagrangian~\eqref{eq:empirical-Lagrangian} we obtain the following bound.
  \begin{align*}
        &\abs{\widehat{\calL}(d,\nu;g,\omega) - \widehat{\calL}(d,\tilde{\nu};g,\tilde{\omega})} \le \abs{\widehat{\calL}(d,\nu;g,\omega) - \widehat{\calL}(d,\tilde{\nu};g,{\omega})} + \abs{\widehat{\calL}(d,\tilde{\nu};g,\omega) - \widehat{\calL}(d,\tilde{\nu};g,\tilde{\omega})}\\
        &\le \abs{(\nu - \tilde{\nu})^\top \Sigma^{-1} \cdot \frac{1}{m} \sum_{j=1}^m \phi_j\left( r_j + \gamma \cdot g(s'_j) - \phi_j^\top \omega \right)  - \frac{\lambda}{2}\left(\norm{\nu}_2^2 - \norm{\tilde{\nu}}_2^2\right) }\\
        &+ \abs{  \tilde{\nu}^\top \Sigma^{-1} \cdot \frac{1}{m} \sum_{j=1}^m \phi_j  \phi_j^\top \left(\tilde{\omega} - \omega \right)} + \abs{\left \langle d, \Phi(\omega - \tilde{\omega})\right \rangle}\\
        &\le \norm{(\nu - \tilde{\nu})^\top \Sigma^{-1}}_2 \left(\norm{\frac{1}{m} \sum_{j=1}^m \phi_j   \left(r_j + \gamma \cdot g(s'_j) - \phi_j^\top \omega \right)}_2 + \frac{\lambda}{2} \norm{\Sigma(\nu + \tilde{\nu})}_2\right)\\
        &+ \left( \norm{\tilde{\nu}^\top \Sigma^{-1}}_2 \norm{\frac{1}{m}\sum_{j=1}^m \phi_j \phi_j^\top }_2 + \norm{\Phi^\top d}_2 \right) \norm{\omega - \tilde{\omega}}_2\\
        &\le \varepsilon_1\left( \sqrt{D} + \gamma G + W + \lambda \cdot D^2 V\right) + \varepsilon_2 \left( V + \frac{1}{1-\gamma}\right)
    \end{align*}
The last inequality uses the following observations -- (a) $\abs{r_j} = \abs{\phi_j^\top \theta_t} \le \norm{\phi_j}_2 \norm{\theta_t}_2 \le \sqrt{D}$, (b) $\norm{\Phi^\top d}_2 = \norm{\sum_{s,a} d(s,a) \phi(s,a)}_2 \le \sum_{s,a} d(s,a) \norm{\phi(s,a)}_2 \le \frac{1}{1-\gamma}$, and (c)$\norm{\nu}_2 \le \norm{\Sigma}_2 \norm{\Sigma^{-1}\nu}_2 \le DV$.
\end{proof}

\begin{lemma}\label{lem:saddle-point-approximation}
    Suppose $$(d^\star, \nu^\star; g^\star, \omega^\star) \in \argmax_{d, \nu} \min_{g, \omega}\calL(d,\nu;g,\omega)$$ and $$(\widehat{d}, \widehat{\nu}; \widehat{g}, \widehat{\omega}) \in \argmax_{d, \nu} \min_{g, \omega}\widehat{\calL}(d,\nu;g,\omega).$$ Moreover, given any $d,\nu, \omega$ suppose the following inequality holds.
    $$
    \abs{\calL(d,\nu;\widehat{g},\omega) - \widehat{\calL}(d,\nu; \widehat{g},\omega)} \le \varepsilon \quad \textrm{ where } \widehat{g} \in \argmin_g \widehat{\calL}(d,\nu; g, \omega)
    $$
    Then we have,
    $$
    \calL(d^\star, \nu^\star; g^\star, \omega^\star) - \calL(\widehat{d}, \widehat{\nu}; g^\star, \omega^\star) \le 2 \varepsilon.
    $$
\end{lemma}
\begin{proof}
Given $d,\nu$ let us define $\widehat{g}(d,\nu)$ and $\widehat{\omega}(d,\nu)$ as follows.
$$
\left(\widehat{g}(d,\nu), \widehat{\omega}(d,\nu) \right) \in \argmin_{g,\omega} \widehat{\calL}(d,\nu; g, \omega)
$$
Moreover, let $\Tilde{g} = \argmin_g \calL(\hat{d},\hat{\nu}; g,\omega^\star)$.
    \begin{align*}
        &\calL(d^\star, \nu^\star; g^\star, \omega^\star) - \calL(\widehat{d}, \widehat{\nu}; g^\star, \omega^\star)\\
        &= \underbrace{\calL(d^\star, \nu^\star; g^\star, \omega^\star) - \calL(d^\star, \nu^\star; \widehat{g}(d^\star, \nu^\star), \widehat{\omega}(d^\star, \nu^\star))}_{:= T_1} \\&+ \underbrace{\calL(d^\star, \nu^\star; \widehat{g}(d^\star, \nu^\star), \widehat{\omega}(d^\star, \nu^\star)) - \widehat{\calL}(d^\star, \nu^\star; \widehat{g}(d^\star, \nu^\star), \widehat{\omega}(d^\star, \nu^\star))}_{:= T_2}\\
        &+\underbrace{\widehat{\calL}(d^\star, \nu^\star; \widehat{g}(d^\star, \nu^\star), \widehat{\omega}(d^\star, \nu^\star)) - \widehat{\calL}(\widehat{d}, \widehat{\nu}; \widehat{g}, \widehat{\omega})}_{:= T_3}\\
        &+\underbrace{\widehat{\calL}(\widehat{d}, \widehat{\nu}; \widehat{g}, \widehat{\omega}) - \widehat{\calL}(\widehat{d}, \widehat{\nu}; \tilde{g}, \omega^\star)}_{:= T_4}
        + \underbrace{\widehat{\calL}(\widehat{d}, \widehat{\nu}; \tilde{g}, \omega^\star) - \calL(\widehat{d}, \widehat{\nu}; \tilde{g}, \omega^\star)}_{:= T_5} + \underbrace{\calL(\widehat{d}, \widehat{\nu}; \tilde{g}, \omega^\star) - \calL(\widehat{d}, \widehat{\nu}; g^\star, \omega^\star)}_{:= T_6}
    \end{align*}
    Given $(d^\star, \nu^\star)$, $(g^\star, \omega^\star)$ is the optimal minimizer, and $T_1 \le 0$. Given $d^\star, \nu^\star, \widehat{\omega}_\star$, $\widehat{g}(d^\star, \nu^\star)$ minimizes the function $\widehat{\calL}(d^\star, \nu^\star; \cdot, \widehat{\omega}_\star)$, and $T_2 \le \varepsilon$.
    Since $(\widehat{d}, \widehat{\nu}; \widehat{g}, \widehat{\omega})$ is a saddle point, $(\widehat{d}, \widehat{\nu})$ is the maximizer of the following optimization problem $\max_{d,\nu} \widehat{\calL}(d,\nu; \widehat{g}(d,\nu), \widehat{\omega}(d,\nu))$. This implies that $T_3 \le 0$.

    Since $(\widehat{g}, \widehat{\omega})$ minimizes $\widehat{\calL}(\widehat{d}, \widehat{\nu}; \cdot)$ the term $T_4 \le 0$. By a similar argument as $T_2$, we have $T_5 \le 0$. Finally, given $\widehat{d}, \widehat{\nu}, \omega^\star$, $\tilde{g}$ is the minimizer of $\calL(\widehat{d}, \widehat{\nu}; \cdot, \omega^\star)$ and $T_6 \le 0$. Combining the six inequalities we get the desired result.
\end{proof}

\begin{lemma}
    Let $g \in \argmin_{g'} \norm{g'}_2^2\ \textrm{ s.t. } v^\top g' = c$ and $\tilde{g} \in \argmin_{g'} \norm{g'}_2^2\ \textrm{ s.t. } \tilde{v}^\top g' = \tilde{c}$. Then $\norm{g - \tilde{g}}_2 \le \abs{c - \tilde{c}} + 2 \abs{\tilde{c}} \frac{\norm{\tilde{v}  - {v}}_2}{\norm{\tilde{v}}_2}$.
\end{lemma}
\begin{proof}
    The solution to the optimization problem $\min_{g'} \norm{g'}_2^2\ \textrm{ s.t. } v^\top g' = c$ is $\frac{c}{\norm{v}_2} \cdot v$. Therefore,
    \begin{align*}
        \norm{g - \tilde{g}}_2 &= \norm{\frac{c}{\norm{v}_2} \cdot v - \frac{\tilde{c}}{\norm{\tilde{v}}_2} \cdot \tilde{v}}_2 \le \abs{c - \tilde{c}} \norm{v / \norm{v}_2}_2 + \abs{\tilde{c}}\norm{\frac{v}{\norm{v}_2} - \frac{\tilde{v}}{\norm{\tilde{v}}_2}}_2\\
        &\le \abs{c - \tilde{c}} + \abs{\tilde{c}} \frac{\norm{v \cdot \norm{\tilde{v}}_2 - \tilde{v} \cdot \norm{v}_2 }_2}{\norm{v}_2 \norm{\tilde{v}}_2}\\
        &\le \abs{c - \tilde{c}} + \abs{\tilde{c}} \frac{\norm{v \left( \norm{\tilde{v}}_2 - \norm{v}_2 \right) + \left(v - \tilde{v} \right) \cdot \norm{v}_2 }_2}{\norm{v}_2 \norm{\tilde{v}}_2}\\
        &\le \abs{c - \tilde{c}} + \abs{\tilde{c}} \left(\frac{\abs{\norm{\tilde{v}}_2 - \norm{v}_2}}{\norm{\tilde{v}}_2} + \frac{\norm{\tilde{v}  - {v}}_2}{\norm{\tilde{v}}_2}\right)\\
        &\le \abs{c - \tilde{c}} + 2 \abs{\tilde{c}} \frac{\norm{\tilde{v}  - {v}}_2}{\norm{\tilde{v}}_2}
    \end{align*}
\end{proof}

\begin{lemma}\label{lem:bound-optima-quadpro}
    Let $A$ be a positive semidefinite matrix, $Q$ be a matrix with full row-rank and $x^\star$ be an optimal solution of the following optimization problem.
    \begin{align*}
        \min_{x \in \R^n}&\ x^\top A x + b^\top x\\
        \textrm{s.t. }&\ Q x \ge d
    \end{align*}
    Then we have $$\norm{x^\star}_2 \le \frac{\norm{b}_2}{2 \sigma^\star_{\min}(A)} + \frac{\norm{A}_2 \norm{d}_2}{\sigma^\star_{\min}(A) \lambda_{\min}(QQ^\top)},$$ where $\sigma^\star_{\min}(A)$ is the smallest positive eigenvalue of the matrix $A$.
\end{lemma}

\begin{proof}
Let us write $q_i$ to denote the rows of the matrix $Q$. We will say that constraint $i$ is active for $x^\star$ if $q_i^\top x^\star = d_i$. Let $C$ be the set of rows that correspond to active constraints for $x^\star$. If ${C} = \emptyset$ then $x^\star$ coincides with the optimal solution of the unconstrained problem i.e. $2A x^\star + b = 0$. In that case, $-\frac{1}{2} A^{\dagger} b$ is the solution of the minimum norm. If $A = U\Sigma U^\top$ is the singular value decomposition of $A$ then $A^\dagger = U\Sigma^{+} U^\top$ where $\Sigma^+$ is obtained by inverting all non-zero eigenvalues of $A$ and leaving the zero eignevalues as they are. 
This gives us $\norm{x^\star}_2 \le \frac{1}{2}  \norm{A^\dagger}_2 \norm{b}_2 \le \frac{\norm{b}_2}{2 \sigma^\star_{\min} (A)}$ 
where $\sigma^{\star}_{\min}(A)$ is the smallest positive eigenvalue of $A$.

Now we consider the case where $\abs{C} = r \le n$. Then the optimal solution lies in the subspace $\Pi = \set{x : Q_C x = d_C}$ where $Q_C$ (resp. $d_C$) is the submatrix with the rows of the matrix $Q$ (resp. vector $d$) indexed by the set $C$. We can also assume that $Q_C$ has full row rank, otherwise, we can eliminate some of the rows, and yet the subspace $\Pi$ will remain the same. Any element of the subspace $\Pi$ is given as $x = Q_C^\dagger d_C + [\Identity - Q_C^\dagger Q_C]w$ for arbitrary $w$. Substituting this value of $x$ we get the following unconstrained problem.
\begin{align*}
    \min_{w \in \R^n} w^\top (\Identity -  Q_C^\dagger Q_C)^\top A (\Identity - Q_C^\dagger Q_C) w + \left(b^\top + 2d_C^\top Q_C^{\dagger^\top} A \right) (I - Q_C^\dagger Q_C)  w
\end{align*}
At an optimal solution $w^\star$ we have the following equality.
$$
2A(\Identity - Q_C^\dagger Q_C) w^\star + (b + 2A^\top Q_C^\dagger d_C) = 0
$$
The minimum norm solution to this problem is the following.
$$
w^\star = - \frac{1}{2}\left(A (\Identity - Q_C^\dagger Q_C) \right)^\dagger (b + 2A^\top Q_C^\dagger d_C)
$$
and
\begin{equation}\label{eq:bound-w-star}
\norm{w^\star}_2 \le \frac{1}{2} \norm{\left( A(\Identity - Q_C^\dagger Q_C) \right)^\dagger }_2 \left(\norm{b}_2 + 2 \norm{A}_2 \norm{Q_C^\dagger}_2 \norm{d_C}_2 \right)
\end{equation}
We first claim that $ \norm{\left( A(\Identity - Q_C^\dagger Q_C) \right)^\dagger }_2 \le \frac{1}{\sigma^\star_{\min}(A)}$. Let $Q_C = U^1 \Sigma^1 V^{1^\top}$. Then $\Identity - Q_C^\dagger Q_C = \Identity - V^1 V^{1^\top}$. Since the columns of $V^1$ are orthonormal the eigenvalues of $\Identity - Q_C^\dagger Q_C$ are either $0$ and $1$. Moreover, we can assume that not all eigenvalues are zero, otherwise, the claim already holds. Now given two matrices $A$ and $B$ the smallest non-zero eigenvalue of $AB$ i.e. $\sigma^{\star}_{\min}(AB)$ is bounded from below by $\sigma^{\star}_{\min}(A) \times \sigma^{\star}_{\min}(B)$ since
$$
\sigma^{\star}_{\min}(AB) = \min_{x : ABx \neq 0} \frac{\norm{ABx}_2}{\norm{x}_2} = \min_{x : ABx \neq 0} \frac{\norm{ABx}_2}{\norm{Bx}_2} \frac{\norm{Bx}_2}{\norm{x}_2} \ge \sigma^{\star}_{\min}(A) \sigma^{\star}_{\min}(B).
$$
Therefore,
$$
\norm{\left( A(\Identity - Q_C^\dagger Q_C) \right)^\dagger }_2 \le \frac{1}{\sigma^\star_{\min}\left( A(\Identity - Q_C^\dagger Q_C) \right)} \le \frac{1}{\sigma^\star_{\min}(A) \sigma^\star_{\min}(\Identity - Q_C^\dagger Q_C)} = \frac{1}{\sigma^\star_{\min}(A)}
$$
Substituting the bound above in \cref{eq:bound-w-star} we obtain the following bound.
$$
\norm{w^\star}_2 \le \frac{1}{2 \sigma^\star_{\min}(A)} \left( \norm{b}_2 + 2 \norm{A}_2 \norm{Q_C^\dagger}_2 \norm{d}_2\right)
$$
As $Q_C$ has full row rank we can write $Q_C = UDV^\top$ for some invertible matrix $D \in \R^{r \times r}$. Then $Q_C^\dagger = V D^{-1} U^\top$ and $\norm{Q_C^\dagger}_2$ is inverse of the smallest singular value of $Q_C$ i.e. $1/\lambda_{\min}(Q_C Q_C^\top)$. Now recall that the matrix $Q_C$ was formed by removing rows of the matrix $Q$. Therefore, given any $x \in \R^r$ we can choose $y = [x\ 0_{n-r}]$ so that $QQ^\top y = Q_C Q_C^\top x$. This means that $\min_{x \neq 0} \frac{\norm{Q_CQ_C^\top x}_2}{\norm{x}_2} \ge \min_{y \neq 0} \frac{\norm{Q Q^\top y}_2}{\norm{y}_2} = \lambda_{\min}(QQ^\top)$. Therefore, we can upper bound $\norm{Q_C^\dagger}_2$ by $1/\lambda_{\min}(QQ^\top)$ and get the desired bound.
\end{proof}
\section{Missing Proofs from Section \texorpdfstring{\ref{sec:finite-sample}}{} }

\subsection{Equivalent Expression of the Lagrangian}\label{subsec:derivation-lagrangian}
\begin{align}
    \calL(d,\nu; g, \omega) &= \nu^\top \left( \theta_t  + \gamma \cdot \bm{\mu}^\top_t g - \omega \right) - \frac{\lambda}{2} \nu^\top \nu + \left \langle g, \rho \right \rangle + \left \langle d, \Phi w - B^\top g\right \rangle \nonumber\\
    &= \nu^\top \Sigma_t^{-1} \Sigma_t \left( \theta_t  + \gamma \cdot \bm{\mu}_t^\top g - \omega \right) - \frac{\lambda}{2} \nu^\top \nu + \left \langle g, \rho \right \rangle + \left \langle d, \Phi w - B^\top g\right \rangle\nonumber \\
    &= \nu^\top \Sigma_t^{-1} \E_{(s,a) \sim \pi_t}\left[ \phi(s,a) \phi(s,a)^\top \left( \theta_t  + \gamma \cdot \bm{\mu}_t\nu - \omega \right) \right] - \frac{\lambda}{2} \nu^\top \nu + \E_{s^0 \sim \rho}\left[ g(s^0) \right] + \left \langle d, \Phi w - B^\top g\right \rangle\nonumber \\
    &= \nu^\top \Sigma_t^{-1} \E_{(s,a) \sim \pi_t}\left[ \phi(s,a) r_t(s,a)  + \gamma \cdot g^\top P_t(\cdot \mid s,a) - \phi(s,a) \phi(s,a)^\top \omega \right] \nonumber \\&- \frac{\lambda}{2} \nu^\top \nu + \E_{s^0 \sim \rho}\left[ g(s^0) \right] + \left \langle d, \Phi w - B^\top g\right \rangle\nonumber \\
    &= \nu^\top \Sigma_t^{-1} \E_{(s,a) \sim \pi_t, s' \sim P_t(\cdot \mid s,a)}\left[ \phi(s,a) r_t(s,a)  + \gamma \cdot g(s') - \phi(s,a) \phi(s,a)^\top \omega \right] \nonumber \\&- \frac{\lambda}{2} \nu^\top \nu + \E_{s^0 \sim \rho}\left[ g(s^0) \right] + \left \langle d, \Phi w - B^\top g\right \rangle\nonumber 
\end{align}
\subsection{Proof of \texorpdfstring{\Cref{thm:convergence-offline-primal-dual}}{} }

\textbf{Definition of Regret.} 
Given $\omega$ let us define the policy $\pi$ as,
$$
\pi(a \mid s) = \frac{\exp(\phi(s,a)^\top \omega)}{\sum_b \exp(\phi(s,b)^\top \omega )}
$$
Moreover, given policy $\pi$ we can define $d^{\pi,\nu}$ and $g^{\pi,\omega}$ as follows.
\begin{equation}\label{defn:d-pi-nu}
d^{\pi,\nu}(s,a) = \pi(a \mid s) \cdot \left(\rho(s) + \gamma \cdot \bm{\mu}_t(s)^\top \nu \right) 
\end{equation}
% Since the measure $\bm{\mu}_t$ is unknown, we will use the following estimate of $d^{\pi,\nu}$. Given a 
% Note that,
% \begin{align*}
%     \E\left[d^{\pi,\nu}(\tilde{s},\tilde{a}) \right] &= \E\left[ \pi(\tilde{a} \mid \tilde{s}) \cdot \left( \one\set{\tilde{s} = s^0} + \nu^\top \Sigma_t^{-1} \phi(s,a) \one\set{\tilde{s} = s'} \right) \right]\\
%     &= \pi(\tilde{a} \mid \tilde{s}) \cdot \rho(\tilde{s}) + \sum_{s,a} \tilde{d}_t(s,a) \pi(\tilde{a} \mid \tilde{s}) \nu^\top \Sigma_t^{-1} \phi(s,a) P(\tilde{s} \mid s,a)\\
%     &= \pi(\tilde{a} \mid \tilde{s}) \cdot \rho(\tilde{s}) + \sum_{s,a} \tilde{d}_t(s,a) \pi(\tilde{a} \mid \tilde{s}) \nu^\top \Sigma_t^{-1} \phi(s,a) \phi(s,a)^\top \bm{\mu}_t(\tilde{s})\\
%     &= \pi(\tilde{a} \mid \tilde{s}) \cdot \rho(\tilde{s}) + \pi(\tilde{a} \mid \tilde{s}) \nu^\top \Sigma_t^{-1} \sum_{s,a} \tilde{d}_t(s,a) \phi(s,a) \phi(s,a)^\top \bm{\mu}_t(\tilde{s})\\
%     &= \pi(\tilde{a} \mid \tilde{s}) \cdot \rho(\tilde{s}) + \pi(\tilde{a} \mid \tilde{s}) \nu^\top \bm{\mu}_t(\tilde{s}) 
% \end{align*}
$$
g^{\pi,\omega}(s) = \sum_a \pi(a \mid s) \phi(s,a)^\top \omega
$$
Let us define the function $f(\pi,\nu,\omega) = \calL_t(d^{\pi,\nu}, \nu, g^{\pi,\omega}, \omega)$. We now define the following notion of regret.
\begin{equation}
    \calR(\nu^\star, \pi^\star, \omega^\star_{1:T_{\inner} }) = \frac{1}{T_{\inner}} \sum_{\ell=1}^{T_{\inner}} f(\nu^\star, \pi^\star, \omega_{\ell-1} ) - f(\nu_{\ell}, \pi_{\ell-1}, \omega^\star_{\ell-1})
\end{equation}

\begin{proof}
    
Given policies $\{\pi_{\ell-1}\}_{\ell=1}^{T_\inner}$ let us define $d^{\tilde{\pi}} = \frac{1}{T_\inner} \sum_{\ell=1}^{T_\inner} d^{\pi_{\ell-1}}$ and $\tilde{\nu} = \Phi^\top d^{\tilde{\pi}}$. Then $\left \langle \tilde{\nu}, \theta \right \rangle = \frac{1}{T_\inner} \sum_{\ell=1}^{T_\inner} \left \langle d^{\pi_{\ell-1}}, \Phi \theta \right \rangle = \frac{1}{T_\inner} \sum_{\ell=1}^{T_\inner} \left \langle V^{\pi_{\ell-1}}, \rho\right \rangle$. Now using \Cref{lem:offline-primal-dual} we obtain the following bound.
\begin{align}
    \left \langle \tilde{\nu}, \theta \right \rangle - \frac{\lambda}{2} \frac{1}{T_\inner} \sum_{\ell=1}^{T_\inner} \nu_\ell^\top \nu_\ell = \frac{1}{T_\inner} \sum_{\ell=1}^{T_\inner} \left \langle V^{\pi_{\ell - 1}}, \rho\right \rangle - \frac{\lambda}{2} \nu_\ell^\top \nu_\ell = \nu^{\star^\top} \theta - \frac{\lambda}{2} \norm{\nu^\star}_2^2 - \calR(\nu^\star, \pi^\star, \omega^\star_{1:T_\inner}) \label{eq:temp-regret}
\end{align}
Let us also define the policy $\pi^\dagger$ as follows $\pi^\dagger(a \mid s) = \frac{d^{\tilde{\pi}}(s,a)}{\sum_b d^{\tilde{\pi}}(s,b)}$. Then $\phi^\top d^{\pi^\dagger} = \Phi^\top d^{\tilde{\pi}} = \tilde{\nu}$. We now apply \Cref{lem:offline-primal-dual} with policy $\pi^\dagger$.
\begin{align*}
    \calR(\tilde{\nu}, \pi^\dagger, \tilde{\omega}_{1:T_\inner}) &= \tilde{\nu}^\top \theta - \frac{\lambda}{2} \norm{\tilde{\nu}}_2^2 - \frac{1}{T_\inner} \sum_{\ell=1}^{T_\inner} \left(\left \langle V^{\pi_{\ell-1}}, \rho\right \rangle - \frac{\lambda}{2} \norm{\nu_\ell}_2^2 \right)\\
    &= - \frac{\lambda}{2} \norm{\tilde{\nu}}_2^2 + \frac{\lambda}{2} \frac{1}{T_\inner}\sum_{\ell=1}^{T_\inner} \norm{\nu_\ell }_2^2 
\end{align*}
Substituting the value of $\frac{\lambda}{2} \frac{1}{T_\inner}\sum_{\ell=1}^{T_\inner} \norm{\nu_\ell}_2^2 $ in \cref{eq:temp-regret} we obtain the following identity.
\begin{align}
    \left \langle \tilde{\nu}, \theta \right \rangle- \frac{\lambda}{2} \norm{\tilde{\nu}}_2^2 = \nu^{\star^\top} \theta - \frac{\lambda}{2} \norm{\nu^\star}_2^2 - \calR(\nu^\star, \pi^\star, \omega^\star_{1:T_\inner}) + \calR(\tilde{\nu}, \pi^\dagger, \tilde{\omega}_{1:T_\inner}) \label{eq:bound-temp-regret-2}
\end{align}
It can be easily verified that with the choices of $\eta_\omega$ and $\eta_\pi$ in algorithm~\eqref{alg:offline-primal-dual}, \cref{lem:regret-upper-bound} provides the following upper bound on the regret.
$$
\calR(\nu^\star, \pi^\star, \omega^\star_{1:T_\inner}) \le   2\sqrt{\frac{D^2 B}{K} \left( B + \frac{1}{(1-\gamma)^2}\right)} + \frac{4D}{1-\gamma} \sqrt{\frac{\log A}{T_\inner}}
$$
If $K \ge \frac{144 D^2 B}{\varepsilon^2}\left(B + (1-\gamma)^{-2}\right)$ and $T \ge \frac{576D^2}{\varepsilon^2} \cdot \frac{\log A}{(1-\gamma)^2}$, then each term above is bounded by $\varepsilon / 6$ and the regret is bounded by $\varepsilon / 3$. Moreover, the result of \Cref{lem:regret-upper-bound} applies for any target policy, and in particular for $\pi^\dagger$, and so $\calR(\tilde{\nu}, \pi^\dagger, \tilde{\omega}_{1:T_\inner}) \le \varepsilon / 3$. Then by \cref{eq:bound-temp-regret-2} we have the following inequality.
$$
 \left \langle \tilde{\nu}, \theta \right \rangle- \frac{\lambda}{2} \norm{\tilde{\nu}}_2^2 \ge \nu^{\star^\top} \theta - \frac{\lambda}{2} \norm{\nu^\star}_2^2 - \varepsilon
$$
\end{proof}

\begin{lemma}\label{lem:offline-primal-dual}
     Given a policy $\pi^\star$ let us define $\nu^\star = \Phi^\top d^{\pi^\star}$, and let $\omega_\ell^\star = \theta_t + \gamma \cdot \bm{\mu}_t^\top V^{\pi_\ell}$. Then we have, $$\calR(\nu^\star, \pi^\star, \omega^\star_{1:T_{\inner}}) = \frac{1}{T_{\inner}} \sum_{\ell=1}^{T_{\inner}} \left( \nu^{\star^\top} \theta - \frac{\lambda}{2} \norm{\nu^\star}_2^2 -  \left \langle V^{\pi_{\ell-1}}, \rho \right \rangle + \frac{\lambda}{2} \norm{\nu_\ell}_2^2 \right).$$
\end{lemma}
\begin{proof}
\begin{align*}
  f(\nu^\star, \pi^\star, \omega_\ell) &= \calL(d^{\pi^\star, \nu^\star}, \nu^\star, g^{\pi^\star, \omega_\ell}, \omega_\ell)\\
    &= \nu^{\star^\top} \theta_t - \frac{\lambda}{2} \norm{\nu^\star}_2^2 + \left\langle g^{\pi^\star, \omega_\ell}, \rho + \gamma \cdot \bm{\mu}_t \nu^\star - B d^{\pi^\star, \nu^\star} \right \rangle + \left \langle \omega_\ell, \Phi^\top d^{\pi^\star, \nu^\star} - \nu^\star \right \rangle\\
    &= \nu^{\star^\top} \theta_t - \frac{\lambda}{2} \norm{\nu^\star}_2^2 
\end{align*}
The last equality uses two observations. First, from the definition of $d^{\pi^\star, \nu^\star}$ in \cref{defn:d-pi-nu}, we have $Bd^{\pi^\star, \nu^\star} = \rho + \gamma \cdot \bm{\mu}_t \nu^\star$, and the third term vanishes. Moreover, from the definition of $\nu^\star$ we can show that $d^{\pi^\star, \nu^\star}(s,a) = d^{\pi^\star}(s,a)$, and the fourth term vanishes.
\begin{align*}
d^{\pi^\star, \nu^\star}(s,a) &= \pi(a \mid s) \cdot \left( \rho(s) + \gamma \cdot \bm{\mu}_t(s)^\top \sum_{s',b} \phi(s',b) d^{\pi^\star}(s',b)\right)\\
&= \pi(a \mid s) \cdot \left( \rho(s) + \gamma \cdot \sum_{s',b} P(s \mid s',b) d^{\pi^\star}(s',b)\right)\\
&= \pi(a \mid s) \cdot d^{\pi^\star}(s) = d^{\pi^\star}(s,a)
\end{align*}
\begin{align}
    f(\nu_{\ell}, \pi_{\ell-1}, \omega^\star_{\ell-1})  &= \calL(d^{\pi_{\ell-1}, \nu_{\ell}}, \nu_{\ell}, g^{\pi_{\ell-1}, \omega^\star_{\ell-1} }, \omega^\star_{\ell-1})\nonumber \\
    &= \nu_{\ell}^\top \left( \theta_t  + \gamma \cdot \bm{\mu}_t^\top g^{\pi_{\ell-1}, \omega^\star_{\ell-1}} - \omega^\star_{\ell-1} \right) - \frac{\lambda}{2} \nu_{\ell}^\top \nu_{\ell} + \left \langle g^{\pi_{\ell-1}, \omega^\star_{\ell-1}}, \rho \right \rangle + \left \langle d^{\pi_{\ell-1}, \nu_{\ell}}, \Phi \omega^\star_{\ell-1} - B^\top g^{\pi_{\ell-1}, \omega^\star_{\ell-1}} \right \rangle \label{eq:f_nu_omega_t}
\end{align}
From the definition of $\omega_{\ell-1}^\star$ we have,
\begin{align*}
\phi(s,a)^\top \omega_{\ell-1}^\star &= r_t(s,a) + \gamma \phi(s,a)^\top \bm{\mu}_t^\top V^{\pi_{\ell-1}}\\
&= r_t(s,a) + \gamma \sum_{s'} P(s' \mid s,a) V^{\pi_{\ell-1}}(s') = Q^{\pi_{\ell-1}}(s,a).
\end{align*}
This also implies $g^{\pi_{\ell-1}, \omega_{\ell-1}^\star}(s) = \sum_a \pi_{\ell-1}(a \mid s) \phi(s,a)^\top \omega_\ell^\star = V^{\pi_{\ell-1}}(s)$. Therefore, the first and the fourth term in \cref{eq:f_nu_omega_t} vanish and we obtain the following identity.
$$
f(\nu_{\ell}, \pi_{\ell-1}, \omega^\star_{\ell-1}) = \left \langle V^{\pi_{\ell-1}}, \rho \right \rangle - \frac{\lambda}{2} \nu_{\ell}^\top \nu_{\ell}
$$

Now the result follows from the definition of the regret $\calR(\nu^\star, \pi^\star, \omega^\star_{1:T_{\inner}})$.

\end{proof}

\begin{lemma}\label{lem:regret-upper-bound}
     Given a policy $\pi^\star$ let us define $\nu^\star = \Phi^\top d^{\pi^\star}$, and let $\omega_\ell^\star = \theta_t + \gamma \cdot \bm{\mu}_t^\top V^{\pi_\ell}$. Then we have,
     $$
    \calR(\nu^\star, \pi^\star, \omega^\star_{1:T_\inner}) \le  \frac{D^2 B}{\eta_\omega K} + \eta_\omega \left( B + \frac{1}{(1-\gamma)^2}\right) + \frac{\log A}{T_\inner\ \eta_\pi} + \frac{2 \eta_\pi D^2}{(1-\gamma)^2} 
    $$
\end{lemma}
\begin{proof}

As we set $\omega_\ell^\star = \theta_t + \gamma \cdot \bm{\mu}_t^\top V^{\pi_\ell}$, its norm is bounded by $\norm{\theta_t}_2 + \gamma \cdot \norm{\bm{\mu}_t}_2 \abs{V^{\pi_\ell}} \le \sqrt{D} + \gamma \cdot \sqrt{D} \frac{\sqrt{D} }{1-\gamma} = \frac{2D}{1-\gamma}$. On the other hand, $\norm{\Sigma^{-1}\nu_\ell}_2 \le \sqrt{B}$. If we write $y = \Sigma^{-1} \nu_\ell$ then $\norm{y}_2 \le \sqrt{B}$. Moreover, $\norm{\nu_\ell}_2 = \norm{\Sigma y}_2 \le \mathrm{Tr}(\Sigma) \norm{y}_2 \le D \sqrt{B}$. Therefore, we can choose the norm of each $\nu_\ell$ to be bounded by at most $D\sqrt{B}$.

We will use the following decomposition of regret.

\begin{align*}
    \calR(\nu^\star, \pi^\star, \omega^\star_{1:T_\inner }) &= \frac{1}{T_\inner} \sum_{\ell=1}^{T_\inner} f(\nu^\star, \pi^\star, \omega_{\ell - 1}) - f(\nu_\ell, \pi_{\ell - 1}, \omega^\star_{\ell - 1})\\
    &= \frac{1}{T_\inner} \sum_{\ell=1}^{T_\inner} f(\nu^\star, \pi_{\ell - 1}, \omega_{\ell - 1}) - f(\nu_\ell, \pi_{\ell - 1}, \omega_{\ell - 1})\\
    &+ \frac{1}{T_\inner} \sum_{\ell=1}^{T_\inner} f(\nu_\ell, \pi_{\ell - 1}, \omega_{\ell - 1}) - f(\nu_\ell, \pi_{\ell - 1}, \omega^\star_{\ell - 1})\\
    &+ \frac{1}{T_\inner} \sum_{\ell=1}^{T_\inner} f(\nu^\star, \pi^\star, \omega_{\ell - 1}) - f(\nu^\star, \pi_{\ell - 1}, \omega_{\ell - 1} )
\end{align*}

\textbf{Regret in $\nu$}: Define $\calR_\nu = \frac{1}{T_\inner} \sum_{\ell=1}^{T_\inner} f(\nu^\star, \pi_{\ell - 1}, \omega_{\ell - 1}) - f(\nu_\ell, \pi_{\ell-1}, \omega_{\ell-1})$. The function $f(\cdot, \pi_{\ell - 1}, \omega_{\ell - 1})$ is $\lambda$-strongly concave, and algorithm \ref{alg:offline-primal-dual} sets $\nu_\ell$ to be a maximize of this function. Therefore, $\calR_\nu \le 0$.

\textbf{Regret in $\omega$}: Define $\calR_\omega = \frac{1}{T_\inner} \sum_{\ell=1}^{T_\inner} f(\nu_\ell, \pi_{\ell-1}, \omega_{\ell - 1}) - f(\nu_\ell, \pi_{\ell - 1}, \omega^\star_{\ell - 1})$. Since $f(\nu_t, \pi_{\ell}, \cdot)$ is a linear function we have, 
$$\calR_\omega = \frac{1}{T_\inner} \sum_{\ell=1}^{T_\inner} \left \langle \omega_{\ell-1} - \omega^\star_{\ell-1}, \nabla_{\omega_{\ell-1}} f(\nu_{\ell}, \pi_{\ell-1}, \omega_{\ell-1}) \right \rangle = \frac{1}{T_\inner} \sum_{\ell=1}^{T_\inner} \frac{1}{K} \sum_{k=1}^K \left \langle \omega_{\ell-1,k} - \omega^\star_{\ell-1}, \nabla_{\omega_{\ell-1}} f(\nu_{\ell}, \pi_{\ell-1}, \omega_{\ell-1}) \right \rangle$$
Since the $k$-th sample is drawn uniformly at random from the dataset, we have, 
$$
\E\left[\widetilde{g}_{\omega_{\ell,k}}\right] = \E\left[ \Phi^\top d^{\pi_{\ell-1}, \nu_{\ell-1}} -  \frac{1}{m_{t}} \sum_{j=1}^{m_{t}} \phi(s_j,a_j) \phi(s_j,a_j)^\top \Sigma_{\ell}^{-1}  \nu_{\ell-1}\right] = \Phi^\top d^{\pi_{\ell-1}, \nu_{\ell-1}} - \nu_{\ell - 1}
$$
We now bound the norm of $\widetilde{g}_{\omega_{\ell,k}}$. 
\begin{align*}
    \norm{\widetilde{g}_{\omega_{\ell,k}}}_2 &\le \norm{\Phi^\top d_{\ell}}_2 + \norm{\phi(s_{j}, a_{j}) \phi(s_{j}, a_{j})^\top \Sigma^{-1} \nu_{\ell-1}}_2\\
    &\le \sum_{s,a} d_{\ell}(s,a) \norm{\phi(s,a)}_2 + \norm{\Sigma^{-1} \nu_{\ell-1}}_2\\
    &\le \frac{1}{1-\gamma} + \sqrt{B}
\end{align*}
Therefore, we can apply \Cref{lem:online-sgd} to obtain the following bound.
$$
\frac{1}{K} \sum_{k=1}^K \left \langle \omega_{\ell-1,k } - \omega^\star_{\ell-1}, \nabla_{\omega_{\ell-1}} f(\nu_{\ell}, \pi_{\ell-1}, \omega_{\ell-1}) \right \rangle \le \frac{\norm{\omega_{\ell-1,1} - \omega^\star_{\ell-1}}_2^2}{2\eta_\omega K} + {\eta_\omega  (B + (1-\gamma)^{-2})}
$$
Summing over all $\ell$ and using the fact $\norm{\omega_{\ell,k}}_2 \le D \sqrt{B}$ we obtain the following upper bound on the regret $\calR_\omega$.
$$
\calR_\omega \le \frac{D^2 B}{\eta_\omega K} + \eta_\omega \left( B + \frac{1}{(1-\gamma)^2}\right)
$$

\textbf{Regret in $\pi$}: Let us define $\calR_\pi = \frac{1}{T_\inner} \sum_{t=1}^{T_\inner} f(\nu^\star, \pi^\star, \omega_{\ell-1}) - f(\nu^\star, \pi_{\ell-1}, \omega_{\ell-1})$. Using the definition of $f(\nu, \pi, \omega)$ we can rewrite the regret as follows.
\begin{align*}
    &f(\nu^\star, \pi^\star, \omega_{\ell}) - f(\nu^\star, \pi_{\ell}, \omega_{\ell}) = \calL(d^{\pi^\star, \nu^\star}, \nu^\star, g^{\pi^\star, \omega_{\ell}}, \omega_{\ell}) - \calL(d^{\pi_{\ell}, \nu^\star}, \nu^\star, g^{\pi_{\ell}, \omega_{\ell}}, \omega_{\ell})\\
    &= \underbrace{\gamma \cdot \nu^{\star^\top} \bm{\mu}^\top \left( g^{\pi^\star, \omega_{\ell}} - g^{\pi_{\ell}, \omega_{\ell}} \right) + \rho^\top \left( g^{\pi^\star, \omega_{\ell}} - g^{\pi_{\ell}, \omega_{\ell}} \right)}_{:= T_1} \\
    &+ \underbrace{\left \langle d^{\pi^\star, \nu^\star}, \Phi \omega_{\ell} - B^\top g^{\pi^\star, \omega_{\ell}} \right \rangle - \left \langle d^{\pi_{\ell}, \nu^\star}, \Phi \omega_{\ell} - B^\top g^{\pi_{\ell}, \omega_{\ell}} \right \rangle }_{:= T_2}
\end{align*}
Now from the definition of $g^{\pi, \omega}$ we have $g^{\pi^\star, \omega_{\ell}} - g^{\pi_{\ell}, \omega_{\ell}} = \sum_a (\pi^\star(a \mid s) - \pi_{\ell}(a \mid s)) \phi(s,a)^\top \omega_{\ell}$. Moreover, with $\nu^\star = \Phi^\top d^{\pi^\star}$, $\rho + \gamma \bm{\mu} \nu^\star = d^{\pi^\star}$. This implies that the term $T_1$ equals $\sum_s d^{\pi^\star}(s) \sum_a (\pi^\star(a \mid s) - \pi_{\ell}(a \mid s)) \phi(s,a)^\top \omega_{\ell}$. In order to bound the term $T_2$ note that
\begin{align*}
    &\left \langle d^{\pi^\star, \nu^\star}, \Phi \omega_{\ell} - B^\top g^{\pi^\star, \omega_{\ell}} \right \rangle - \left \langle d^{\pi_{\ell}, \nu^\star}, \Phi \omega_{\ell} - B^\top g^{\pi_{\ell}, \omega_{\ell}} \right \rangle \\
    &= \sum_{s} d^{\pi^\star}(s) \sum_a \pi^\star(a \mid s) \left( \phi(s,a)^\top \omega_{\ell} - \sum_{b} \pi^\star(b \mid s) \phi(s,b)^\top \omega_{\ell} \right) = 0
\end{align*}
Similarly the second term in $T_2$ evaluates to zero. Therefore, we have the following expression of regret.
\begin{align}
    \calR_\pi &= \frac{1}{T_\inner} \sum_{\ell=1}^{T_\inner} f(\nu^\star, \pi^\star, \omega_{\ell-1}) - f(\nu^\star, \pi_{\ell-1}, \omega_{\ell-1}) \\
    &= \frac{1}{T_\inner} \sum_{t=1}^{T_\inner} \sum_s d^{\pi^\star}(s) \sum_a (\pi^\star(a \mid s) - \pi_{\ell-1}(a \mid s)) \phi(s,a)^\top \omega_{\ell-1}
\end{align}
Now we can apply \Cref{lem:mirror-descent} to obtain the following bound on regret.
$$
\calR_\pi \le \frac{\calH(\pi^\star \lVert \pi_1)}{T_\inner \eta_\pi} + \frac{2 \eta_\pi D^2}{(1-\gamma)^2} \le \frac{\log A}{T_\inner\ \eta_\pi} + \frac{2 \eta_\pi D^2}{(1-\gamma)^2} 
$$
Therefore, the total regret can be bounded as follows.
\begin{align*}
    &\calR(\nu^\star, \pi^\star, \omega^\star_{1:T}) = \calR_\nu + \calR_\omega + \calR_\pi \\
    &\le \frac{D^2 B}{\eta_\omega K} + \eta_\omega \left( B + \frac{1}{(1-\gamma)^2}\right) + \frac{\log A}{T_\inner\cdot \eta_\pi} + \frac{2 \eta_\pi D^2}{(1-\gamma)^2} 
\end{align*}

\end{proof}

\begin{lemma}[Online Stochastic Gradient Descent]\label{lem:online-sgd}
    Let $y_1 \in W$, and $\eta > 0$. Define the sequence $y_2, \ldots, y_{n+1}$ and $h_1,\ldots,h_n$ such that for $k=1,\ldots,n$ 
    $$
    y_{k+1} = \textrm{Proj}_W\left( y_k + \eta \widehat{h}_k\right)
    $$
    and $\widehat{h}_k$ satisfies $\E\left[ \widehat{h}_k \mid \calF_{k-1}\right] = h_k$ and $\E\left[ \norm{\widehat{h}_k}_2^2 \mid \calF_{k-1}\right] \le G^2$. Then for any $y^\star \in W$,
    $$
    \E\left[ \sum_{k=1}^n \left\langle y^\star - y_k, h_k \right \rangle  \right] \le \frac{\norm{y_1 - y^\star}_2^2}{2\eta} + \frac{\eta n G^2}{2}.
    $$
\end{lemma}

\begin{lemma}[Online Stochastic Gradient Descent for Strongly Convex Loss Function]\label{lem:online-sgd-strongly-convex}
    Let $y_1 \in W$, and $\eta > 0$. Define the sequence $y_2, \ldots, y_{n+1}$ and $h_1,\ldots,h_n$ such that for $k=1,\ldots,n$ 
    $$
    y_{k+1} = \textrm{Proj}_W\left( y_k + \eta_k \widehat{h}_k\right)
    $$
    and $\widehat{h}_k$ satisfies $\E\left[ \widehat{h}_k \mid \calF_{k-1}\right] = \nabla f_k(y_k)$ and $\E\left[ \norm{\widehat{h}_k}_2^2 \mid \calF_{k-1}\right] \le G^2$. Moreover, suppose $f_k$ is $\alpha$-strongly convex for each $k$ and stepsize $\eta_k = \frac{1}{\alpha k}$. Then for any $y^\star \in W$,
    $$
    \E\left[ \sum_{k=1}^n f_k(y^\star) - f_k(y_k) \right] \le \frac{G^2}{\alpha} \log n.
    $$
\end{lemma}
\begin{proof}
    See \citet{hazan2016} (chapter 3) for a proof.
\end{proof}
\begin{lemma}[Mirror Descent, Lemma D.2 of \cite{GNOP23}]\label{lem:mirror-descent}
    Let $q_1,q_2,\ldots,q_T$ be a sequence of functions from $\calS \times \calA \rightarrow \R$ so that $\norm{q_t}_\infty \le D$. Given an initial policy $\pi_1$, and a learning rate $\alpha > 0$, define a sequence of policies 
    $$\pi_{t+1}(a \mid s) \propto \pi_t(a \mid s) e^{\alpha q_t(s,a)}$$
    for $t=1,2,\ldots,T-1$. Then for any comparator policy $\pi^\star$,
    $$
    \frac{1}{T}\sum_{t=1}^T \sum_{s \in \calS} q^{\pi^\star}(s) \left \langle \pi^\star(\cdot \mid s) - \pi_t(\cdot \mid s), q_t(s,\cdot) \right \rangle \le \frac{\calH(\pi^\star \lVert \pi_1)}{T \alpha} + \frac{\alpha D^2}{2}
    $$
\end{lemma}

\subsection{Proof of  \texorpdfstring{\Cref{cor:apx-empirical-saddle-point}}{} }
\begin{proof}
The proof of \Cref{lem:apx-saddle-point} shows that as long as the number of samples $m_t \ge  O\left( \frac{D^5 B \lambda^4 }{(1-\gamma)^2 c_1^4\varepsilon^2}\log \frac{DB\lambda }{c_1 \varepsilon \delta_0}\right)$
 \begin{align*}
            \abs{\calL_t(d,\nu;g,\omega) - \widehat{\calL}_t(d,\nu; g, \omega) } &\le \varepsilon  
    \end{align*}
    for any $d, \nu \in \calV, \omega \in \calW$, and $g = \argmin_{g'} \widehat{\calL}_t(d,\nu;g',\omega)$. 

  Given $\tilde{d}, \tilde{\nu}$ let us pick $(\tilde{g}, \tilde{\omega}) \in \argmin_{g,\omega} \widehat{\calL}_t(\tilde{d}, \tilde{\nu}; g, \omega)$. Then, we have
  $
  \widehat{\calL}_t(\tilde{d}, \tilde{\nu}; \tilde{g}, \tilde{\omega} ) \le \min_{g,\omega} \widehat{\calL}_t(\tilde{d}, \tilde{\nu}; g, \omega)
  $
  and
  \begin{align*}
      \widehat{\calL}_t(\tilde{d}, \tilde{\nu}; \tilde{g}, \tilde{\omega} ) &\ge {\calL}_t(\tilde{d}, \tilde{\nu}; \tilde{g}, \tilde{\omega} ) -\varepsilon = \tilde{\nu}^\top \theta_t - \frac{\lambda}{2} \norm{\tilde{\nu}}_2^2 - \varepsilon \ge \nu^{\star^\top} \theta_t - \frac{\lambda}{2} \norm{\nu^\star}_2^2 - 2\varepsilon  \\
      &= \calL_t(d^{\pi^\star}, {\nu}^\star, \tilde{g}, \tilde{\nu}) -2\varepsilon \ge \max_{d,\nu} \calL_t(d, {\nu}, \tilde{g}, \tilde{\nu}) - 2\varepsilon
  \end{align*}
\end{proof}

\subsection{Proof of  \texorpdfstring{\Cref{cor:apx-oracle-convergence}}{}}
\label{subsec:proof-cor-apx-oracle-convergence}
\begin{proof}
We first bound the distance between $\tilde{\nu}_t$ and $\widehat{\nu}_t$. Since $(\widehat{d}_t, \widehat{\nu}_t)$ is the optimal solution to $\widehat{\calL}_t(\cdot, \widehat{g}_t, \widehat{\omega}_t)$ we obtain the following inequality.

\begin{align*}
    \norm{\widehat{\nu}_t - \tilde{\nu}_t}_2 &\le \sqrt{\frac{\widehat{\calL}_t(\widehat{d}_t, \widehat{\nu}_t, \widehat{g}_t, \widehat{\omega}_t) - \widehat{\calL}_t(\tilde{d}_t, \tilde{\nu}_t, \widehat{g}_t, \widehat{\omega}_t)}{2\lambda}}\\
    &\le \sqrt{\frac{\widehat{\calL}_t(\widehat{d}_t, \widehat{\nu}_t, \tilde{g}_t, \tilde{\omega}_t) - \widehat{\calL}_t(\tilde{d}_t, \tilde{\nu}_t, \widehat{g}_t, \widehat{\omega}_t)}{2\lambda}}\\
    &\le \sqrt{\frac{\widehat{\calL}_t(\tilde{d}_t, \tilde{\nu}_t, \tilde{g}_t, \tilde{\omega}_t) - \widehat{\calL}_t(\tilde{d}_t, \tilde{\nu}_t, \widehat{g}_t, \widehat{\omega}_t) + 2\varepsilon}{2\lambda}}\\
    &\le \sqrt{\frac{4\varepsilon}{2\lambda}}
\end{align*}
The first inequality follows from the fact that $(\widehat{d}_t, \widehat{\nu}_t, \widehat{g}_t, \widehat{\omega}_t)$ is an exact saddle point of the objective $\widehat{\calL}_t(\cdot, \cdot)$. The last two inequalities follow because $(\tilde{d}_t, \tilde{\nu}_t, \tilde{g}_t, \tilde{\omega}_t)$ is a $2\varepsilon$-approximate saddle point of the objective $\widehat{\calL}_t(\cdot, \cdot)$. 

Now we bound the distance between $\widehat{d}_t$ and $\tilde{d}_t$. First note that, the proof of \cref{thm:finite-sample-convergence} shows that if $m_t \ge O\left(\frac{D^5 B \lambda^4}{(1-\gamma)^2 c_1^4 \varepsilon^2}\log \frac{DB\lambda t}{c_1 \varepsilon p} \right)$ we obtain $\norm{\widehat{\nu}_t - \Phi^\top \widehat{d}_t}_2 \le  8\sqrt{B}\varepsilon$ with probability at least $1-\frac{p}{t^2} \cdot \frac{6}{2\pi^2}$. Furthermore, $\tilde{\nu}_t = \Phi^\top \tilde{d}_t$, and we obtain the following inequality.
\begin{align*}
    \sqrt{\kappa} \norm{\widehat{d}_t - \tilde{d}_t}_2 \le \norm{\Phi^\top \widehat{d}_t - \Phi^\top \tilde{d}_t}_2 \le \norm{\tilde{\nu}_t - \widehat{\nu}_t}_2 + \norm{\widehat{\nu}_t - \Phi^\top \widehat{d}_t}_2 \le \sqrt{\frac{4\varepsilon}{2\lambda} } + 8 \sqrt{B} \varepsilon
\end{align*}
After rearranging we obtain,
\begin{align}
    \norm{\widehat{d}_t - \tilde{d}_t}_2 \le \sqrt{\frac{2\varepsilon}{\lambda \kappa} } + 8 \sqrt{\frac{B}{\kappa}} \varepsilon
\end{align}
The proof of \cref{thm:finite-sample-convergence} also establishes the following inequality.
\begin{align*}
    \norm{{d}^\star_t - \widehat{d}_t}_2 \le \sqrt{\frac{\varepsilon}{\lambda \kappa} } + 8 \sqrt{\frac{B}{\kappa}} \varepsilon
\end{align*}
The above two inequalities imply the following bound on the distance between $d^\star_t$ and $\tilde{d}_t$.
\begin{align*}
    \norm{d_t^\star - \tilde{d}_t}_2 \le 3\sqrt{\frac{\varepsilon}{\lambda \kappa} } + 16 \sqrt{\frac{B}{\kappa}} \varepsilon
\end{align*}
Now we can proceed very similarly to the proof of \cref{thm:finite-sample-convergence} and establish that the sequence $\{\tilde{d}_t\}_{t \ge 1}$ converges to $d_S$ if the regularization constant $\lambda$ is chosen slightly larger than required in \cref{thm:finite-sample-convergence}.
\end{proof}

\section{Missing Proofs from Section \texorpdfstring{\ref{sec:applications}}{} }

\subsection{Proof of \texorpdfstring{\Cref{lem:two-agent-mdp}}{} }
\begin{proof}
    Suppose $Q_2^\star$ (resp. $\tilde{Q}_2^\star$) be the optimal $Q$-function when the first agent deploys policy $\pi_1$ (resp. $\tilde{\pi}_1$). Then we have,
    \begin{align}\label{eq:Q-to-policy}
        \abs{\pi_2(a \mid s) - \tilde{\pi}_2(a \mid s)} = \abs{\frac{\exp(\beta Q_2^\star(s,a))}{\sum_b \exp(\beta Q_2^\star(s,b))} - \frac{\exp(\beta \tilde{Q}_2^\star(s,a))}{\sum_b \exp(\beta \tilde{Q}_2^\star(s,b))}} \le \sqrt{2}\beta \norm{Q_2^\star(s,\cdot ) - \tilde{Q}_2^\star(s,\cdot)}_2
    \end{align}
    The last inequality follows from the observation that for any $L$-Lipschitz function $f$ we have $\abs{f(x) - f(y)} \le L \norm{x-y}_2$, and for the function $f_j(x) = \frac{\exp(\beta x_j)}{\sum_i \exp(\beta x_i)}$ we have
    \begin{align*}
        \nabla_{x_k} f_j(x) = \left\{ \begin{array}{cc}
            \beta f_j(x)\left(1 - f_j(x) \right) &  \textrm{ if } k=j\\
            -\beta f_j(x) f_k(x) & \textrm{ if } k \neq j
        \end{array}\right.
    \end{align*}
    which implies $\norm{\nabla_x f_j(x)}_2 \le \sqrt{2} \beta$.

    Now we provide a bound on the norm $\norm{Q_2^\star(s,\cdot ) - \tilde{Q}_2^\star(s,\cdot)}_2$. Let $\overline{r}_2$ (resp. $\tilde{r}_2$) be the reward function of agent $2$ in response to policy $\pi_1$ (resp. $\tilde{\pi}_1$). 
    \begin{align*}
        \abs{\overline{r}_2(s,a) - \tilde{r}_2(s,a)} &\le \sum_{a_1} \abs{\pi_1(a_1 \mid s) - \tilde{\pi}_1(a_1 \mid s)} \abs{r_2(s, a_1, a_2)}\\
        &\le \norm{\pi_1(\cdot \mid s) - \tilde{\pi}_1(\cdot \mid s)}_1 \max_{a_1, a_2} \abs{r_2(s, a_1, a_2)}\\
        &\le \delta \cdot r_{\max}
    \end{align*}
    Similarly, let $\bar{P}_2$ (resp. $\tilde{P}_2$) be the probability transition function of agent $2$ in response to policy $\pi_1$ (resp. $\tilde{\pi}_1$). Then for any state $s$, and action $a_2$ we have,
    \begin{align*}
        \sum_{s'} \abs{\overline{P}_2(s' \mid s,a_2) - \tilde{P}_2(s' \mid s,a_2)} &\le \sum_{s'}\sum_{a_1} \abs{\pi_1(a_1 \mid s) - \tilde{\pi}_1(a_1 \mid s)} P(s' \mid s, a_1, a_2)\\
        &\le \norm{\pi_1(\cdot \mid s) - \tilde{\pi}_1(\cdot \mid s)}_1 \sum_{a_1} \sum_{s'} P(s' \mid s, a_1, a_2)\\
        &\le \delta \cdot A_1 
    \end{align*}

    Now consider the problem of computing the optimal policies of agent $2$ starting from the state, action pair $(s,a)$. Let $\bar{\pi}_2$ (resp. $\tilde{\pi}_2$) be the optimal policy with reward $\overline{r}_2$, and transition $\overline{P}_2$ (resp. reward $\tilde{r}_2$, and transition $\tilde{P}_2$). Moreover, let $\bar{d}_2$ (resp. $\bar{q}_2$) be the state, action occupancy measure of the policy $\bar{\pi}_2$ under probability transition $\overline{P}_2$ (resp. $\tilde{P}_2$). Similarly, let $\tilde{d}_2$ (resp. $\tilde{q}_2$) be the state, action occupancy measure of the policy $\tilde{\pi}_2$ under probability transition $\tilde{P}_2$ (resp. $\overline{P}_2$).
    
    Let $\bar{d}_2^\star$ (resp. $\tilde{d}_2^\star$) be the optimal state, action occupancy measure under reward function $\overline{r}_2$ (resp. $\tilde{r}_2$). Then we have, 
    \begin{align*}
        Q_2^\star(s,a) &= \sum_{s',b'} \overline{r}_2(s',b') \bar{d}_2 (s',b') \ge \sum_{s',b'} \overline{r}_2(s',b') \tilde{q}_2 (s',b') \\
        & = \sum_{s',b'} \overline{r}_2(s',b') \tilde{d}_2 (s',b') + \sum_{s',b'} \overline{r}_2(s',b') \left(\tilde{q}_2 (s',b') - \tilde{d}_2(s',b') \right)\\
        &\ge \sum_{s',b'} \tilde{r}_2(s',b') \tilde{d}_2(s',b') + \sum_{s',b'} \left(\overline{r}_2(s',b') - \tilde{r}_2(s',b') \right) \tilde{d}_2 (s',b') - r_{\max} \cdot \norm{\tilde{q}_2 - \tilde{d}_2}_1\\
        &\ge \tilde{Q}_2^\star(s,a) - \delta \cdot r_{\max} \sum_{s',b'} \tilde{d}_2(s',b') - r_{\max} \cdot \frac{\delta A_1 \gamma}{(1-\gamma)^2}\\
        &\ge \tilde{Q}_2^\star(s,a) - \frac{\delta \cdot r_{\max}}{1-\gamma} - r_{\max} \cdot \frac{\delta A_1 \gamma}{(1-\gamma)^2}
    \end{align*}
    The first inequality follows because the policy $\tilde{\pi}_2$ is sub-optimal under the reward $\bar{r}_2$ and transition $\overline{P}_2$. The third inequality uses \Cref{lem:transition-to-measure-bound} and $\tilde{Q}_2^\star(s,a) = \sum_{s',b'} \tilde{r}_2(s',b') \tilde{d}_2(s',b')$. The final inequality uses that $\sum_{s',a'} d(s',a') = 1/(1-\gamma)$ for any occupancy measure $d$. Similarly we can show that $\tilde{Q}_2^\star(s,a) \ge Q_2^\star(s,a) - \frac{\delta \cdot r_{\max}}{1-\gamma} - r_{\max} \cdot \frac{\delta A_1 \gamma}{(1-\gamma)^2}$. Therefore, 
    $$\norm{Q_2^\star(s,\cdot) - \tilde{Q}_2^\star(s,\cdot)}_2 \le \sqrt{A_2} r_{\max} \frac{2\delta A_1}{(1-\gamma)^2}.$$ 
    Therefore, using \Cref{eq:Q-to-policy} we obtain the following bound.
    $$
    \abs{\pi_2(a \mid s) - \tilde{\pi}_2(a \mid s)} \le 2\sqrt{2} r_{\max} \frac{A_1 \sqrt{A_2} \beta \delta}{(1-\gamma)^2}
    $$
    Now, let $r_1$ (resp. $\tilde{r}_1$) be the reward function of agent $1$ when agent $2$ adopts policy $\pi_2$ (resp. $\tilde{\pi}_2$).
    \begin{align*}
        \abs{{r}_1(s,a) - \tilde{r}_1(s,a)} &\le \sum_{a_2} \abs{\pi_2(a_2 \mid s) - \tilde{\pi}_2(a_2 \mid s)} \abs{r_1(s, a_1, a_2)}\\
        &\le \delta \cdot \frac{2\sqrt{2} \beta A_1 A_2^{3/2} r^2_{\max} }{(1-\gamma)^2}
    \end{align*}
    Similarly, let $P_1$ (resp. $\tilde{P}_1$) be the probability transition function of agent $1$ when agent $2$ adopts policy $\pi_2$ (resp. $\tilde{\pi}_2$).
    \begin{align*}
        \abs{P_1(s' \mid s,a) - \tilde{P}_1(s' \mid s,a)} &\le \sum_{a_2} \abs{\pi_2(a_2 \mid s) - \tilde{\pi}_2(a_2 \mid s)} P(s, a_1, a_2)\\
        &\le \delta \cdot \frac{2\sqrt{2} \beta A_1 A_2^{3/2} r_{\max} }{(1-\gamma)^2}
    \end{align*}
\end{proof}

\begin{lemma}\label{lem:transition-to-measure-bound}
    Suppose $\norm{P(\cdot \mid s,a) - \tilde{P}(\cdot \mid s,a)}_1 \le \beta$ for any state,action pair $(s,a)$. Then for any policy $\pi$ and any starting state distribution $\rho$ we have $\norm{d^\pi_\rho - \tilde{d}^\pi_\rho}_1 \le \frac{\beta \gamma}{(1-\gamma)^2}$. 
\end{lemma}
\begin{proof}
    Let $P^\pi_h$ (resp. $\tilde{P}^\pi$) be the state distribution at time-step $h$ resulting from the starting distribution $\rho$ under the probability transition function $P$ (resp. $\tilde{P}$). 
    \begin{align*}
        P^\pi_h(s') - \tilde{P}^\pi_h(s') &= \sum_{s,a} \left(P^\pi_{h-1}(s) P(s' \mid s, a) - \tilde{P}^\pi_{h-1}(s) \tilde{P}(s' \mid s,a) \right) \pi(a \mid s)\\
        &= \sum_{s} P^\pi_{h-1}(s) \sum_a \left(P(s' \mid s, a)  -  \tilde{P}(s' \mid s,a)\right) \pi(a \mid s)\\
        &+ \sum_s \left(P^\pi_{h-1}(s) - \tilde{P}^\pi_{h-1}(s) \right) \sum_a \tilde{P}(s' \mid s,a) \pi(a \mid s)
    \end{align*}
    Taking absolute values and summing over all the states we obtain the following inequality.
    \begin{align*}
        \sum_{s'} \abs{P^\pi_h(s') - \tilde{P}^\pi_h(s') } &\le \sum_{s,a}\norm{P(\cdot \mid s,a) - \tilde{P}(\cdot \mid s,a)}_1 P^\pi_{h-1}(s) \pi(a \mid s)\\
        &+ \sum_{s'} \sum_s \abs{P^\pi_{h-1}(s) - \tilde{P}^\pi_{h-1}(s)} \sum_a \tilde{P}(s' \mid s,a) \pi(a \mid s)\\
        &\le \beta \cdot \sum_{s,a}P^\pi_{h-1}(s) \pi(a \mid s) + \sum_s \abs{P^\pi_{h-1}(s) - \tilde{P}^\pi_{h-1}(s)} \sum_a  \pi(a \mid s)\\
        &\le \beta + \sum_s \abs{P^\pi_{h-1}(s) - \tilde{P}^\pi_{h-1}(s)}
    \end{align*}
    Since $P^pi_0(s) = \rho(s)$ we have $\norm{P^\pi_0 - \tilde{P}^\pi_0}_1 = 0$, and $\norm{P^\pi_h - \tilde{P}^\pi_h}_1 = \beta\cdot h$. 
    Now using the definition of state, action occupancy measure we get,
    \begin{align*}
        d^\pi(s,b) - \tilde{d}^{\pi}(s,b) = \sum_{h} \gamma^h \left( P^\pi_h(s) - \tilde{P}^\pi_h(s) \right) \pi(b \mid s) \le \sum_{h} \gamma^h \abs{P^\pi_h(s) - \tilde{P}^\pi_h(s) } \pi(b \mid s).
    \end{align*}
    Therefore,
    \begin{align*}
        \norm{d^\pi - \tilde{d}^\pi}_1 \le \sum_h \sum_{s'}\gamma^h \abs{P^\pi_h(s) - \tilde{P}^\pi_h(s) } \sum_b \pi(b \mid s) =\sum_h \beta \gamma^h h = \frac{\beta \gamma}{(1-\gamma)^2}
    \end{align*}
\end{proof}
\subsection{Proof of \texorpdfstring{\Cref{lem:multi-agent-mdp} }{} }

\newcommand{\br}{\bar{r}}
\newcommand{\tr}{\tilde{r}}

\begin{proof}
    Suppose $Q_f^\star$ (resp. $\tilde{Q}_f^\star$) be the optimal state,action welfare-function when the first agent deploys policy $\pi_1$ (resp. $\tilde{\pi}_1$). Similar to the proof of \Cref{lem:two-agent-mdp} we can establish the following  inequality.
    \begin{align*}
        \abs{\pi_f(\bm{a} \mid s) - \tilde{\pi}_f(\bm{a} \mid s)} \le \sqrt{2}\beta \norm{Q_f^\star(s,\cdot ) - \tilde{Q}_f^\star(s,\cdot)}_2
    \end{align*}

Let us also write $\bar{r}_j(s, \bm{a})$ (resp. $\tilde{r}_j(s,\bm{a})$ be the reward function of agent $j \in \set{2,\ldots,m+1}$ when the first agent plays policy $\pi_1$ (resp. $\tilde{\pi}_1$). Then we have,
\begin{align*}
    \abs{\br_j(s,\bm{a}) - \tr_j(s,\bm{a})} &\le \sum_{a_1} \abs{\pi_1(a_1\mid s) - \tilde{\pi}_1(a_1 \mid s)} \abs{r_j(s,a_1,\bm{a}}\\
    &\le \norm{\pi_1(\cdot \mid s) - \tilde{\pi}_1(\cdot \mid s)} r_{\max}\\
    &\le \delta \cdot r_{\max}.
\end{align*}

Let us also write $\bar{P}$ (resp. $\tilde{P}$) to be the probability transition function when the first agent plays policy $\pi_1$ (resp. $\tilde{\pi}_1$). Then we have,
 \begin{align*}
        \sum_{s'} \abs{\overline{P}_2(s' \mid s,\bm{a}) - \tilde{P}_2(s' \mid s,\bm{a})} &\le \sum_{s'}\sum_{a_1} \abs{\pi_1(a_1 \mid s) - \tilde{\pi}_1(a_1 \mid s)} P(s' \mid s, a_1,\bm{a})\\
        &\le \norm{\pi_1(\cdot \mid s) - \tilde{\pi}_1(\cdot \mid s)}_1 \sum_{a_1} \sum_{s'} P(s' \mid s, a_1, \bm{a})\\
        &\le \delta \cdot A_1.
    \end{align*}
    
    Now consider the problem of computing the optimal CCE of agents $\set{2,\ldots,m+1}$ starting from the state, action pair $(s,\bm{a})$. Let $\bar{\pi}_f$ (resp. $\tilde{\pi}_f$) be the optimal policy with reward $\overline{r}_j$, and transition $\overline{P}$ (resp. reward $\tilde{r}_j$, and transition $\tilde{P}_j$). There are two cases to consider. First, the policy $\tilde{\pi}_f$ is a CCE under reward $\bar{r}_j$ and transition $\bar{P}$. In that case, we can proceed similar to the proof \Cref{lem:two-agent-mdp} and establish the following inequality.
    \begin{align}\label{eq:prelim-bound}
        Q^\star_f(s,\bm{a}) \ge \tilde{Q}^\star_f(s,\bm{a}) - \frac{\delta \cdot r_{\max}}{1-\gamma} - \frac{\delta A_1 \gamma \cdot r_{\max}}{(1-\gamma)^2}
    \end{align}
    The second case occurs when $\tilde{\pi}_f$ is no longer a CCE. In that case, let $P \subseteq \set{2,\ldots,m+1}$ be the set of agents that can improve under $\tilde{\pi}_f$ i.e. for each $j \in P$ there exists a strategy $\pi_j' : S \rightarrow \Delta(A_j)$ so that $Q^{\pi_j', \tilde{\pi}_{f,-j}}_j(s,\bm{a}) > Q^{\tilde{\pi}_{f}}_j(s,\bm{a}) $. However, $\tilde{\pi}_f$ is a CCE under reward $\tilde{r}_j$ and transition $\tilde{P}$, which implies $\tilde{Q}_j^{\tilde{\pi}_f}(s,\bm{a}) \ge \tilde{Q}^{\pi_j', \tilde{\pi}_{f,-j}}_j(s,\bm{a})$. Now using \Cref{lem:transition-to-measure-bound} we get   and $Q_j^{\tilde{\pi}_f}(s,\bm{a}) \ge \tilde{Q}_j^{\tilde{\pi}_f}(s,\bm{a}) - r_{\max} \cdot \frac{\delta A_1 \gamma}{(1-\gamma)^2}$. Therefore, we obtain the following inequality.
    \begin{align*}
    Q_j^{\tilde{\pi}_f}(s,\bm{a}) + 2 r_{\max} \cdot \frac{\delta A_1 \gamma}{(1-\gamma)^2} &\ge \tilde{Q}_j^{\tilde{\pi}_f}(s,\bm{a}) + r_{\max} \cdot \frac{\delta A_1 \gamma}{(1-\gamma)^2} \\
    &\ge \tilde{Q}^{\pi_j', \tilde{\pi}_{f,-j}}_j(s,\bm{a})+ r_{\max} \cdot \frac{\delta A_1 \gamma}{(1-\gamma)^2} \\
    &\ge Q^{\pi_j', \tilde{\pi}_{f,-j}}_j(s,\bm{a}) > Q^{\tilde{\pi}_{f}}_j(s,\bm{a})
    \end{align*}
    This inequality shows that $\tilde{\pi}_f$ is $\varepsilon$-CCE under reward $\bar{r}_j$ and transition $\bar{P}$ for $\varepsilon = 2 r_{\max} \cdot \frac{\delta A_1 \gamma}{(1-\gamma)^2}$. Since the optimal welfare over the set of $\varepsilon$-CCE is at most $\varepsilon m$ plus the optimal welfare over the set of exact CCE, we have the following inequality.
    \begin{align*}
        Q^\star_f(s,\bm{a}) &\ge Q^{\tilde{\pi}_f}(s,\bm{a}) -  2 r_{\max} \cdot \frac{m \delta A_1 \gamma}{(1-\gamma)^2} \\
        &= \sum_j Q^{\tilde{\pi}_f}_j(s,\bm{a}) -  2 r_{\max} \cdot \frac{m \delta A_1 \gamma}{(1-\gamma)^2}\\
        &\ge \sum_j \tilde{Q}^{\tilde{\pi}_f}_j(s,\bm{a}) -  3 r_{\max} \cdot \frac{m \delta A_1 \gamma}{(1-\gamma)^2}\\
        &= \tilde{Q}^\star_f(s,\bm{a}) - 3 r_{\max} \cdot \frac{m \delta A_1 \gamma}{(1-\gamma)^2}
    \end{align*}
    The above inequality and \cref{eq:prelim-bound} imply $Q^\star_f(s,\bm{a}) \ge \tilde{Q}^\star_f(s,\bm{a}) - 3 r_{\max} \cdot \frac{m \delta A_1}{(1-\gamma)^2}$. Similarly we can show $\tilde{Q}^\star_f(s,\bm{a}) \ge {Q}^\star_f(s,\bm{a}) - 3 r_{\max} \cdot \frac{m \delta A_1}{(1-\gamma)^2}$, and we get the following bound.
    $$
    \abs{\pi_f(\bm{a} \mid s) - \tilde{\pi}_f(\bm{a} \mid s)} \le \sqrt{2} \beta \norm{Q^\star_f(s,\cdot) - \tilde{Q}^\star_f(s,\cdot) }_2 \le \beta \delta \cdot \frac{3 \sqrt{2} m A^{m/2+1} \cdot r_{\max}}{(1-\gamma)^2}
    $$
    Now, let $r_1$ (resp. $\tilde{r}_1$) be the reward function of agent $1$ when the follower agents adopt policy $\pi_f$ (resp. $\tilde{\pi}_f$).
    \begin{align*}
        \abs{{r}_1(s,a_1) - \tilde{r}_1(s,a_1)} &\le \sum_{\bm{a}} \abs{\pi_f(\bm{a} \mid s) - \tilde{\pi}_f(\bm{a} \mid s)} \abs{r_1(s, a_1, \bm{a})}\\
        &\le A^m \cdot  \beta \delta \cdot \frac{3 \sqrt{2} m A^{m/2+1} \cdot r_{\max}}{(1-\gamma)^2}\cdot r_{\max}\\
        &\le \delta \cdot \frac{3 \sqrt{2} \beta m A^{3m/2+1}  r^2_{\max}}{(1-\gamma)^2}
    \end{align*}
    Similarly, let $P_1$ (resp. $\tilde{P}_1$) be the probability transition function of agent $1$ when the follower agents adopt policy $\pi_f$ (resp. $\tilde{\pi}_f$).
    \begin{align*}
        \abs{P_1(s' \mid s,a_1) - \tilde{P}_1(s' \mid s,a_1)} &\le \sum_{\bm{a}} \abs{\pi_f(\bm{a} \mid s) - \tilde{\pi}_f(\bm{a} \mid s)} P(s' \mid s, a_1, \bm{a})\\
        &\le A^m \cdot  \beta \delta \cdot \frac{3 \sqrt{2} m A^{m/2+1} \cdot r_{\max}}{(1-\gamma)^2}\cdot 1\\
        &\le \delta \cdot \frac{3 \sqrt{2} \beta m A^{3m/2+1}  r_{\max}}{(1-\gamma)^2}
    \end{align*}
\end{proof}

\end{document}